\pdfoutput=1
\documentclass{article}

\PassOptionsToPackage{numbers,sort&compress}{natbib}

\usepackage[final,main]{neurips_2026}

\usepackage[T1]{fontenc}

\usepackage{hyperref}
\hypersetup{hidelinks,hypertexnames=false}
\usepackage{booktabs}
\usepackage{amsfonts}
\usepackage{xcolor}

\usepackage{amsmath,amsthm,bm}
\usepackage{algorithm}
\usepackage{algpseudocode}
\algnewcommand\algorithmicinput{\textbf{Input:}}
\algnewcommand\Input{\item[\algorithmicinput]}
\algnewcommand\algorithmicparams{\textbf{Parameters:}}
\algnewcommand\Parameters{\item[\algorithmicparams]}
\algnewcommand\algorithmicoutput{\textbf{Output:}}
\algnewcommand\Output{\item[\algorithmicoutput]}
\usepackage{float}

\newtheorem{theorem}{Theorem}[section]
\newtheorem{proposition}[theorem]{Proposition}
\theoremstyle{definition}
\newtheorem{definition}[theorem]{Definition}

\usepackage{subcaption}
\usepackage{pifont}
\newcommand{\cmark}{\ding{51}}
\newcommand{\xmark}{\ding{55}}
\usepackage{tabularx}
\usepackage{colortbl}
\usepackage[normalem]{ulem}
\definecolor{figNavy}{HTML}{3E5F8A}
\definecolor{figTerracotta}{HTML}{C85C4E}
\definecolor{figPlum}{HTML}{8470B0}
\definecolor{navyBand}{HTML}{E2EAFB}
\definecolor{terracottaBand}{HTML}{FCE4E0}
\definecolor{plumBand}{HTML}{EDE7F8}
\definecolor{inkGreen}{HTML}{3C7A5A}
\definecolor{inkRed}{HTML}{B5432C}
\colorlet{errcolor}{inkRed}
\colorlet{correctcolor}{inkGreen}

\usepackage{graphicx}
\usepackage{wrapfig}
\usepackage{enumerate}
\newcommand{\itemref}[1]{(\ref{#1})}
\DeclareMathOperator*{\argmax}{arg\,max}

\colorlet{cwText}{figNavy}
\colorlet{cwBand}{navyBand}
\colorlet{rdText}{figTerracotta}
\colorlet{rdBand}{terracottaBand}
\colorlet{knText}{figPlum}
\colorlet{knBand}{plumBand}
\newsavebox{\bandbox}\newsavebox{\bandsavedstrut}
\newlength{\bandht}\newlength{\banddp}\newlength{\bandcap}\newlength{\bandextra}\newlength{\bandbelow}
\newlength{\bandheight}
\newcommand{\bandlabelsize}{\small}
\makeatletter
\newcommand{\bandmeasure}{\sbox{\bandbox}{\small\textbf{Class-wise forgetting}}\global\bandht=\ifdim\ht\bandbox>\ht\@arstrutbox\ht\bandbox\else\ht\@arstrutbox\fi\relax
  \global\banddp=\ifdim\dp\bandbox>\dp\@arstrutbox\dp\bandbox\else\dp\@arstrutbox\fi\relax
  \global\bandextra=\dimexpr\bandht+\banddp+\belowrulesep+\aboverulesep-\bandheight\relax
  \global\setbox\bandsavedstrut\copy\@arstrutbox\global\setbox\@arstrutbox\hbox{}}
\newcommand{\toprulebeforeband}{\noalign{\bandmeasure\global\bandbelow=\bandextra}\specialrule{\heavyrulewidth}{\abovetopsep}{0pt}}
\newcommand{\midrulebeforeband}{\noalign{\bandmeasure\global\bandbelow=\dimexpr\bandextra/2\relax}\specialrule{\lightrulewidth}{\dimexpr\aboverulesep+\bandextra/2\relax}{0pt}}
\newcommand{\midruleafterband}{\noalign{\global\setbox\@arstrutbox\copy\bandsavedstrut}\specialrule{\lightrulewidth}{0pt}{\dimexpr\belowrulesep+\bandbelow\relax}}
\newcommand{\bandfill}[1]{\settoheight{\bandcap}{\bandlabelsize\textbf{X}}\rule[-\dimexpr(\bandheight-\bandcap)/2\relax]{0pt}{\bandheight}{\bandlabelsize #1}}
\makeatother

\newcommand{\cw}[1]{\textcolor{cwText}{#1}}
\newcommand{\rd}[1]{\textcolor{rdText}{#1}}
\newcommand{\kn}[1]{\textcolor{knText}{#1}}
\newcommand{\stdnum}[2]{#1\textsubscript{$\pm$#2}}
\newcommand{\plainzero}{0.00}
\newcommand{\cwdiff}[1]{\cw{(#1)}}
\newcommand{\rddiff}[1]{\rd{(#1)}}

\title{Retain-Neutral Surrogates for Min--Max Unlearning}

\author{Junhao~Cai \\ Korea University \\ {\small\texttt{junhochae@korea.ac.kr}} \And
  Dohun~Kim \\ Korea University \\ {\small\texttt{dohunkim@korea.ac.kr}} \And
  Dowon~Kim \\ Korea University \\ {\small\texttt{dowonkim@korea.ac.kr}} \AND
  Sung~Il~Choi \\ Korea University \\ {\small\texttt{sungchoi@korea.ac.kr}} \And
  Chengjun~Jin \\ Korea University \\ {\small\texttt{chengjunjin2001@korea.ac.kr}} \And
  Juhyun~Park \\ Korea University \\ {\small\texttt{juhyunpark@korea.ac.kr}} \AND
  Changhee~Joo\thanks{Corresponding author.} \\ Korea University \\ {\small\texttt{changhee@korea.ac.kr}}
}

\begin{document}
\raggedbottom

\maketitle

\begin{abstract}
Machine unlearning aims to make a model forget designated training data while retaining performance on the rest. In \mbox{min--max} unlearning, an inner adversary perturbs the model to maximize the forget loss, and the model is updated to minimize the retain loss at the resulting surrogate point. When the forget and retain gradients are positively coupled, the standard perturbation also moves toward higher retain loss, and the retain step begins by recovering that loss. We propose \textbf{R}etain-\textbf{O}rthogonal \textbf{S}urrogate \textbf{U}nlearning (ROSU), which redefines the adversary to admit only perturbations orthogonal to the retain gradient. The one that raises the forget loss fastest has a closed form, and ROSU uses it to set the surrogate point and amplifies it into the forgetting step. Under positive coupling the ROSU surrogate avoids the retain cost that coupling induces in the standard surrogate, and the difference shrinks as the gradients approach orthogonality. ROSU has the smallest accuracy gap to retraining in every vision setting (\mbox{CIFAR-10}, \mbox{CIFAR-100}, \mbox{Tiny-ImageNet}) and is competitive with the baselines on language benchmarks (TOFU, WMDP). The redefinition helps most under high positive coupling. Code is available at \url{https://github.com/unlearn-ml/machine-unlearning}.
\end{abstract}

\section{Introduction}
\label{sec:intro}

Machine unlearning aims to remove the influence of a designated forget set from a trained model while preserving performance on the remaining retain set~\cite{unlearning,bourtoule2021machine}. Retraining without the forget data is the reference, but it is prohibitive at modern scale~\cite{izzo2021approximate,zhao2024makes}, so approximate methods edit the trained model instead~\cite{scrubbing,golatkar2021mixed}. The two goals conflict because the same parameters serve both and because the forget and retain sets are statistically dependent, which can drive unconstrained ascent away from the retraining reference~\cite{mavrothalassitis2025ascent}. Moving along the forget gradient is the most direct way to forget, but it works against retention whenever the forget and retain gradients point in similar directions. Figure~\ref{fig:motivation_a} shows where this happens. The coupling between the two gradients is positive in every setting we measure. It is low in class-wise forgetting~\cite{kim2025unlearningaware} but high in 10\% random forgetting~\cite{jia2023model} and in knowledge unlearning on TOFU (Task of Fictitious Unlearning)~\cite{maini2024tofu,dorna2025openunlearning}.

\begin{figure*}[t]
\centering

\begin{subfigure}[t]{0.32\textwidth}
    \centering
    \includegraphics[width=\linewidth]{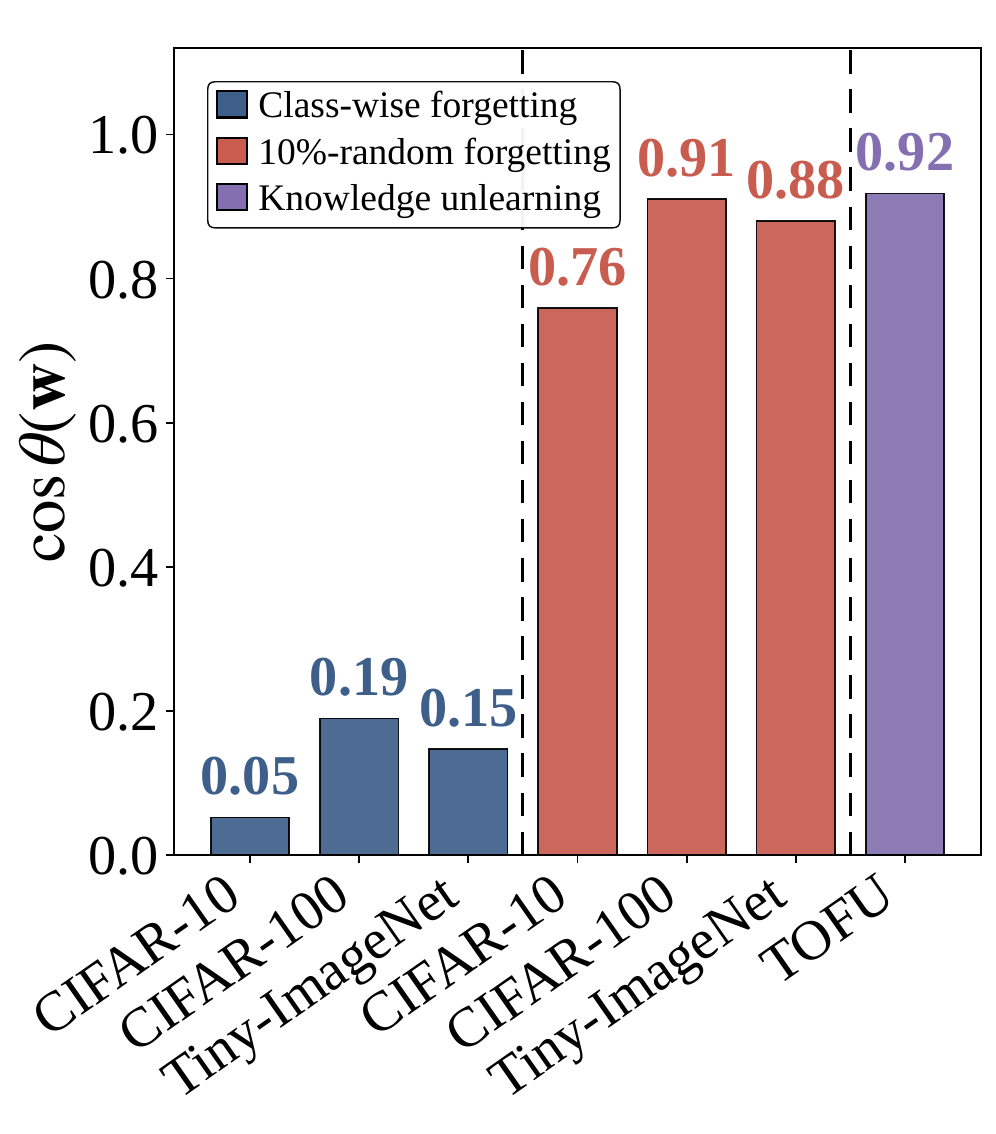}
    \caption{Coupling diagnosis.}
    \label{fig:motivation_a}
\end{subfigure}
\hfill
\begin{subfigure}[t]{0.32\textwidth}
    \centering
    \includegraphics[width=\linewidth]{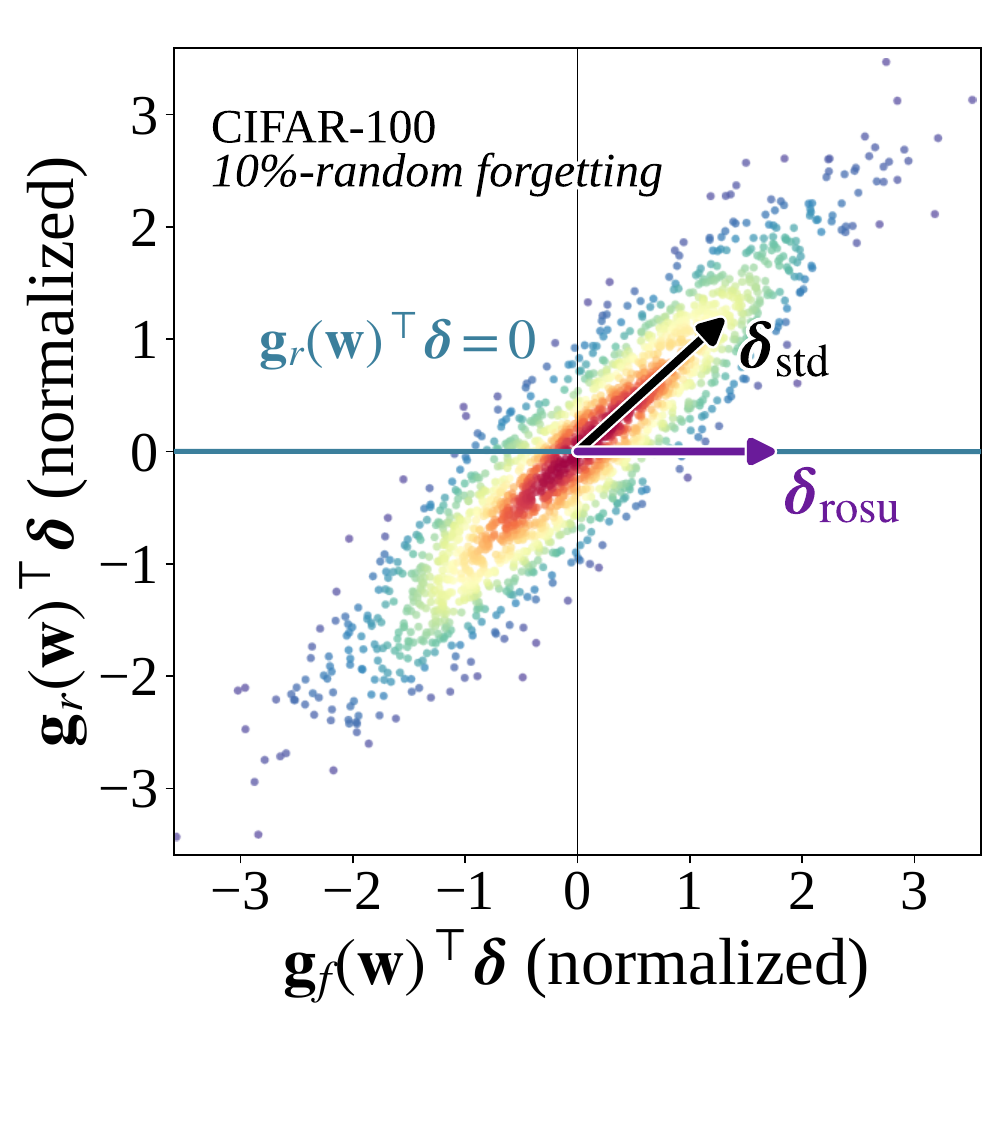}
    \caption{Retain-orthogonal surrogate.}
    \label{fig:motivation_b}
\end{subfigure}
\hfill
\begin{subfigure}[t]{0.32\textwidth}
    \centering
    \includegraphics[width=\linewidth]{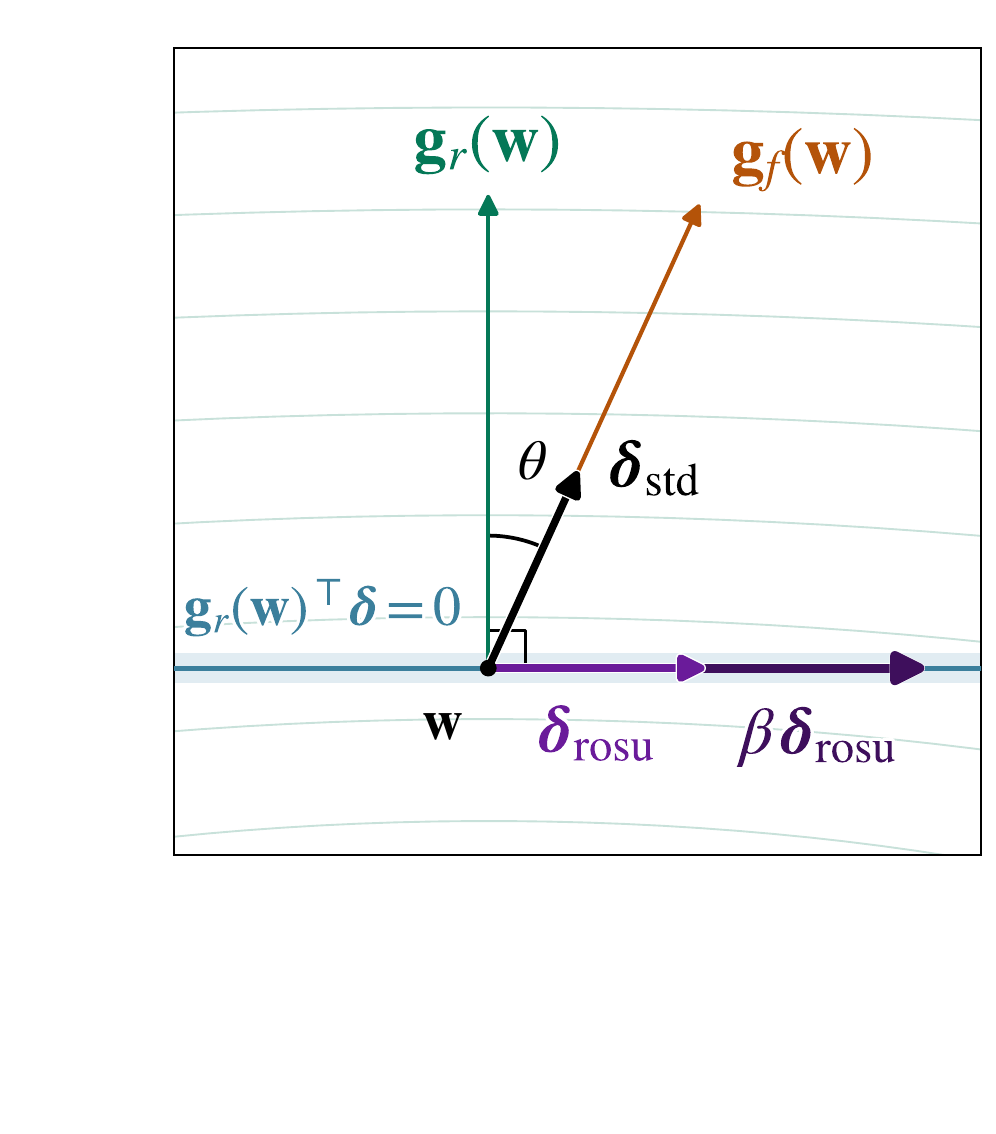}
    \caption{Retain-neutral amplification.}
    \label{fig:motivation_c}
\end{subfigure}

\caption{\textbf{(a)}~The coupling \(\cos\theta(\mathbf{w})\) between the forget and retain gradients \(\mathbf{g}_f(\mathbf{w})\) and \(\mathbf{g}_r(\mathbf{w})\) is low in class-wise forgetting and high in random forgetting and TOFU. \textbf{(b)}~In a highly coupled setting, \(\boldsymbol{\delta}_{\mathrm{std}}\) moves toward higher forget and retain losses, while \(\boldsymbol{\delta}_{\mathrm{rosu}}\) moves toward higher forget loss on the retain-neutral hyperplane \(\mathbf{g}_r(\mathbf{w})^\top\boldsymbol{\delta}=0\). The arrows indicate directions. \textbf{(c)}~The same direction \(\boldsymbol{\delta}_{\mathrm{rosu}}\) sets the surrogate point, and its amplification \(\beta\boldsymbol{\delta}_{\mathrm{rosu}}\) is the forgetting step, which stays retain-neutral.}
\label{fig:motivation}
\end{figure*}

Min--max unlearning addresses this conflict by assigning the two goals to different levels rather than one combined loss. An inner adversary perturbs the model to raise the forget loss, and a retain step updates the model to minimize the retain loss at the resulting surrogate point. Since retention belongs to the retain step, the adversary should not disturb it. In the prior min--max method~\cite{kim2025unlearningaware}, the adversary is free to move the model in any direction to maximize the forget loss. Its perturbation follows the forget gradient, so under positive coupling it also moves toward higher retain loss, and the retain step begins by recovering that loss (Figure~\ref{fig:motivation_b}). The retain cost comes from the directions the adversary is allowed to take. This raises a question. \emph{Which perturbations should the inner adversary use?}

We admit only retain-neutral perturbations, which are orthogonal to the retain gradient at the current weights. Among them we take the one that raises the forget loss fastest. Unless the two gradients are parallel, this choice is unique and has a closed form. \textbf{R}etain-\textbf{O}rthogonal \textbf{S}urrogate \textbf{U}nlearning (ROSU) uses this one direction twice (Figure~\ref{fig:motivation_c}). First, the direction sets the surrogate point, and the retain gradient there is transported back to the current weights as the retain step. Second, the amplified direction is the forgetting step, which is retain-neutral and leaves retention to the retain step. Among retain-neutral steps of the same length, it raises the forget loss fastest.

The contributions of our work are as follows:
\begin{itemize}
\item We identify how the retain cost that coupling induces enters min--max unlearning. When the forget and retain gradients are positively coupled, a step along the forget gradient moves toward higher retain loss. An unconstrained inner adversary takes such a step before the retain step, and this cost grows with the coupling and vanishes at orthogonality.
\item We propose ROSU, a min--max unlearning method whose inner adversary admits only perturbations orthogonal to the retain gradient. These perturbations avoid the retain cost that coupling induces in an unconstrained inner adversary. Among them, the one that raises the forget loss fastest has a closed form. This perturbation sets the surrogate point of the retain step, and its amplification is a retain-neutral forgetting step.
\item We evaluate ROSU on vision benchmarks (CIFAR-10, CIFAR-100, Tiny-ImageNet) and language benchmarks (TOFU, WMDP). It has the smallest accuracy gap to retraining in both class-wise and random forgetting on every vision benchmark and is competitive on both language benchmarks.
\end{itemize}

\section{Related work}
\label{sec:related}

Exact unlearning restructures training so that a deletion request updates a few summations or retrains one affected sub-model~\cite{unlearning,bourtoule2021machine,yan2022arcane}. Formal guarantees for approximate deletion were obtained for linear and logistic models~\cite{certified_removal,izzo2021approximate}. For deep networks, retraining remains the reference, and its cost motivates approximate methods~\cite{nguyen2025survey,zhao2024makes}. Baselines fine-tune on the retained data or ascend the forget loss, and Fisher forgetting adds noise shaped by the retained data~\cite{scrubbing}. Later methods distill from the original model, unlearn on a sparsified model, or update only salient or forget-specific weights~\cite{kurmanji2023towards,jia2023model,fan2023salun,foster2024fast}. Large language model (LLM) unlearning adapted gradient ascent and preference objectives and introduced benchmarks for fictitious and hazardous knowledge~\cite{yao2024large,zhang2024negative,maini2024tofu,li2024the}. Across these methods the forgetting objective varies, while any retain term is evaluated at the current weights.

Min--max methods evaluate the objective at a surrogate point instead. Sharpness-aware minimization perturbs the weights toward higher loss and descends with the gradient taken there, which favors neighborhoods of uniformly low loss~\cite{foret2021sharpnessaware,kwon2021asam}. Unlearning-Aware Minimization (UAM) carries this structure to unlearning, perturbing the model to maximize the forget loss and minimizing the retain loss with the gradient at that point~\cite{kim2025unlearningaware}. Projection methods instead orthogonalize the final update, from gradient surgery in multi-task learning~\cite{yu2020gradient} to null-space and retain-orthogonal updates in unlearning~\cite{null_unlearn,shamsian2026meansunlearningpersamplegradient,zhou2026geometricdisentangelmentunlearning,wang2025gru,hu2026falcon}. Both lines leave the inner adversary unconstrained or absent. ROSU redefines the adversary to admit only retain-neutral perturbations, so the surrogate point, the transported retain step, and the forgetting step come from one solution. Appendix~\ref{app:related} gives an extended comparison.

\section{Surrogate geometry in min--max unlearning}
\label{sec:geometry}

\subsection{Notation}
\label{sec:notation}

Let \(\mathbf{w}\in\mathbb{R}^p\) denote the full parameter vector. All coordinates are optimized unless a trainable subset is specified. The training data are partitioned into a forget set \(\mathcal D_f\) and a retain set \(\mathcal D_r\). Write \(L(\mathbf{w};\mathcal D)\) for the empirical loss on dataset \(\mathcal D\). Set \(L_f(\mathbf{w}):=L(\mathbf{w};\mathcal D_f)\) and \(L_r(\mathbf{w}):=L(\mathbf{w};\mathcal D_r)\), with gradients \(\mathbf{g}_f(\mathbf{w}):=\nabla_{\mathbf{w}}L_f(\mathbf{w})\) and \(\mathbf{g}_r(\mathbf{w}):=\nabla_{\mathbf{w}}L_r(\mathbf{w})\).

For nonempty mini-batches \(\mathcal B_f\subseteq\mathcal D_f\) and \(\mathcal B_r\subseteq\mathcal D_r\), write \(L_f(\mathbf{w};\mathcal B_f):=L(\mathbf{w};\mathcal B_f)\) and \(L_r(\mathbf{w};\mathcal B_r):=L(\mathbf{w};\mathcal B_r)\). We use \(\mathbf I\) for the identity matrix, \(\|\cdot\|_2\) for the Euclidean norm, and \(\|\cdot\|_{\mathrm{op}}\) for the induced operator norm. When both gradients are nonzero, define their coupling by
\begin{equation*}
\begin{aligned}
\cos\theta(\mathbf{w}) &:= \frac{\mathbf{g}_f(\mathbf{w})^\top \mathbf{g}_r(\mathbf{w})}{\|\mathbf{g}_f(\mathbf{w})\|_2 \, \|\mathbf{g}_r(\mathbf{w})\|_2},\qquad \theta(\mathbf{w})\in[0,\pi].
\end{aligned}
\end{equation*}

\subsection{Min--max unlearning}
\label{sec:update}

Fine-tuning (FT)~\cite{WarneckePirchWressnegger2023_1000166047,scrubbing} minimizes \(L_r\), while negative gradient (NG)~\cite{scrubbing} maximizes \(L_f\). We study updates that evaluate retention at a point selected by a forget-seeking perturbation~\cite{kim2025unlearningaware}. To isolate surrogate selection, let \(\rho>0\) be a perturbation radius and write the bilevel template as follows:
\begin{equation*}
\begin{aligned}
&\min\limits_{\mathbf{w}}\; L_r\bigl(\mathbf{w}+\boldsymbol{\delta}^{\star}(\mathbf{w})\bigr)\quad\text{s.t.}\quad\boldsymbol{\delta}^{\star}(\mathbf{w})\in\argmax\limits_{\|\boldsymbol{\delta}\|_2\le\rho}\; L_f(\mathbf{w}+\boldsymbol{\delta}).
\end{aligned}
\end{equation*}
Here \(\boldsymbol{\delta}^{\star}(\mathbf w)\) denotes a selected inner maximizer, and \(\mathbf w+\boldsymbol{\delta}^{\star}(\mathbf w)\) is the surrogate point. The template keeps only the part of min--max unlearning~\cite{kim2025unlearningaware} that chooses the surrogate point. For a perturbation \(\boldsymbol{\delta}\), we call \(\mathbf g_f(\mathbf w)^\top\boldsymbol{\delta}\) its forget gain, the first-order increase of the forget loss. For \(\mathbf g_f(\mathbf w)\neq\mathbf{0}\), maximizing the forget gain over the fixed Euclidean ball gives the standard perturbation
\begin{equation*}
\begin{aligned}
\boldsymbol{\delta}_{\mathrm{std}}(\mathbf{w}) &= \rho\,\frac{\mathbf{g}_f(\mathbf{w})}{\|\mathbf{g}_f(\mathbf{w})\|_2}.
\end{aligned}
\end{equation*}
We use \(\boldsymbol{\delta}_{\mathrm{std}}\) as the fixed-budget comparator. For \(\mathbf g_r(\mathbf w)\neq\mathbf{0}\), its surrogate retain objective and first-order retain effect are
\begin{equation}
\label{eq:std-retain-effect}
\begin{aligned}
F_{\mathrm{std}}(\mathbf w)&:=L_r\bigl(\mathbf w+\boldsymbol{\delta}_{\mathrm{std}}(\mathbf w)\bigr), \qquad \mathbf{g}_r(\mathbf{w})^\top\boldsymbol{\delta}_{\mathrm{std}}(\mathbf{w}) = \rho\,\|\mathbf{g}_r(\mathbf{w})\|_2\cos\theta(\mathbf{w}).
\end{aligned}
\end{equation}
Positive coupling makes this term positive, so the surrogate point incurs a first-order retain cost before the retain step. The inner perturbation should therefore satisfy \(\mathbf g_r(\mathbf w)^\top\boldsymbol{\delta}=0\) while retaining as much forget gain as possible under the same budget.

\subsection{Retain-neutral inner problem}
\label{sec:inner}

\begin{definition}[Retain-neutral inner problem]
\label{def:retain-neutral}
Let \(\mathbf{g}_r(\mathbf{w})\neq\mathbf{0}\). A perturbation \(\boldsymbol{\delta}\) is retain-neutral at \(\mathbf w\) if \(\mathbf g_r(\mathbf w)^\top\boldsymbol{\delta}=0\). Among retain-neutral perturbations of norm at most \(\rho\), the retain-neutral inner problem selects one with the largest forget gain,
\begin{equation}
    \label{eq:rosu-inner}
\begin{aligned}
&\max\limits_{\boldsymbol{\delta}}\;\mathbf g_f(\mathbf w)^\top\boldsymbol{\delta}\quad\text{s.t.}\quad\|\boldsymbol{\delta}\|_2\le\rho,\quad\mathbf g_r(\mathbf w)^\top\boldsymbol{\delta}=0.
\end{aligned}
\end{equation}
\end{definition}
The feasible set is the intersection of the perturbation ball and the retain-neutral hyperplane. Write
\begin{equation*}
\begin{aligned}
\mathbf{u}_r(\mathbf{w}) &:=\frac{\mathbf{g}_r(\mathbf{w})}{\|\mathbf{g}_r(\mathbf{w})\|_2},\qquad \mathbf{P}_r(\mathbf{w}):=\mathbf{u}_r(\mathbf{w})\mathbf{u}_r(\mathbf{w})^\top,\qquad \mathbf{q}(\mathbf{w}):=(\mathbf{I}-\mathbf{P}_r(\mathbf{w}))\mathbf{g}_f(\mathbf{w}).
\end{aligned}
\end{equation*}
Here \(\mathbf u_r\) is the unit retain direction, \(\mathbf P_r\) projects onto its span, and \(\mathbf q\) is the forget component orthogonal to that span. The next proposition shows that the constrained inner problem has a unique closed-form solution. The proof is in Appendix~\ref{app:core-geometry}.

\begin{proposition}[Closed-form solution of the inner problem]
\label{prop:inner}
Assume \(\mathbf{g}_r(\mathbf{w})\neq \mathbf{0}\) and \(\mathbf{q}(\mathbf{w})\neq \mathbf{0}\). Then the problem in~\eqref{eq:rosu-inner} admits the unique maximizer
\begin{equation}
\label{eq:rosu-solution}
\begin{aligned}
\boldsymbol{\delta}_{\mathrm{rosu}}(\mathbf{w}) &= \rho\,\frac{\mathbf{q}(\mathbf{w})}{\|\mathbf{q}(\mathbf{w})\|_2},
\end{aligned}
\end{equation}
with optimal value \(\rho \|\mathbf{q}(\mathbf{w})\|_2\).
\end{proposition}

The maximizer removes from the forget gradient its component along the retain gradient and spends the whole budget on the remainder. When \(\mathbf q(\mathbf w)=\mathbf{0}\), every feasible perturbation has zero forget gain, and the normalized construction is undefined.

Retain neutrality scales the optimal forget gain by a factor that depends only on the angle between the two gradients. The constrained gain \(\rho\|\mathbf q\|_2\) equals the unconstrained gain \(\rho\|\mathbf g_f\|_2\) times \(\sin\theta\) (Appendix~\ref{app:core-geometry}). The relative loss is negligible near orthogonality and grows as \(|\cos\theta(\mathbf{w})|\) approaches one, while the standard perturbation incurs a positive first-order retain cost whenever \(\cos\theta(\mathbf w)>0\).

\subsection{Retain-loss guarantees}
\label{sec:local}

A retain-neutral perturbation can still change the retain loss through curvature. To quantify this effect, define the surrogate retain objective
\begin{equation*}
\begin{aligned}
F_{\mathrm{rosu}}(\mathbf w)&:=L_r\bigl(\mathbf w+\boldsymbol{\delta}_{\mathrm{rosu}}(\mathbf w)\bigr).
\end{aligned}
\end{equation*}
This objective supplies the retain component of the full update in Section~\ref{sec:correction}. The next theorem connects its value to the retain loss at the current weights \(\mathbf w\) and compares it with the standard surrogate. Its proof expands both surrogates to second order under a Hessian bound on the two perturbation segments (Appendix~\ref{app:core-geometry}). The ideal construction gives these bounds, and Appendices~\ref{app:regularization} and~\ref{app:retain-batch-size} quantify the residuals introduced by regularization and mini-batch estimation.

\begin{theorem}[Retain properties of the surrogate]
\label{thm:retain-damage}
Let \(\mathbf{g}_f(\mathbf{w})\neq\mathbf{0}\), \(\mathbf{g}_r(\mathbf{w})\neq\mathbf{0}\), and \(\mathbf{q}(\mathbf{w})\neq\mathbf{0}\). Assume that \(L_r\) is \(C^2\) on an open neighborhood of the two segments \(\{\mathbf{w}+\zeta\boldsymbol{\delta}_{\mathrm{std}}(\mathbf{w}):\zeta\in[0,1]\}\) and \(\{\mathbf{w}+\zeta\boldsymbol{\delta}_{\mathrm{rosu}}(\mathbf{w}):\zeta\in[0,1]\}\), and that \(\|\nabla^2L_r(\mathbf{z})\|_{\mathrm{op}}\le M_r\) for every \(\mathbf{z}\) on these segments. Then
\begin{enumerate}[(i)]
\item\label{thm:retain-damage-i} \(\mathbf{g}_r(\mathbf{w})^\top\boldsymbol{\delta}_{\mathrm{rosu}}(\mathbf{w})=0\), and
\begin{equation*}
\begin{aligned}
\bigl|F_{\mathrm{rosu}}(\mathbf{w})-L_r(\mathbf{w})\bigr| &\le \frac{M_r}{2}\rho^2 .
\end{aligned}
\end{equation*}
\item\label{thm:retain-damage-ii} The surrogate retain objectives of the two perturbations satisfy
\begin{equation*}
\begin{aligned}
F_{\mathrm{std}}(\mathbf{w})-F_{\mathrm{rosu}}(\mathbf{w}) &\ge \rho\,\|\mathbf{g}_r(\mathbf{w})\|_2\cos\theta(\mathbf{w})-M_r\rho^2 .
\end{aligned}
\end{equation*}
\end{enumerate}
\end{theorem}

By part~\itemref{thm:retain-damage-i} of Theorem~\ref{thm:retain-damage}, the retain loss at the surrogate point departs from the retain loss at \(\mathbf w\) only through curvature, because the constraint removes the linear term. In part~\itemref{thm:retain-damage-ii}, the first term is the retain cost that coupling induces in the standard surrogate, and the second bounds the difference between the curvature terms of both surrogates. The comparison between the two surrogates depends on the coupling measured in Figure~\ref{fig:motivation_a}. Under positive coupling the first term favors ROSU, and Appendix~\ref{app:core-geometry} states when it outweighs the second. Under negative coupling the standard perturbation has a negative first-order retain effect. Near orthogonality the two perturbations are close.

\begin{algorithm}[t]
\caption{Retain-Orthogonal Surrogate Unlearning (ROSU)}
\label{alg:rosu_full}
\begin{algorithmic}[1]
\Input weights \(\mathbf{w}\), nonempty forget mini-batch \(\mathcal{B}_f\) and retain mini-batch \(\mathcal{B}_r\)
\Parameters learning rate \(\eta>0\), perturbation radius \(\rho>0\), amplification coefficient \(\beta\ge0\), stabilizer \(\tau>0\), threshold \(\varepsilon_q>0\)
\Output updated weights \(\mathbf{w}\)
\State \(\mathbf{g}_f \gets \nabla L_f(\mathbf{w};\mathcal{B}_f)\), \(\quad \mathbf{g}_r \gets \nabla L_r(\mathbf{w};\mathcal{B}_r)\)
\State \(\mathbf{q}_\tau \gets \mathbf{g}_f - \frac{\mathbf{g}_f^\top \mathbf{g}_r}{\|\mathbf{g}_r\|_2^2+\tau}\,\mathbf{g}_r\)
\If{\(\|\mathbf{q}_\tau\|_2 \le \varepsilon_q\)}
    \State \(\mathbf{w} \gets \mathbf{w} - \eta\,\mathbf{g}_r\) \Comment{Degenerate fallback, Section~\ref{sec:fallback}}
\Else
    \State \(\mathbf{u}_\tau \gets \frac{\mathbf{q}_\tau}{\|\mathbf{q}_\tau\|_2}\), \(\quad \boldsymbol{\delta}_\tau \gets \rho\,\mathbf{u}_\tau\), \(\quad \alpha_\tau \gets \frac{\rho}{\|\mathbf{q}_\tau\|_2}\)
    \State \(\widetilde{\mathbf{g}}_{r,\tau} \gets \left.\nabla_{\mathbf z}L_r(\mathbf z;\mathcal B_r)\right|_{\mathbf z=\mathbf w+\boldsymbol{\delta}_\tau}\)
    \State \(\mathbf{v} \gets \widetilde{\mathbf{g}}_{r,\tau} + \alpha_\tau\!\left(\widetilde{\mathbf{g}}_{r,\tau} - \frac{\mathbf{g}_r^\top \widetilde{\mathbf{g}}_{r,\tau}}{\|\mathbf{g}_r\|_2^2+\tau}\,\mathbf{g}_r - (\mathbf{u}_\tau^\top \widetilde{\mathbf{g}}_{r,\tau})\,\mathbf{u}_\tau\right)\)
    \State \(\mathbf{w} \gets \mathbf{w} + \beta\,\boldsymbol{\delta}_\tau - \eta\,\mathbf{v}\)
\EndIf
\end{algorithmic}
\end{algorithm}

\section{Retain-Orthogonal Surrogate Unlearning}
\label{sec:method}
\label{sec:correction}

Proposition~\ref{prop:inner} identifies the direction of maximal forget gain under retain neutrality. Theorem~\ref{thm:retain-damage} bounds the retain loss at the resulting surrogate point. Let \(\mathbf w^{+}\) denote the next iterate, \(\eta>0\) the learning rate, \(\beta\ge0\) the amplification coefficient, and \(\mathbf v(\mathbf w)\) the retain correction defined below. ROSU uses the same inner solution for retain correction and retain-neutral amplification and combines them in one update,
\begin{equation*}
\begin{aligned}
\mathbf{w}^{+}&=\mathbf{w}+\beta\,\boldsymbol{\delta}_{\mathrm{rosu}}(\mathbf{w})-\eta\,\mathbf{v}(\mathbf{w}).
\end{aligned}
\end{equation*}
The correction \(\mathbf v(\mathbf w)\) transports the retain gradient from the surrogate point back to \(\mathbf w\). The retain-neutral amplification term \(\beta\boldsymbol{\delta}_{\mathrm{rosu}}\) supplies the largest forget gain under its retain-neutral budget (Section~\ref{sec:inner}).

\subsection{Computing the retain correction}

The gradient of the surrogate retain objective must account for the dependence of the surrogate point on \(\mathbf w\). For \(L_f,L_r\in C^2\), \(\mathbf g_r\neq\mathbf{0}\), and \(\mathbf q\neq\mathbf{0}\), set \(\widetilde{\mathbf{g}}_r(\mathbf{w}):=\left.\nabla_{\mathbf z}L_r(\mathbf z)\right|_{\mathbf z=\mathbf w+\boldsymbol{\delta}_{\mathrm{rosu}}(\mathbf w)}\). Let \(\mathbf J_{\boldsymbol{\delta}_{\mathrm{rosu}}}\) denote the Jacobian of the perturbation map with \(\rho\) fixed. The chain rule gives
\begin{equation*}
\begin{aligned}
\nabla F_{\mathrm{rosu}}(\mathbf{w})&=\bigl(\mathbf{I}+\mathbf{J}_{\boldsymbol{\delta}_{\mathrm{rosu}}}(\mathbf{w})^\top\bigr)\widetilde{\mathbf{g}}_r(\mathbf{w}).
\end{aligned}
\end{equation*}
Here \(\widetilde{\mathbf g}_r\) is the retain gradient at the ideal surrogate point. This Jacobian contains the forget Hessian and the differential of \(\mathbf P_r\) (Appendix~\ref{app:exact-relaxed}). Both the surrogate point and its Jacobian depend on the forget data through \(\boldsymbol{\delta}_{\mathrm{rosu}}\). The retain-neutral amplification term \(\beta\boldsymbol{\delta}_{\mathrm{rosu}}\) scales this \(\boldsymbol{\delta}_{\mathrm{rosu}}\), so the retain correction and the amplification direction come from the same inner solution.

Relaxations A1 and A2 make this correction tractable. A1 drops the projector differential from the derivative computation. A2 replaces the forget Hessian \(\nabla^2L_f(\mathbf w)\) by \(\mathbf I\), using an identity-Hessian approximation previously considered in surrogate-based unlearning~\cite{kim2025unlearningaware}. The retain-neutral constraint is unchanged, and the projector is recomputed at each step. Set \(\mathbf u_\perp:=\frac{\mathbf q}{\|\mathbf q\|_2}\) and \(\mathbf P_\perp:=\mathbf u_\perp\mathbf u_\perp^\top\), with all quantities evaluated at \(\mathbf w\). The relaxed Jacobian and correction are
\begin{equation}
\label{eq:rosu-update}
\begin{aligned}
\widehat{\mathbf{J}}_{\boldsymbol{\delta}_{\mathrm{rosu}}}(\mathbf{w})&=\frac{\rho}{\|\mathbf{q}(\mathbf{w})\|_2}\bigl(\mathbf{I}-\mathbf{P}_r(\mathbf{w})-\mathbf{P}_\perp(\mathbf{w})\bigr),\qquad \mathbf{v}:=\bigl(\mathbf{I}+\widehat{\mathbf{J}}_{\boldsymbol{\delta}_{\mathrm{rosu}}}(\mathbf{w})\bigr)\widetilde{\mathbf{g}}_r(\mathbf{w}).
\end{aligned}
\end{equation}
The Jacobian correction projects \(\widetilde{\mathbf g}_r\) onto the complement of the retain and perturbation directions and scales it by \(\frac{\rho}{\|\mathbf q\|_2}\) before adding it to \(\widetilde{\mathbf g}_r\). The relaxed Jacobian is symmetric, so \(\mathbf v\) is the chain-rule gradient with \(\widehat{\mathbf J}_{\boldsymbol{\delta}_{\mathrm{rosu}}}\) in place of \(\mathbf J_{\boldsymbol{\delta}_{\mathrm{rosu}}}\). For this ideal unregularized map, the added term is orthogonal to both gradients at the current weights, so it leaves their first-order projections unchanged.

The next proposition bounds the deviation of the ideal unregularized A1/A2 correction from the exact chain-rule gradient. The bound depends on how far the forget Hessian is from the identity and on the projector differential \(\mathrm D\mathbf P_r(\mathbf w)[\boldsymbol{\xi}]\) along a direction \(\boldsymbol{\xi}\). The proof is in Appendix~\ref{app:exact-relaxed}.

\begin{proposition}[Deviation of the relaxed correction]
\label{prop:approx-grad}
Assume \(L_f,L_r\in C^2\), \(\mathbf{g}_r(\mathbf{w})\neq\mathbf{0}\), and \(\mathbf{q}(\mathbf{w})\neq\mathbf{0}\), and let \(\varepsilon_H,\varepsilon_P\ge0\) satisfy
\[
\begin{aligned}
\|\nabla^2L_f(\mathbf{w})-\mathbf{I}\|_{\mathrm{op}}&\le\varepsilon_H,\qquad \sup\limits_{\|\boldsymbol{\xi}\|_2=1}\|\mathrm{D}\mathbf{P}_r(\mathbf{w})[\boldsymbol{\xi}]\|_{\mathrm{op}}\le\varepsilon_P .
\end{aligned}
\]
Then
\begin{subequations}
\label{eq:approx-grad-bound}
\begin{align}
\bigl\|\mathbf{J}_{\boldsymbol{\delta}_{\mathrm{rosu}}}(\mathbf{w})-\widehat{\mathbf{J}}_{\boldsymbol{\delta}_{\mathrm{rosu}}}(\mathbf{w})\bigr\|_{\mathrm{op}}
&\le\frac{\rho}{\|\mathbf{q}(\mathbf{w})\|_2}\bigl(\varepsilon_H+\varepsilon_P\|\mathbf{g}_f(\mathbf{w})\|_2\bigr),\label{eq:approx-grad-bound-jac}\\
\bigl\|\nabla F_{\mathrm{rosu}}(\mathbf{w})-\mathbf{v}\bigr\|_2
&\le\frac{\rho\,\|\widetilde{\mathbf{g}}_r(\mathbf{w})\|_2}{\|\mathbf{q}(\mathbf{w})\|_2}\bigl(\varepsilon_H+\varepsilon_P\|\mathbf{g}_f(\mathbf{w})\|_2\bigr).\label{eq:approx-grad-bound-grad}
\end{align}
\end{subequations}
\end{proposition}

The bound shrinks as the forget Hessian approaches the identity and as the projector varies slowly. When the deviation is smaller than \(\|\mathbf v\|_2\), \(-\mathbf v\) is a descent direction of \(F_{\mathrm{rosu}}\) (Appendix~\ref{app:exact-relaxed}).

\subsection{Retain-neutral amplification}

Let \(\Delta\mathbf w:=\beta\boldsymbol{\delta}_{\mathrm{rosu}}-\eta\mathbf v\) be the ideal nondegenerate displacement. The identities \(\mathbf g_r^\top\mathbf q=0\) and \(\mathbf g_f^\top\mathbf q=\|\mathbf q\|_2^2\) give
\begin{equation}
\label{eq:full-update-identities}
\begin{aligned}
\mathbf{g}_r(\mathbf{w})^\top\Delta\mathbf{w}&=-\eta\,\mathbf{g}_r(\mathbf{w})^\top\mathbf{v},\qquad \mathbf{g}_f(\mathbf{w})^\top\Delta\mathbf{w}=\beta\rho\,\|\mathbf{q}(\mathbf{w})\|_2-\eta\,\mathbf{g}_f(\mathbf{w})^\top\mathbf{v}.
\end{aligned}
\end{equation}

For the ideal displacement in~\eqref{eq:full-update-identities}, retain-neutral amplification adds no first-order retain change. Its forget contribution \(\beta\rho\|\mathbf q\|_2\) is optimal among retain-neutral displacements of norm at most \(\beta\rho\). For \(\beta>0\), this follows from Proposition~\ref{prop:inner}. For \(\beta=0\), the only such displacement is zero.

The forgetting step \(\beta\boldsymbol{\delta}_{\mathrm{rosu}}\) has length \(\beta\rho\), so the radius \(\rho\) also scales it. At iteration \(t\), we therefore write the amplification level as \(\beta_t=\frac{c_t\eta_t}{\rho}\), where \(\eta_t\) is the learning rate and \(c_t\ge0\) is a scale factor. On ideal nondegenerate steps the forgetting step is then \(\beta_t\boldsymbol{\delta}_{\mathrm{rosu}}=c_t\eta_t\frac{\mathbf q}{\|\mathbf q\|_2}\), so its length \(c_t\eta_t\) is independent of \(\rho\), while the surrogate point and correction still depend on \(\rho\). The update then reads \(\mathbf w^{+}=\mathbf w-\eta_t\bigl(\mathbf v-c_t\frac{\mathbf q}{\|\mathbf q\|_2}\bigr)\), so \(c_t\) weights the forgetting step against the retain correction. In our experiments, \(c_t\) is chosen so that the step length is either held fixed or largest in the early epochs and smaller afterwards. The schedule and level are selected together with \(\rho\) and \(\eta\) (Appendix~\ref{app:beta-tuning}).

\subsection{Regularization and fallback}
\label{sec:fallback}

Algorithms~\ref{alg:rosu_full}--\ref{alg:rosu_repr} give the update and its fallback with mini-batch gradients and a regularized retain operator. In the algorithms, \(\mathbf g_f\) and \(\mathbf g_r\) denote gradients on the indicated mini-batches, and the optimizer settings are in Appendix~\ref{app:impl}.

For a stabilizer \(\tau>0\), Algorithm~\ref{alg:rosu_full} uses \(\mathbf P_r^{(\tau)}:=\frac{\mathbf g_r\mathbf g_r^\top}{\|\mathbf g_r\|_2^2+\tau}\) and \(\mathbf q_\tau=(\mathbf I-\mathbf P_r^{(\tau)})\mathbf g_f\). For \(\mathbf q_\tau\neq\mathbf{0}\), set \(\mathbf u_\tau:=\frac{\mathbf q_\tau}{\|\mathbf q_\tau\|_2}\), \(\boldsymbol{\delta}_\tau:=\rho\mathbf u_\tau\), and \(\alpha_\tau:=\frac{\rho}{\|\mathbf q_\tau\|_2}\).

Algorithm~\ref{alg:rosu_full} tests \(\|\mathbf q_\tau\|_2\le\varepsilon_q\), where \(\varepsilon_q>0\) is the degeneracy threshold. For \(\mathbf g_r\neq\mathbf{0}\), orthogonal decomposition gives \(\|\mathbf q\|_2\le\|\mathbf q_\tau\|_2\). The test therefore bounds the optimal forget gain in~\eqref{eq:rosu-inner} by \(\rho\varepsilon_q\). The fallback takes a retain-gradient step and forgoes at most this gain.

The algorithm evaluates \(\widetilde{\mathbf g}_{r,\tau}:=\left.\nabla_{\mathbf z}L_r(\mathbf z;\mathcal B_r)\right|_{\mathbf z=\mathbf w+\boldsymbol{\delta}_\tau}\) on steps with \(\|\mathbf q_\tau\|_2>\varepsilon_q\) and uses the symmetric relaxed Jacobian of~\eqref{eq:rosu-update} at the regularized direction. Appendix~\ref{app:regularization} bounds the resulting differences from the ideal construction, which are \(O(\tau)\) as \(\tau\to0\) at any \(\mathbf w\) with \(\mathbf g_r\neq\mathbf{0}\) and \(\mathbf q\neq\mathbf{0}\). All experiments fix \(\tau=10^{-8}\) and \(\varepsilon_q=10^{-6}\).

\section{Experiments}
\label{sec:experiments}

We ask whether ROSU comes closest to retraining on vision benchmarks (Section~\ref{sec:exp-cv}) and is competitive with the baselines on LLM unlearning (Section~\ref{sec:llm}), and which components matter (Section~\ref{sec:ablation}). Our main benchmarks are CIFAR-10, CIFAR-100~\cite{cifar}, and Tiny-ImageNet~\cite{TinyImageNet} for vision, and TOFU~\cite{maini2024tofu} and WMDP~\cite{li2024the} for language. Tables~\ref{tab:cifar10_unlearning}--\ref{tab:llm} and Table~\ref{tab:tinyimagenet_unlearning} in Appendix~\ref{app:remaining-benchmark-table} report the main results, and Figures~\ref{fig:qualitative} and~\ref{fig:a1a2_ablation} show qualitative TOFU examples and an ablation of the A1/A2 relaxations. Additional results and extended ablations are provided in Appendix~\ref{app:addl-results}. Unless noted otherwise, all results use the full ROSU update in Algorithm~\ref{alg:rosu_full}, and WMDP uses the representation-level variant in Algorithm~\ref{alg:rosu_repr}.

\begin{table*}[t]
\centering
\caption{\textbf{Unlearning on CIFAR-10.} Values in parentheses are absolute differences from Retrain, and standard deviations are over ten forgotten classes for \cw{class-wise} forgetting and three seeds for \rd{random} forgetting. In each protocol, bold and underline mark the smallest and second-smallest \(\Delta\mathrm{Acc}\) among the unlearning methods.}
\label{tab:cifar10_unlearning}
{\scriptsize
\setlength{\tabcolsep}{2pt}
\begin{tabular*}{\textwidth}{@{}l@{\extracolsep{\fill}}cccccc@{}}
\toprule
\textbf{Method} & \textbf{RA} & \textbf{FA} & \textbf{TA} & \textbf{$\Delta$Acc.}{\scriptsize$\downarrow$} & \textbf{MIA-Eff.} & \textbf{Time (min)} \\
\midrulebeforeband

\rowcolor{cwBand}
\multicolumn{7}{c}{\bandfill{\textbf{\cw{Class-wise forgetting}}}} \\
\midruleafterband

Retrain
& \stdnum{99.99}{0.00\phantom{0}}\phantom{\ (00.00)}
& \phantom{0}\stdnum{0.00}{0.00\phantom{0}}\phantom{\ (00.00)}
& \stdnum{95.15}{0.56\phantom{0}}\phantom{\ (00.00)}
& \phantom{00}0.00
& \stdnum{100.00}{0.00\phantom{0}}\phantom{\ (00.00)}
& 33.8\phantom{0} \\

FT
& \stdnum{98.22}{0.54\phantom{0}}\ \cwdiff{1.77}\phantom{0}
& \stdnum{24.94}{6.99\phantom{0}}\ \cwdiff{24.94}
& \stdnum{92.68}{1.07\phantom{0}}\ \cwdiff{2.47}\phantom{0}
& \phantom{0}29.18
& \phantom{0}\stdnum{99.13}{0.98\phantom{0}}\ \cwdiff{0.87}\phantom{0}
& \phantom{0}1.63 \\

NG
& \stdnum{88.03}{2.92\phantom{0}}\ \cwdiff{11.96}
& \phantom{0}\stdnum{0.84}{1.21\phantom{0}}\ \cwdiff{0.84}\phantom{0}
& \stdnum{83.00}{2.89\phantom{0}}\ \cwdiff{12.15}
& \phantom{0}24.95
& \phantom{0}\stdnum{99.28}{1.10\phantom{0}}\ \cwdiff{0.72}\phantom{0}
& \phantom{0}1.64 \\

FF
& \stdnum{98.95}{0.12\phantom{0}}\ \cwdiff{1.04}\phantom{0}
& \stdnum{10.98}{11.29}\ \cwdiff{10.98}
& \stdnum{93.26}{0.62\phantom{0}}\ \cwdiff{1.89}\phantom{0}
& \phantom{0}13.91
& \stdnum{100.00}{0.00\phantom{0}}\ \cwdiff{0.00}\phantom{0}
& \phantom{0}1.07 \\

IU
& \stdnum{90.86}{9.85\phantom{0}}\ \cwdiff{9.13}\phantom{0}
& \stdnum{10.28}{15.62}\ \cwdiff{10.28}
& \stdnum{85.86}{8.80\phantom{0}}\ \cwdiff{9.29}\phantom{0}
& \phantom{0}28.70
& \phantom{0}\stdnum{98.27}{3.09\phantom{0}}\ \cwdiff{1.73}\phantom{0}
& \phantom{0}0.59 \\

$\ell_{1}$-sparse
& \stdnum{97.18}{0.46\phantom{0}}\ \cwdiff{2.81}\phantom{0}
& \phantom{0}\stdnum{0.00}{0.00\phantom{0}}\ \cwdiff{0.00}\phantom{0}
& \stdnum{92.07}{0.76\phantom{0}}\ \cwdiff{3.08}\phantom{0}
& \phantom{00}5.89
& \stdnum{100.00}{0.00\phantom{0}}\ \cwdiff{0.00}\phantom{0}
& \phantom{0}1.73 \\

GU
& \stdnum{94.20}{2.36\phantom{0}}\ \cwdiff{5.79}\phantom{0}
& \phantom{0}\stdnum{6.48}{3.72\phantom{0}}\ \cwdiff{6.48}\phantom{0}
& \stdnum{88.72}{2.19\phantom{0}}\ \cwdiff{6.43}\phantom{0}
& \phantom{0}18.70
& \phantom{0}\stdnum{95.19}{3.10\phantom{0}}\ \cwdiff{4.81}\phantom{0}
& \phantom{0}7.59 \\

OrthoGrad
& \stdnum{98.59}{0.72\phantom{0}}\ \cwdiff{1.40}\phantom{0}
& \phantom{0}\stdnum{0.46}{0.77\phantom{0}}\ \cwdiff{0.46}\phantom{0}
& \stdnum{92.31}{0.92\phantom{0}}\ \cwdiff{2.84}\phantom{0}
& \phantom{00}4.70
& \phantom{0}\stdnum{99.66}{0.66\phantom{0}}\ \cwdiff{0.34}\phantom{0}
& \phantom{0}7.62 \\

PCGrad
& \stdnum{98.58}{0.32\phantom{0}}\ \cwdiff{1.41}\phantom{0}
& \phantom{0}\stdnum{0.07}{0.12\phantom{0}}\ \cwdiff{0.07}\phantom{0}
& \stdnum{92.93}{0.78\phantom{0}}\ \cwdiff{2.22}\phantom{0}
& \phantom{00}3.70
& \stdnum{100.00}{0.00\phantom{0}}\ \cwdiff{0.00}\phantom{0}
& \phantom{0}2.10 \\

UAM
& \stdnum{99.91}{0.05\phantom{0}}\ \cwdiff{0.08}\phantom{0}
& \phantom{0}\stdnum{0.02}{0.03\phantom{0}}\ \cwdiff{0.02}\phantom{0}
& \stdnum{94.42}{0.68\phantom{0}}\ \cwdiff{0.73}\phantom{0}
& \phantom{00}\underline{0.83}
& \stdnum{100.00}{0.00\phantom{0}}\ \cwdiff{0.00}\phantom{0}
& \phantom{0}2.47 \\

\textbf{ROSU}
& \stdnum{99.96}{0.01\phantom{0}}\ \cwdiff{0.03}\phantom{0}
& \phantom{0}\stdnum{0.00}{0.00\phantom{0}}\ \cwdiff{0.00}\phantom{0}
& \stdnum{95.01}{0.53\phantom{0}}\ \cwdiff{0.14}\phantom{0}
& \phantom{00}\textbf{0.17}
& \stdnum{100.00}{0.00\phantom{0}}\ \cwdiff{0.00}\phantom{0}
& \phantom{0}2.98 \\

\midrulebeforeband

\rowcolor{rdBand}
\multicolumn{7}{c}{\bandfill{\textbf{\rd{Random forgetting}}}} \\
\midruleafterband

Retrain
& \stdnum{99.99}{0.00\phantom{0}}\phantom{\ (00.00)}
& \stdnum{95.11}{0.37\phantom{0}}\phantom{\ (00.00)}
& \stdnum{94.38}{0.23\phantom{0}}\phantom{\ (00.00)}
& \phantom{00}0.00
& \phantom{0}\stdnum{11.61}{0.68\phantom{0}}\phantom{\ (00.00)}
& 33.8\phantom{0} \\

FT
& \stdnum{99.29}{0.06\phantom{0}}\ \rddiff{0.70}\phantom{0}
& \stdnum{99.06}{0.08\phantom{0}}\ \rddiff{3.95}\phantom{0}
& \stdnum{93.28}{0.15\phantom{0}}\ \rddiff{1.10}\phantom{0}
& \phantom{00}5.75
& \phantom{00}\stdnum{4.07}{0.30\phantom{0}}\ \rddiff{7.54}\phantom{0}
& \phantom{0}1.64 \\

NG
& \stdnum{64.30}{45.24}\ \rddiff{35.69}
& \stdnum{64.18}{45.23}\ \rddiff{30.93}
& \stdnum{60.74}{42.32}\ \rddiff{33.64}
& 100.26
& \phantom{00}\stdnum{4.27}{3.88\phantom{0}}\ \rddiff{7.34}\phantom{0}
& \phantom{0}1.63 \\

FF
& \stdnum{79.72}{7.37\phantom{0}}\ \rddiff{20.27}
& \stdnum{80.02}{7.12\phantom{0}}\ \rddiff{15.09}
& \stdnum{77.22}{5.93\phantom{0}}\ \rddiff{17.16}
& \phantom{0}52.52
& \phantom{0}\stdnum{40.24}{18.03}\ \rddiff{28.63}
& \phantom{0}1.16 \\

IU
& \stdnum{95.72}{5.57\phantom{0}}\ \rddiff{4.27}\phantom{0}
& \stdnum{95.67}{5.51\phantom{0}}\ \rddiff{0.56}\phantom{0}
& \stdnum{89.89}{5.37\phantom{0}}\ \rddiff{4.49}\phantom{0}
& \phantom{00}9.32
& \phantom{00}\stdnum{6.87}{7.00\phantom{0}}\ \rddiff{4.74}\phantom{0}
& \phantom{0}0.59 \\

$\ell_{1}$-sparse
& \stdnum{98.22}{0.51\phantom{0}}\ \rddiff{1.77}\phantom{0}
& \stdnum{96.43}{0.85\phantom{0}}\ \rddiff{1.32}\phantom{0}
& \stdnum{92.13}{0.52\phantom{0}}\ \rddiff{2.25}\phantom{0}
& \phantom{00}5.34
& \phantom{00}\stdnum{8.20}{0.48\phantom{0}}\ \rddiff{3.41}\phantom{0}
& \phantom{0}1.77 \\

GU
& \stdnum{98.60}{0.29\phantom{0}}\ \rddiff{1.39}\phantom{0}
& \stdnum{98.77}{0.49\phantom{0}}\ \rddiff{3.66}\phantom{0}
& \stdnum{92.96}{0.31\phantom{0}}\ \rddiff{1.42}\phantom{0}
& \phantom{00}6.47
& \phantom{00}\stdnum{3.78}{0.87\phantom{0}}\ \rddiff{7.83}\phantom{0}
& \phantom{0}7.59 \\

OrthoGrad
& \stdnum{99.97}{0.02\phantom{0}}\ \rddiff{0.02}\phantom{0}
& \stdnum{94.81}{0.82\phantom{0}}\ \rddiff{0.30}\phantom{0}
& \stdnum{92.65}{0.31\phantom{0}}\ \rddiff{1.73}\phantom{0}
& \phantom{00}\underline{2.05}
& \phantom{00}\stdnum{7.81}{1.12\phantom{0}}\ \rddiff{3.80}\phantom{0}
& \phantom{0}7.60 \\

PCGrad
& \stdnum{98.48}{0.30\phantom{0}}\ \rddiff{1.51}\phantom{0}
& \stdnum{95.57}{0.41\phantom{0}}\ \rddiff{0.46}\phantom{0}
& \stdnum{92.27}{0.41\phantom{0}}\ \rddiff{2.11}\phantom{0}
& \phantom{00}4.08
& \phantom{00}\stdnum{8.85}{0.50\phantom{0}}\ \rddiff{2.76}\phantom{0}
& \phantom{0}1.99 \\

UAM
& \stdnum{99.79}{0.05\phantom{0}}\ \rddiff{0.20}\phantom{0}
& \stdnum{96.25}{0.22\phantom{0}}\ \rddiff{1.14}\phantom{0}
& \stdnum{93.29}{0.28\phantom{0}}\ \rddiff{1.09}\phantom{0}
& \phantom{00}2.43
& \phantom{00}\stdnum{8.02}{0.22\phantom{0}}\ \rddiff{3.59}\phantom{0}
& \phantom{0}2.44 \\

\textbf{ROSU}
& \stdnum{99.84}{0.03\phantom{0}}\ \rddiff{0.15}\phantom{0}
& \stdnum{95.13}{0.23\phantom{0}}\ \rddiff{0.02}\phantom{0}
& \stdnum{94.29}{0.09\phantom{0}}\ \rddiff{0.09}\phantom{0}
& \phantom{00}\textbf{0.26}
& \phantom{00}\stdnum{7.88}{0.16\phantom{0}}\ \rddiff{3.73}\phantom{0}
& \phantom{0}2.92 \\

\bottomrule
\end{tabular*}
}
\end{table*}

\begin{table*}[t]
\centering
\caption{\textbf{Unlearning on CIFAR-100.} Same notation as Table~\ref{tab:cifar10_unlearning}.}
\label{tab:cifar100_unlearning}
{\scriptsize
\setlength{\tabcolsep}{2pt}
\begin{tabular*}{\textwidth}{@{}l@{\extracolsep{\fill}}cccccc@{}}
\toprule
\textbf{Method} & \textbf{RA} & \textbf{FA} & \textbf{TA} & \textbf{$\Delta$Acc.}{\scriptsize$\downarrow$} & \textbf{MIA-Eff.} & \textbf{Time (min)} \\
\midrulebeforeband

\rowcolor{cwBand}
\multicolumn{7}{c}{\bandfill{\textbf{\cw{Class-wise forgetting}}}} \\
\midruleafterband

Retrain
& \stdnum{99.98}{0.00\phantom{0}}\phantom{\ (00.00)}
& \phantom{0}\stdnum{0.00}{0.00\phantom{0}}\phantom{\ (00.00)}
& \stdnum{77.80}{0.19\phantom{0}}\phantom{\ (00.00)}
& \phantom{00}0.00
& \stdnum{100.00}{0.00\phantom{0}}\phantom{\ (00.00)}
& 36.8\phantom{0} \\

FT
& \stdnum{99.98}{0.00\phantom{0}}\ \cwdiff{0.00}\phantom{0}
& \stdnum{84.44}{5.72\phantom{0}}\ \cwdiff{84.44}
& \stdnum{78.34}{0.22\phantom{0}}\ \cwdiff{0.54}\phantom{0}
& \phantom{0}84.98
& \phantom{0}\stdnum{99.42}{0.47\phantom{0}}\ \cwdiff{0.58}\phantom{0}
& \phantom{0}1.79 \\

NG
& \stdnum{53.58}{14.54}\ \cwdiff{46.40}
& \stdnum{10.12}{29.96}\ \cwdiff{10.12}
& \stdnum{38.83}{15.37}\ \cwdiff{38.97}
& \phantom{0}95.49
& \phantom{0}\stdnum{89.78}{30.26}\ \cwdiff{10.22}
& \phantom{0}1.82 \\

FF
& \stdnum{99.98}{0.00\phantom{0}}\ \cwdiff{0.00}\phantom{0}
& \phantom{0}\stdnum{0.00}{0.00\phantom{0}}\ \cwdiff{0.00}\phantom{0}
& \stdnum{77.53}{0.15\phantom{0}}\ \cwdiff{0.27}\phantom{0}
& \phantom{00}\underline{0.27}
& \stdnum{100.00}{0.00\phantom{0}}\ \cwdiff{0.00}\phantom{0}
& 11.24 \\

IU
& \stdnum{99.08}{0.74\phantom{0}}\ \cwdiff{0.90}\phantom{0}
& \stdnum{10.28}{16.95}\ \cwdiff{10.28}
& \stdnum{73.83}{1.53\phantom{0}}\ \cwdiff{3.97}\phantom{0}
& \phantom{0}15.15
& \phantom{0}\stdnum{99.34}{1.30\phantom{0}}\ \cwdiff{0.66}\phantom{0}
& \phantom{0}0.57 \\

$\ell_{1}$-sparse
& \stdnum{99.98}{0.00\phantom{0}}\ \cwdiff{0.00}\phantom{0}
& \phantom{0}\stdnum{0.00}{0.00\phantom{0}}\ \cwdiff{0.00}\phantom{0}
& \stdnum{74.92}{0.38\phantom{0}}\ \cwdiff{2.88}\phantom{0}
& \phantom{00}2.88
& \stdnum{100.00}{0.00\phantom{0}}\ \cwdiff{0.00}\phantom{0}
& \phantom{0}1.99 \\

GU
& \stdnum{90.34}{4.81\phantom{0}}\ \cwdiff{9.64}\phantom{0}
& \phantom{0}\stdnum{1.28}{1.18\phantom{0}}\ \cwdiff{1.28}\phantom{0}
& \stdnum{66.46}{3.22\phantom{0}}\ \cwdiff{11.34}
& \phantom{0}22.26
& \phantom{0}\stdnum{98.90}{0.91\phantom{0}}\ \cwdiff{1.10}\phantom{0}
& \phantom{0}8.34 \\

OrthoGrad
& \stdnum{97.30}{2.53\phantom{0}}\ \cwdiff{2.68}\phantom{0}
& \phantom{0}\stdnum{1.68}{1.88\phantom{0}}\ \cwdiff{1.68}\phantom{0}
& \stdnum{71.70}{2.44\phantom{0}}\ \cwdiff{6.10}\phantom{0}
& \phantom{0}10.46
& \phantom{0}\stdnum{98.74}{1.17\phantom{0}}\ \cwdiff{1.26}\phantom{0}
& \phantom{0}8.18 \\

PCGrad
& \stdnum{99.24}{0.69\phantom{0}}\ \cwdiff{0.74}\phantom{0}
& \phantom{0}\stdnum{0.99}{0.87\phantom{0}}\ \cwdiff{0.99}\phantom{0}
& \stdnum{74.06}{1.66\phantom{0}}\ \cwdiff{3.74}\phantom{0}
& \phantom{00}5.47
& \stdnum{100.00}{0.00\phantom{0}}\ \cwdiff{0.00}\phantom{0}
& \phantom{0}2.43 \\

UAM
& \stdnum{99.97}{0.01\phantom{0}}\ \cwdiff{0.01}\phantom{0}
& \phantom{0}\stdnum{0.18}{0.25\phantom{0}}\ \cwdiff{0.18}\phantom{0}
& \stdnum{76.63}{0.73\phantom{0}}\ \cwdiff{1.17}\phantom{0}
& \phantom{00}1.36
& \stdnum{100.00}{0.00\phantom{0}}\ \cwdiff{0.00}\phantom{0}
& \phantom{0}2.78 \\

\textbf{ROSU}
& \stdnum{99.98}{0.00\phantom{0}}\ \cwdiff{0.00}\phantom{0}
& \phantom{0}\stdnum{0.12}{0.26\phantom{0}}\ \cwdiff{0.12}\phantom{0}
& \stdnum{77.78}{0.13\phantom{0}}\ \cwdiff{0.02}\phantom{0}
& \phantom{00}\textbf{0.14}
& \stdnum{100.00}{0.00\phantom{0}}\ \cwdiff{0.00}\phantom{0}
& \phantom{0}3.22 \\

\midrulebeforeband

\rowcolor{rdBand}
\multicolumn{7}{c}{\bandfill{\textbf{\rd{Random forgetting}}}} \\
\midruleafterband

Retrain
& \stdnum{99.98}{0.00\phantom{0}}\phantom{\ (00.00)}
& \stdnum{77.70}{0.26\phantom{0}}\phantom{\ (00.00)}
& \stdnum{76.66}{0.18\phantom{0}}\phantom{\ (00.00)}
& \phantom{00}0.00
& \phantom{0}\stdnum{47.39}{0.35\phantom{0}}\phantom{\ (00.00)}
& 33.9\phantom{0} \\

FT
& \stdnum{99.98}{0.00\phantom{0}}\ \rddiff{0.00}\phantom{0}
& \stdnum{99.96}{0.03\phantom{0}}\ \rddiff{22.26}
& \stdnum{78.38}{0.09\phantom{0}}\ \rddiff{1.72}\phantom{0}
& \phantom{0}23.98
& \phantom{0}\stdnum{18.20}{0.62\phantom{0}}\ \rddiff{29.19}
& \phantom{0}1.65 \\

NG
& \stdnum{32.71}{2.02\phantom{0}}\ \rddiff{67.27}
& \stdnum{31.53}{2.04\phantom{0}}\ \rddiff{46.17}
& \stdnum{25.79}{1.28\phantom{0}}\ \rddiff{50.87}
& 164.31
& \phantom{0}\stdnum{65.20}{1.81\phantom{0}}\ \rddiff{17.81}
& \phantom{0}1.67 \\

FF
& \phantom{0}\stdnum{1.08}{0.09\phantom{0}}\ \rddiff{98.90}
& \phantom{0}\stdnum{0.99}{0.09\phantom{0}}\ \rddiff{76.71}
& \phantom{0}\stdnum{1.06}{0.12\phantom{0}}\ \rddiff{75.60}
& 251.21
& \phantom{0}\stdnum{33.69}{46.21}\ \rddiff{13.70}
& 11.17 \\

IU
& \stdnum{97.51}{0.41\phantom{0}}\ \rddiff{2.47}\phantom{0}
& \stdnum{96.59}{0.32\phantom{0}}\ \rddiff{18.89}
& \stdnum{71.28}{0.40\phantom{0}}\ \rddiff{5.38}\phantom{0}
& \phantom{0}26.74
& \phantom{0}\stdnum{10.46}{0.81\phantom{0}}\ \rddiff{36.93}
& \phantom{0}0.59 \\

$\ell_{1}$-sparse
& \stdnum{99.93}{0.03\phantom{0}}\ \rddiff{0.05}\phantom{0}
& \stdnum{79.45}{0.42\phantom{0}}\ \rddiff{1.75}\phantom{0}
& \stdnum{73.02}{0.27\phantom{0}}\ \rddiff{3.64}\phantom{0}
& \phantom{00}5.44
& \phantom{0}\stdnum{50.67}{0.20\phantom{0}}\ \rddiff{3.28}\phantom{0}
& \phantom{0}1.79 \\

GU
& \stdnum{96.48}{0.20\phantom{0}}\ \rddiff{3.50}\phantom{0}
& \stdnum{96.45}{0.32\phantom{0}}\ \rddiff{18.75}
& \stdnum{73.05}{0.10\phantom{0}}\ \rddiff{3.61}\phantom{0}
& \phantom{0}25.86
& \phantom{00}\stdnum{9.91}{0.35\phantom{0}}\ \rddiff{37.48}
& \phantom{0}7.46 \\

OrthoGrad
& \stdnum{99.64}{0.02\phantom{0}}\ \rddiff{0.34}\phantom{0}
& \stdnum{76.72}{1.16\phantom{0}}\ \rddiff{0.98}\phantom{0}
& \stdnum{69.05}{0.09\phantom{0}}\ \rddiff{7.61}\phantom{0}
& \phantom{00}8.93
& \phantom{0}\stdnum{32.99}{1.69\phantom{0}}\ \rddiff{14.40}
& \phantom{0}7.51 \\

PCGrad
& \stdnum{95.88}{0.59\phantom{0}}\ \rddiff{4.10}\phantom{0}
& \stdnum{78.22}{0.54\phantom{0}}\ \rddiff{0.52}\phantom{0}
& \stdnum{69.48}{0.56\phantom{0}}\ \rddiff{7.18}\phantom{0}
& \phantom{0}11.80
& \phantom{0}\stdnum{29.75}{0.58\phantom{0}}\ \rddiff{17.64}
& \phantom{0}2.09 \\

UAM
& \stdnum{99.86}{0.02\phantom{0}}\ \rddiff{0.12}\phantom{0}
& \stdnum{78.16}{0.30\phantom{0}}\ \rddiff{0.46}\phantom{0}
& \stdnum{74.70}{0.07\phantom{0}}\ \rddiff{1.96}\phantom{0}
& \phantom{00}\underline{2.54}
& \phantom{0}\stdnum{38.55}{0.95\phantom{0}}\ \rddiff{8.84}\phantom{0}
& \phantom{0}2.41 \\

\textbf{ROSU}
& \stdnum{99.78}{0.00\phantom{0}}\ \rddiff{0.20}\phantom{0}
& \stdnum{77.98}{0.08\phantom{0}}\ \rddiff{0.28}\phantom{0}
& \stdnum{76.85}{0.18\phantom{0}}\ \rddiff{0.19}\phantom{0}
& \phantom{00}\textbf{0.67}
& \phantom{0}\stdnum{41.37}{0.75\phantom{0}}\ \rddiff{6.02}\phantom{0}
& \phantom{0}2.98 \\
\bottomrule
\end{tabular*}
}
\end{table*}

\subsection{Image classification unlearning}
\label{sec:exp-cv}

The vision configuration and the class-wise and random forgetting protocols follow UAM~\cite{kim2025unlearningaware}, with ResNet-18~\cite{resnet} for CIFAR-10/100~\cite{cifar} and VGG~\cite{VGG} for \mbox{Tiny-ImageNet}~\cite{TinyImageNet}. We compare against FT~\cite{WarneckePirchWressnegger2023_1000166047,scrubbing}, NG~\cite{scrubbing}, Fisher forgetting (FF)~\cite{scrubbing}, influence unlearning (IU)~\cite{izzo2021approximate,koh2017understanding}, \(\ell_1\)-sparse~\cite{jia2023model}, a PCGrad-style projection~\cite{yu2020gradient}, OrthoGrad~\cite{shamsian2026meansunlearningpersamplegradient}, GU~\cite{zhou2026geometricdisentangelmentunlearning}, and UAM~\cite{kim2025unlearningaware}. The tables also report runtime in minutes, measured as the wall-clock time of the unlearning loop on a single GPU. Baseline formulations, protocol details, and selection rules are provided in Appendix~\ref{app:setup}.

We report retain accuracy (RA), forget accuracy (FA), test accuracy (TA), MIA-Eff.~\cite{jia2023model}, \(\Delta\mathrm{Acc}\), and runtime. MIA-Eff. (membership-inference attack efficacy) is the forget-set non-member rate. The aggregate \(\Delta\mathrm{Acc}\) is computed from the mean RA, FA, and TA by summing their absolute differences from Retrain.

Tables~\ref{tab:cifar10_unlearning}, \ref{tab:cifar100_unlearning}, and~\ref{tab:tinyimagenet_unlearning} compare ROSU with the approximate-unlearning baselines. ROSU has the smallest \(\Delta\mathrm{Acc}\) in each of the six dataset--protocol settings. Its reduction in \(\Delta\mathrm{Acc}\) relative to UAM is largest in random forgetting, which Figure~\ref{fig:motivation_a} shows to be highly coupled. PCGrad, OrthoGrad, and GU enforce orthogonality only at the final update, and each has a larger \(\Delta\mathrm{Acc}\) than ROSU in all six settings. Section~\ref{sec:ablation} examines the update components and derivative variants separately.

\subsection{LLM unlearning}
\label{sec:llm}

For TOFU, we use \texttt{Llama-3.2-1B-Instruct}~\cite{grattafiori2024llama} with the OpenUnlearning pipeline~\cite{dorna2025openunlearning}, following its checkpoint-selection and evaluation protocol. For WMDP, we follow the setup of UAM~\cite{kim2025unlearningaware}, which uses the \texttt{Zephyr-7B-\(\beta\)}~\cite{tunstall2023zephyr} backbone of the WMDP benchmark~\cite{li2024the}. Appendices~\ref{app:setup-llm} and~\ref{app:algorithms} give the setup and the formal definitions.

For TOFU, we compare ROSU with GradDiff, IdkDPO, IdkNLL~\cite{maini2024tofu}, NPO~\cite{zhang2024negative}, SimNPO~\cite{fan2025simplicity}, AltPO~\cite{mekala2025alternate}, UNDIAL~\cite{dong2025undial}, PCGrad~\cite{yu2020gradient}, RMU~\cite{li2024the}, and UAM~\cite{kim2025unlearningaware}. For WMDP, we compare against SSD~\cite{foster2024fast}, SCRUB~\cite{kurmanji2023towards}, LLMU~\cite{yao2024large}, RMU~\cite{li2024the}, and UAM~\cite{kim2025unlearningaware}.

\begin{wraptable}[30]{r}{0.47\textwidth}
\colorlet{cwText}{knText}\colorlet{cwBand}{knBand}\colorlet{rdText}{knText}\colorlet{rdBand}{knBand}
\centering\vspace{-13.2pt}
\caption{\textbf{Unlearning on TOFU and WMDP.} \textit{Init.\ ft.}\ and Base are the TOFU and WMDP models before unlearning, and \textit{Retain} is the model fine-tuned on the retain data only. Bold and underline mark the best and second-best Agg.\ on TOFU and \(\Delta\)Score on WMDP among the unlearning methods we compare.}
\label{tab:llm}
{\footnotesize
\setlength{\tabcolsep}{3.2pt}
\begin{tabularx}{\linewidth}{l*{4}{>{\centering\arraybackslash}X}}
\toprulebeforeband
\rowcolor{cwBand}
\multicolumn{5}{c}{\bandfill{\textbf{\cw{TOFU}}}} \\
\midruleafterband

\textbf{Method} & \textbf{Agg.\scriptsize$\uparrow$} & \textbf{Mem.\scriptsize$\uparrow$} & \textbf{Priv.\scriptsize$\uparrow$} & \textbf{Util.\scriptsize$\uparrow$} \\
\midrule
\textit{Init.\ ft.} & 0.00 & 0.00 & 0.01 & 1.00 \\
\textit{Retain} & 0.61 & 0.35 & 1.00 & 0.99 \\
\cmidrule(lr){1-5}
GradDiff & 0.00 & 0.99 & 0.00 & 0.77 \\
NPO & 0.02 & 0.52 & 0.01 & 0.97 \\
IdkDPO & 0.03 & 0.56 & 0.01 & 0.93 \\
AltPO & 0.05 & 0.67 & 0.02 & 0.94 \\
PCGrad & 0.06 & 0.98 & 0.02 & 0.61 \\
IdkNLL & 0.13 & 0.10 & 0.08 & 0.91 \\
UNDIAL & 0.44 & 0.30 & 0.46 & 0.79 \\
SimNPO & 0.49 & 0.36 & 0.44 & 0.99 \\
UAM & 0.50 & 0.38 & 0.62 & 0.57 \\
RMU & \underline{0.54} & 0.48 & 0.53 & 0.62 \\
\cmidrule(lr){1-5}
\textbf{ROSU} & \textbf{0.58} & 0.38 & 0.74 & 0.82 \\
\midrulebeforeband
\rowcolor{rdBand}
\multicolumn{5}{c}{\bandfill{\textbf{\rd{WMDP}}}} \\
\midruleafterband

\textbf{Method} & \textbf{\(\Delta\)Score}{\scriptsize$\downarrow$} & \textbf{MMLU}{\scriptsize$\uparrow$} & \textbf{Bio} & \textbf{Cyber} \\
\midrule
Base & 0.577 & 0.581 & 0.637 & 0.440 \\
SSD & 0.526 & 0.407 & 0.502 & 0.350 \\
SCRUB & 0.400 & 0.512 & 0.438 & 0.393 \\
LLMU & 0.624 & 0.447 & 0.595 & 0.395 \\
RMU & 0.101 & 0.566 & 0.310 & 0.276 \\
UAM & \underline{0.077} & 0.564 & 0.293 & 0.233 \\
\cmidrule(lr){1-5}
\textbf{ROSU} & \textbf{0.054} & 0.570 & 0.289 & 0.246 \\
\bottomrule
\end{tabularx}
}
\end{wraptable}
For TOFU, Mem, Priv, and Util summarize memorization removal, membership-attack agreement with the retain model, and retained utility. Agg is their harmonic mean (HM), while Agg\(_{\text{MU}}\) uses Mem and Util for model selection. For WMDP, MMLU~\cite{hendrycks2021measuring} measures general knowledge~\cite{li2024the}, while Bio and Cyber measure hazardous-domain accuracy. The score \(\Delta\mathrm{Score}\) sums the absolute MMLU change from the base model and the Bio and Cyber distances from chance. Appendix~\ref{app:metrics} gives the full metric definitions.

Table~\ref{tab:llm} shows that ROSU attains the best aggregate score on TOFU among the unlearning methods. It matches UAM on memorization and exceeds it in privacy and utility. Figure~\ref{fig:qualitative} shows two forget-set prompts and the responses before and after unlearning. In both examples, the post-unlearning response is fluent but gives a name that differs from the ground truth.

On WMDP, ROSU attains the smallest \(\Delta\mathrm{Score}\) among the unlearning methods (Table~\ref{tab:llm}). It also has the highest MMLU apart from the base model and the lowest Bio accuracy. Its Cyber accuracy is near the chance level.

\begin{figure}[t]
    \centering
    \includegraphics[width=\linewidth]{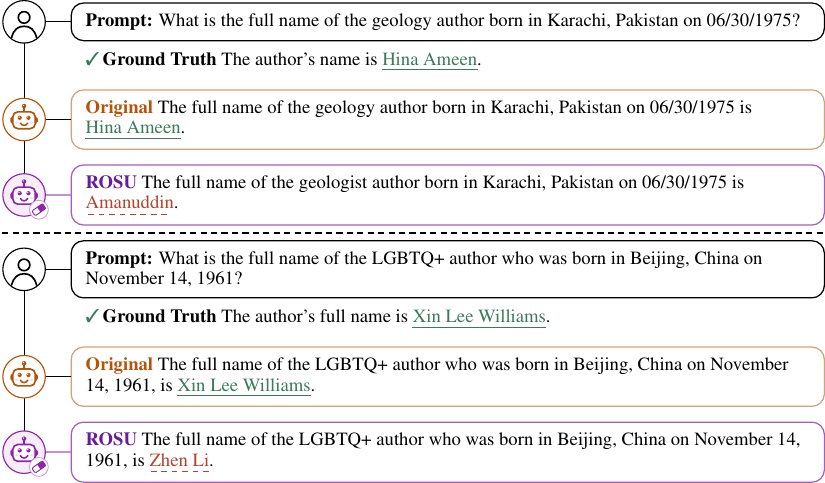}
    \caption{\textbf{Qualitative TOFU forgetting examples.} Responses to two forget-set prompts before and after unlearning with ROSU. \textcolor{correctcolor}{\uline{Green underline}} marks content consistent with the ground truth, and \textcolor{errcolor}{\dashuline{red dashed underline}} marks incorrect content.}
    \label{fig:qualitative}
\end{figure}

\subsection{Ablations and sensitivity}
\label{sec:ablation}

Appendix~\ref{app:component-ablations} separates retain correction from retain-neutral amplification. Amplification alone forgets but degrades retention in every setting. Correction alone keeps retention but forgets too little, with forget accuracy above Retrain in all six vision settings and no memorization removal on TOFU. On WMDP, it instead reduces general utility toward chance. Without the Jacobian correction, the \(\Delta\mathrm{Acc}\) gap to full ROSU is larger in the more coupled settings of each protocol (Figure~\ref{fig:motivation_a}), and Agg on TOFU drops. Appendices~\ref{app:beta-tuning} and~\ref{app:core-geometry} discuss amplification tuning and its retain-loss bound.

\begin{wraptable}[25]{r}{0.46\textwidth}
\centering\vspace{-13.3pt}
\caption{\textbf{Sensitivity to \(\rho\) on CIFAR-100 and CIFAR-10.} Results are for \rd{random} forgetting, with standard deviations over three seeds.}
\label{tab:rho_sensitivity_vision}
{\scriptsize
\setlength{\tabcolsep}{3pt}
\begin{tabular*}{\linewidth}{@{}l@{\extracolsep{\fill}}cccc@{}}
\toprule
\(\rho\) & \textbf{RA} & \textbf{FA} & \textbf{TA} & \textbf{$\Delta$Acc.}{\scriptsize$\downarrow$} \\
\midrulebeforeband
\rowcolor{rdBand}
\multicolumn{5}{c}{\bandfill{\textbf{\rd{CIFAR-100}}}} \\
\midruleafterband
\(0.001\) & \stdnum{99.95}{0.00} & \stdnum{77.31}{0.62} & \stdnum{76.35}{0.17} & 0.73 \\
\(0.005\) & \stdnum{99.97}{0.00} & \stdnum{77.33}{0.56} & \stdnum{76.37}{0.08} & 0.67 \\
\(0.01\) & \stdnum{99.97}{0.01} & \stdnum{77.47}{0.20} & \stdnum{76.22}{0.04} & 0.68 \\
\midrulebeforeband
\rowcolor{rdBand}
\multicolumn{5}{c}{\bandfill{\textbf{\rd{CIFAR-10}}}} \\
\midruleafterband
\(0.001\) & \stdnum{99.99}{0.00} & \stdnum{95.05}{0.02} & \stdnum{94.18}{0.02} & 0.26 \\
\(0.003\) & \stdnum{99.99}{0.01} & \stdnum{95.02}{0.22} & \stdnum{94.20}{0.12} & 0.27 \\
\(0.005\) & \stdnum{100.00}{0.00} & \stdnum{95.23}{0.18} & \stdnum{94.27}{0.11} & 0.24 \\
\bottomrule
\end{tabular*}
}

\vspace{3pt}
\caption{\textbf{Sensitivity to \(\rho\) on TOFU.} Following~\cite{dorna2025openunlearning}, the selected \(\rho\) maximizes Agg\(_{\text{MU}}\), and bold marks this \(\rho\) and its Agg\(_{\text{MU}}\).}
\label{tab:rho_sensitivity}
{\scriptsize
\setlength{\tabcolsep}{3pt}
\begin{tabular*}{\linewidth}{@{}l@{\extracolsep{\fill}}ccccc@{}}
\toprulebeforeband
\rowcolor{knBand}
\multicolumn{6}{c}{\bandfill{\textbf{\kn{TOFU}}}} \\
\midruleafterband
\textbf{$\rho$} & \textbf{Agg.}{\scriptsize$\uparrow$} & \textbf{Agg$_{\text{MU}}$}{\scriptsize$\uparrow$} & \textbf{Mem.}{\scriptsize$\uparrow$} & \textbf{Priv.}{\scriptsize$\uparrow$} & \textbf{Util.}{\scriptsize$\uparrow$} \\
\midrule
$0.20$ & 0.59 & 0.49 & 0.35 & 1.00 & 0.81 \\
$0.25$ & 0.59 & 0.49 & 0.36 & 1.00 & 0.78 \\
$0.30$ & 0.60 & 0.50 & 0.38 & 0.98 & 0.73 \\
$0.35$ & 0.58 & 0.49 & 0.35 & 0.93 & 0.80 \\
$\textbf{0.40}$ & 0.58 & \textbf{0.52} & 0.38 & 0.74 & 0.82 \\
$0.45$ & 0.58 & 0.49 & 0.35 & 0.92 & 0.81 \\
$0.50$ & 0.56 & 0.49 & 0.36 & 0.82 & 0.76 \\
\bottomrule
\end{tabular*}
}
\end{wraptable}
Figure~\ref{fig:a1a2_ablation} compares four derivative variants. The retain-neutral constraint is kept in every variant, so A1 and A2 change only how the correction \(\mathbf v\) is computed. A1+A2 drops the retain-projector differential and replaces the forget Hessian by the identity. No A1 keeps the projector differential, No A2 keeps the forget Hessian through a Hessian--vector product, and Exact (no A1, no A2) keeps both. No A1 closely tracks A1+A2, whereas No A2 and Exact use the true forget Hessian and diverge during training in the two random-forgetting settings with the highest coupling. Both collapse early in training, and the identity-Hessian relaxation keeps training stable.

Appendix~\ref{app:retain-batch-size} bounds full-retain orthogonality violation by gradient-estimation error and a regularization residual. Figure~\ref{fig:orth_violin} shows lower measured error with larger retain batches. Appendix~\ref{app:sensitivity} reports a radius sweep and safeguard audits. For every radius in Table~\ref{tab:rho_sensitivity_vision}, ROSU keeps \(\Delta\mathrm{Acc}\) close to the main-table values on \mbox{CIFAR-10} and \mbox{CIFAR-100} random forgetting. On TOFU, Agg stays close to its main-table value for every radius in Table~\ref{tab:rho_sensitivity}. The recorded projected-gradient norms in Figure~\ref{fig:grad_norms} stay above the degeneracy threshold \(\varepsilon_q\) in all audited records. Appendix~\ref{app:rank-k-empirical} extends ROSU to sampled rank-\(k\) protected subspaces on CIFAR-100 random forgetting.

\begin{figure}[t]
\centering
\includegraphics[width=\linewidth]{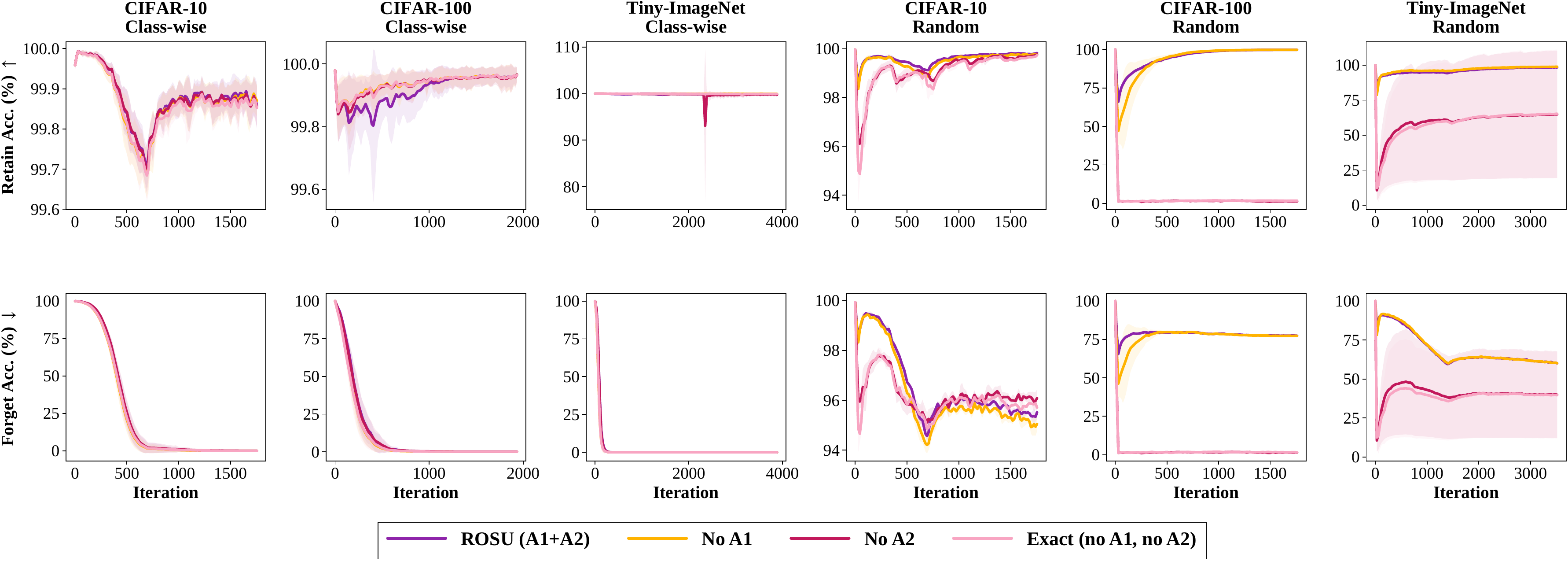}
\caption{\textbf{A1/A2 derivative variants.} Retain and forget accuracy of four ways of computing the retain correction. Curves and bands are the mean and one standard deviation over ten forgotten classes or three seeds.}
\label{fig:a1a2_ablation}
\end{figure}

\section{Limitations and broader impacts}
\label{sec:limitations}

Deletion requests, memorized sensitive content, and hazardous knowledge all call for making a trained model forget designated data without retraining it from scratch. Unlearning could also be misused, for example to remove the safety behaviors of a model or to hide which data it was trained on. The difficulty is to forget the designated data without hurting the model's performance on the data to keep. It is greater when the fastest way to raise the model's error on the data to forget also raises its error on the data to keep. We find this effect strong when random training examples or facts about fictitious authors are forgotten, and weak when a whole class is forgotten. ROSU looks for the change that forgets fastest among changes that leave the error on the data to keep unchanged to a first approximation.

ROSU uses the data to keep at every step, and each step requires more computation than a step of fine-tuning on these data. We see three directions for future research. One is unlearning when the data to keep are no longer available, which would require an unlearning algorithm to protect performance on data it cannot access. Another is unlearning without iterative training, which would require a single change to the model to forget without hurting performance on the data to keep. A third is unlearning when deletion requests arrive one after another, which would require each change to leave the data of earlier requests forgotten.

\section{Conclusion}
\label{sec:conclusion}

In min--max unlearning, an inner adversary perturbs the model to maximize the forget loss, and the retain loss is minimized at the resulting surrogate point. When the forget and retain gradients are positively coupled, this perturbation also moves toward higher retain loss. We propose ROSU, whose inner adversary avoids the retain cost that coupling induces by admitting only perturbations orthogonal to the retain gradient. The one that raises the forget loss fastest has a closed form, sets the surrogate point, and is amplified into the forgetting step. ROSU has the smallest accuracy gap to retraining in every vision setting and is competitive with the baselines on the language benchmarks.

\begin{ack}
This work is supported in part by the NRF grants funded by the Korea government(MSIT) (RS-2026-25488671, 33\% and RS-2022-NR070834, 33\%), and in part by ICT Creative Consilience Program through the IITP grant funded by the Korea government(MSIT)(IITP-2026-RS-2020-II201819, 33\%).
\end{ack}

\bibliographystyle{unsrt}
\bibliography{refs}

\newpage
\appendix

\section{Extended related work}
\label{app:related}

Machine unlearning is motivated by data-deletion requirements and the ``right to be forgotten'' in privacy regulation~\cite{mantelero2013eu,gdpr_review}. Unlearning methods trade the removal of the forget data against performance on the retained data. For deep classifiers, early baselines take the two extremes of this trade-off~\cite{scrubbing}. One ascends the loss on the forget data, and the other fine-tunes on the retained data and relies on catastrophic forgetting~\cite{continual}. Later methods refine the trade-off by changing what is removed, which data and targets the update uses, or which parameters change. Some unlearn features and labels or individual instances~\cite{WarneckePirchWressnegger2023_1000166047,learning2unlearn}. Others replace the retain set with out-of-distribution images~\cite{bonato2024retain}, use the forget data alone~\cite{cai2026scope}, or take their targets from the original model~\cite{kurmanji2023towards}. Parameter-selection methods decide which parameters to dampen, mask, or sparsify~\cite{foster2024fast,fan2023salun,jia2023model}.

For large language models~\cite{grattafiori2024llama,team2023gemini}, the same trade-off arises for new forget targets, such as memorized text, copyrighted content, unsafe behavior, and hazardous knowledge~\cite{carlini2023quantifying,karamolegkou2023copyright,wei2023jailbroken,li2024the}. The goal is then to remove specific facts or capabilities while preserving general knowledge and utility~\cite{yao2024large,liu2025rethinking,wang2025llm,xu2025relearn,gao2025on}. Exact unlearning becomes efficient when in-context learning replaces fine-tuning~\cite{muresanu2025fast}. Approximate methods instead change the forgetting objective~\cite{zhang2024negative,fan2025simplicity,mekala2025alternate,dong2025undial}, the optimizer~\cite{jia2024soul}, the regularizer~\cite{choi2025opt}, or the trained parameters~\cite{bhaila2025soft,ding2025unified}. Benchmarks measure the trade-off on fictitious, hazardous, memorized, and real-world knowledge~\cite{maini2024tofu,li2024the,shi2025muse,jin2024rwku}, and OpenUnlearning unifies methods, metrics, and stress tests such as relearning and quantization~\cite{dorna2025openunlearning}. Privacy attacks reveal the limits of unlearning, since unlearned texts remain detectable~\cite{du2025textual} and the unlearning update can expose the deleted data to reconstruction~\cite{bertran2024reconstruction}.

Across both settings, projection methods restrict the final update, most often to resolve a conflict between the forget and retain gradients. As in gradient surgery for multi-task learning~\cite{yu2020gradient}, some methods project the forget gradient away from the retain gradient only when the two conflict~\cite{wang2025gru,hu2026falcon}. Others protect retention more strictly and keep the forget gradient orthogonal to every per-sample retain gradient~\cite{shamsian2026meansunlearningpersamplegradient} or to the layer features of the retained data~\cite{null_unlearn}. When the model interpolates the retained data, the retain gradients vanish, so one method keeps the update orthogonal to the gradients of the model outputs and minimizes a complexity objective~\cite{block2025machine}. The constraint itself can also change, with orthogonality measured in an optimizer-induced metric~\cite{zhou2026geometricdisentangelmentunlearning} or the update restricted to low-rank subspaces of the forget gradients~\cite{sendera2025semu}. In every case the constraint acts on the final update.

Min--max methods act earlier, at the point where an objective is evaluated. UAM maximizes the forget loss over norm-bounded perturbations and minimizes the retain loss at the resulting surrogate point~\cite{kim2025unlearningaware}. A sharpness-aware formulation for LLMs applies this structure to the forgetting objective, with an inner perturbation that simulates a relearning attack~\cite{fan2025towards}. Two related methods replace the inner maximization with a temporary step on one objective and evaluate the other objective after it. One steps on the retain loss before evaluating the forget loss~\cite{patel2025learning}, and the other steps on the forget loss before testing retention and projects away gradient conflicts~\cite{huang2024learning}. Each of these methods chooses its evaluation point by one objective alone, and the projection methods constrain only the final update. ROSU uses both gradients to choose its surrogate point. Its inner perturbation raises the forget loss only in directions orthogonal to the retain gradient, which places the constraint on the surrogate point as well as on the forgetting step.

\section{Experimental setup and metrics}
\label{app:setup}

Appendix~\ref{app:setup-cv} describes vision experiments, \ref{app:setup-llm} covers TOFU and WMDP, \ref{app:metrics}--\ref{app:model-selection} define metrics and model selection, \ref{app:impl} provides implementation details, and \ref{app:algorithms} documents algorithm variants.

\subsection{Image classification protocol}
\label{app:setup-cv}

The vision configuration and the class-wise and random forgetting protocols follow the image-classification setup of UAM~\cite{kim2025unlearningaware}. CIFAR-10 and \mbox{CIFAR-100}~\cite{cifar} use ResNet-18~\cite{resnet}, while \mbox{Tiny-ImageNet}~\cite{TinyImageNet} uses VGG~\cite{VGG}.

The unlearning budgets are 5 epochs for UAM, ROSU, PCGrad, OrthoGrad, and GU, and 10 epochs for FT~\cite{WarneckePirchWressnegger2023_1000166047,scrubbing}, NG~\cite{scrubbing}, and \(\ell_1\)-sparse~\cite{jia2023model}. FF and IU apply a single parameter update~\cite{jia2023model}.

Supervised pretraining uses SGD with initial learning rate 0.1, momentum 0.9, weight decay \(5\times 10^{-4}\), and a 200-epoch schedule with learning-rate drops at epochs 100 and 150~\cite{kim2025unlearningaware}.

The update-orthogonalization baselines are defined as follows:
\begin{itemize}

\item The PCGrad-style~\cite{yu2020gradient} baseline uses a trade-off coefficient \(\lambda_{\mathrm{pc}}\) and the projected forget gradient $\mathbf{q}_{\mathrm{pc}} = \mathbf{g}_f - \frac{\mathbf{g}_f^\top \mathbf{g}_r}{\|\mathbf{g}_r\|_2^2}\,\mathbf{g}_r$ for $\mathbf{g}_r\neq\mathbf{0}$, with $\mathbf{q}_{\mathrm{pc}}=\mathbf{g}_f$ for $\mathbf{g}_r=\mathbf{0}$. The update is
        \begin{equation*}
\begin{aligned}
\mathbf{w} &\leftarrow \mathbf{w} - \eta\bigl(\mathbf{g}_r - \lambda_{\mathrm{pc}}\,\mathbf{q}_{\mathrm{pc}}\bigr).
\end{aligned}
        \end{equation*}

\item The OrthoGrad~\cite{shamsian2026meansunlearningpersamplegradient} baseline uses a mixing coefficient \(\alpha\) and the mini-batch average retain gradient $\mathbf{g}_r$. The forget gradient is projected as $\mathbf{g}_f^\perp = \mathbf{g}_f - \mathbf{Q}_r \mathbf{Q}_r^\top \mathbf{g}_f$, where a QR decomposition gives the orthonormal columns of \(\mathbf{Q}_r\), which span the per-sample retain gradients. The update is
        \begin{equation*}
\begin{aligned}
\mathbf{w} &\leftarrow \mathbf{w} - \eta\bigl(\alpha\,\mathbf{g}_r - (1-\alpha)\,\mathbf{g}_f^\perp\bigr).
\end{aligned}
        \end{equation*}

\item The GU~\cite{zhou2026geometricdisentangelmentunlearning} baseline in its Euclidean case uses a learning rate \(\eta_{\mathrm{GU}}\) and the same basis \(\mathbf Q_r\) as above. The update is
        \begin{equation*}
\begin{aligned}
\mathbf{w} &\leftarrow \mathbf{w} + \eta_{\mathrm{GU}}\bigl((\mathbf{I}-\mathbf{Q}_r\mathbf{Q}_r^\top)\,\mathbf{g}_f - \mathbf{Q}_r\mathbf{Q}_r^\top\mathbf{g}_r\bigr).
\end{aligned}
        \end{equation*}

\end{itemize}
These baselines modify gradient directions at the final update. ROSU imposes its ideal first-order retain-neutral constraint on the inner surrogate perturbation. Hyperparameters are selected under the protocols of Appendix~\ref{app:model-selection}. Some results are taken from UAM~\cite{kim2025unlearningaware}.

\subsection{TOFU and WMDP protocols}
\label{app:setup-llm}

For TOFU, we follow the standardized OpenUnlearning pipeline for implementation and evaluation~\cite{dorna2025openunlearning}, using \texttt{Llama-3.2-1B-Instruct} as the target model. This framework standardizes tokenization, collation, generation-time evaluation, and the aggregation of memorization, privacy, and utility metrics. The compared methods are listed in Section~\ref{sec:llm}.

For WMDP, we use the representation-level instantiation given in Algorithm~\ref{alg:rosu_repr}. The target layer is \(\ell=7\), and the trainable coordinate set \(\mathcal{T}\) indexes scalar parameters in the MLP output (\texttt{down\_proj}) matrices in transformer layers 5--7, with all other parameters frozen. The forget and retain losses are the representation-level objectives of Algorithm~\ref{alg:rosu_repr}. The forget corpus is WMDP, and the retain corpus is WikiText.
Evaluation uses a four-choice question-answering benchmark, where random selection corresponds to an expected accuracy of \(0.25\) on both WMDP-Bio and WMDP-Cyber. The restricted coordinate set and representation-level configuration follow UAM~\cite{kim2025unlearningaware}. Some baseline results are taken directly from UAM~\cite{kim2025unlearningaware}.

\subsection{Evaluation metrics}
\label{app:metrics}

For vision, let \(\mathbf x\) be an input, \(y\) its class label, and \(\mathbf f_{\mathbf w}(\mathbf x)\) the vector of class scores. Define \(\hat y_{\mathbf w}(\mathbf x):=\argmax_j\mathbf f_{\mathbf w}(\mathbf x)_j\), with class index \(j\) and ties resolved by a fixed rule. The nonempty sets \(\mathcal D_r^{\mathrm{eval}}\), \(\mathcal D_f^{\mathrm{eval}}\), and \(\mathcal D_{\mathrm{test}}\) contain retain, forget, and test evaluation examples, respectively. In class-wise forgetting, \(\mathcal D_{\mathrm{test}}\) excludes the test examples of the forgotten class. With \(\mathbb{I}\{\cdot\}\) equal to 1 when its condition holds and 0 otherwise, we evaluate retain accuracy (RA), forget accuracy (FA), and test accuracy (TA) as
\[
\begin{aligned}
    \mathrm{RA}(\mathbf{w})
        &=
            \frac{1}{|\mathcal D_r^{\mathrm{eval}}|}
            \sum\limits_{(\mathbf{x},y)\in\mathcal D_r^{\mathrm{eval}}}
            \mathbb{I}\{\hat y_{\mathbf{w}}(\mathbf{x})=y\},\\
    \mathrm{FA}(\mathbf{w})
        &=
            \frac{1}{|\mathcal D_f^{\mathrm{eval}}|}
            \sum\limits_{(\mathbf{x},y)\in\mathcal D_f^{\mathrm{eval}}}
            \mathbb{I}\{\hat y_{\mathbf{w}}(\mathbf{x})=y\},\\
    \mathrm{TA}(\mathbf{w})
        &=
            \frac{1}{|\mathcal D_{\mathrm{test}}|}
            \sum\limits_{(\mathbf{x},y)\in\mathcal D_{\mathrm{test}}}
            \mathbb{I}\{\hat y_{\mathbf{w}}(\mathbf{x})=y\}.
\end{aligned}
\]
Let \(A_{\mathrm{Retrain}}\) denote metric \(A\) evaluated on the retrained reference. We summarize the aggregate accuracy discrepancy by
\[
\begin{aligned}
\Delta\mathrm{Acc}(\mathbf{w}) &:= \sum\limits_{A\in\{\mathrm{RA},\mathrm{FA},\mathrm{TA}\}} \bigl|A_{\mathrm{Retrain}}-A(\mathbf{w})\bigr|,
\end{aligned}
\]
where lower values are better~\cite{kim2025unlearningaware}. In the vision tables, accuracies and MIA-Eff. are reported as percentages, and $\Delta\mathrm{Acc}$ is reported in percentage points (i.e., the ratio-valued definitions are multiplied by $100$). Table values of \(\Delta\mathrm{Acc}\) are computed from the mean RA, FA, and TA, taking the absolute differences after averaging over targets.

The membership-inference efficacy metric reports the fraction of forget examples classified as non-members~\cite{jia2023model,kim2025unlearningaware}. Let \(\mathcal A_{\mathrm{mia}}(\mathbf{x};\mathbf{w})\in\{0,1\}\) denote the membership predictor, where \(1\) indicates ``member''. It is the confidence-based attack~\cite{robust_privacy,mia} used by UAM, implemented as a support-vector classifier on output probabilities. We define
\[
\begin{aligned}
\mathrm{MIA\text{-}Eff.}(\mathbf{w}) &:= \frac{1}{|\mathcal D_f^{\mathrm{eval}}|} \sum\limits_{(\mathbf{x},y)\in\mathcal D_f^{\mathrm{eval}}} \mathbb{I}\bigl\{\mathcal A_{\mathrm{mia}}(\mathbf{x};\mathbf{w})=0\bigr\}.
\end{aligned}
\]
Higher values indicate more forget examples classified as non-members. For random forgetting, agreement with the Retrain value is the evaluation target.

\looseness=-1 For TOFU, we adopt OpenUnlearning's metric categories and aggregation~\cite{dorna2025openunlearning}. HM denotes the harmonic mean of nonnegative arguments, with value zero whenever any argument is zero. Memorization measures forgetting through four knowledge metrics from OpenUnlearning~\cite{dorna2025openunlearning}. They are extraction strength (ES), exact memorization (EM)~\cite{tirumala2022memorization}, paraphrased probability (Para.\ Prob., denoted \(\mathrm{ParaProb}\)), and truth ratio (TR), each inverted so that a higher value means less forget-set knowledge. Their harmonic mean is
\[
\begin{aligned}
\mathrm{Mem}&=\mathrm{HM}\bigl(1-\mathrm{ES},\,1-\mathrm{EM},\,1-\mathrm{ParaProb},\,1-\mathrm{TR}\bigr).
\end{aligned}
\]
\looseness=-1 Privacy measures agreement with the retain model under the four membership inference attacks implemented in OpenUnlearning~\cite{dorna2025openunlearning}, LOSS~\cite{mia}, ZLib~\cite{carlini2021extracting}, MinK~\cite{shi2024detecting}, and MinK++~\cite{zhang2025min}. For each attack \(a\), OpenUnlearning computes a score \(s_a\in[0,1]\) that quantifies how closely the unlearned model's behavior under that attack aligns with the retain model's, higher meaning closer. Their harmonic mean is
\[
\begin{aligned}
\mathrm{Priv}&=\mathrm{HM}\bigl(s_{\mathrm{LOSS}},\,s_{\mathrm{ZLib}},\,s_{\mathrm{MinK}},\,s_{\mathrm{MinK{+}{+}}}\bigr).
\end{aligned}
\]
Utility measures retention through TOFU's model utility (MU)~\cite{maini2024tofu} and through fluency. MU is the harmonic mean of probability, ROUGE~\cite{lin2004rouge}, and truth ratio on the retain, real-author, and world-fact splits. Fluency is a classifier score that penalizes gibberish in answers to forget-related prompts. The harmonic mean of MU and fluency is
\[
\begin{aligned}
\mathrm{Util}&=\mathrm{HM}(\mathrm{MU},\,\mathrm{Fluency}).
\end{aligned}
\]
Metric preprocessing and normalization follow the OpenUnlearning evaluation pipeline used for these experiments. The headline score and the model-selection score are
\[
\begin{aligned}
\mathrm{Agg}&=\mathrm{HM}(\mathrm{Mem},\,\mathrm{Priv},\,\mathrm{Util}),\qquad \mathrm{Agg}_{\mathrm{MU}}=\mathrm{HM}(\mathrm{Mem},\,\mathrm{Util}),
\end{aligned}
\]
so a high Agg requires forgetting, agreement with the retain model, and retained utility at the same time. The detailed computation of each metric follows OpenUnlearning.

For WMDP, Bio and Cyber are four-choice question-answering accuracies on the biosecurity and cybersecurity splits of WMDP~\cite{li2024the}, and MMLU is four-choice accuracy on general academic knowledge. Higher MMLU indicates better general utility, while Bio and Cyber closer to the chance level \(0.25\) indicate more complete forgetting without oversuppression. We summarize performance using
\begin{equation*}
\begin{aligned}
\Delta\mathrm{Score}(\mathbf{w}) &:= \bigl|\mathrm{MMLU}(\mathbf{w})-\mathrm{MMLU}(\mathbf{w}_0)\bigr| +\bigl|\mathrm{Bio}(\mathbf{w})-0.25\bigr| +\bigl|\mathrm{Cyber}(\mathbf{w})-0.25\bigr|,
\end{aligned}
\end{equation*}
where $\mathbf{w}_0$ denotes the pre-unlearning weights. The first term measures the absolute change in general-knowledge accuracy, and the remaining terms measure deviation from chance-level forgetting on hazardous domains. Lower $\Delta\mathrm{Score}$ is better. We additionally report raw Bio and Cyber accuracies. Under this metric, the ideal corresponds to values close to $0.25$ rather than arbitrarily low accuracy. The OpenUnlearning evaluation~\cite{dorna2025openunlearning} penalizes both under- and over-unlearning, and $\Delta\mathrm{Score}$ likewise penalizes deviation in both directions.

\subsection{Model selection and amplification tuning}
\label{app:model-selection}

For vision, we use the epoch budgets of the UAM setup~\cite{kim2025unlearningaware}. Hyperparameters are tuned per method following UAM~\cite{kim2025unlearningaware}, and for each method and setting we report the best configuration and checkpoint found under this protocol.

For TOFU, we follow OpenUnlearning's benchmarking protocol~\cite{dorna2025openunlearning}. Each method is tuned over its own search space, and model selection uses $\mathrm{Agg}_{\mathrm{MU}}=\mathrm{HM}(\mathrm{Mem},\mathrm{Util})$.

For WMDP, we follow the training protocol of UAM and report the full (MMLU, Bio, Cyber) triple together with the composite $\Delta\mathrm{Score}$ defined in Appendix~\ref{app:metrics}.

\label{app:beta-tuning}

The amplification schedule of Section~\ref{sec:correction} and its level are selected together with \(\rho\) and \(\eta\) under the selection rules above.

Raising the amplification level scales the first-order forget gain along the retain-neutral direction linearly (Section~\ref{sec:correction}) and the second-order retain bound quadratically (Appendix~\ref{app:core-geometry}). Two ablations isolate the two update components. Setting \(\beta=0\) gives the \(-\eta\mathbf v\)-only update, while the \(\beta\boldsymbol{\delta}_\tau\)-only variant keeps only retain-neutral amplification.

\subsection{Implementation details and hyperparameter search}
\label{app:impl}

\begin{algorithm}[t]
\caption{Representation-Level ROSU}
\label{alg:rosu_repr}
\begin{algorithmic}[1]
\Input frozen pre-unlearning weights \(\mathbf{w}_0\), trainable coordinate set \(\mathcal{T}\), current trainable vector \(\boldsymbol{\theta}=\mathbf{w}_{\mathcal{T}}\), nonempty forget mini-batch \(\mathcal{B}_f\) and retain mini-batch \(\mathcal{B}_r\)
\Parameters target layer \(\ell\), learning rate \(\eta>0\), perturbation radius \(\rho>0\), amplification coefficient \(\beta\ge0\), bootstrap standard deviation \(\sigma_b\ge0\), stabilizer \(\tau>0\), threshold \(\varepsilon_q>0\)
\Output updated trainable vector \(\boldsymbol{\theta}\) and the full model \(\mathbf{w}(\boldsymbol{\theta})\)
\State Let \(\mathbf{w}(\boldsymbol{\theta})\) be the full model obtained by inserting \(\boldsymbol{\theta}\) into coordinates \(\mathcal{T}\) and keeping all other coordinates fixed at \(\mathbf{w}_0\), and let \(\mathbf{h}_{\mathbf{w}}^{(\ell)}(\mathbf{x})\) be the layer-\(\ell\) representation of \(\mathbf{x}\)
\State Sample \(\boldsymbol{\sigma} \sim \mathcal{N}(\mathbf{0},\sigma_b^2 \mathbf{I})\) with the shape of \(\mathbf{h}^{(\ell)}\) \Comment{Mini-batch shared bootstrap noise}
\State \(L_f(\boldsymbol{\theta}) \gets \frac{1}{2|\mathcal{B}_f|}\!\sum\limits_{\mathbf{x}\in \mathcal{B}_f}\!\bigl\|\mathbf{h}_{\mathbf{w}(\boldsymbol{\theta})}^{(\ell)}(\mathbf{x})+\boldsymbol{\sigma}-\mathbf{h}_{\mathbf{w}_0}^{(\ell)}(\mathbf{x})\bigr\|_2^{2}\)
\State \(L_r(\boldsymbol{\theta}) \gets \frac{1}{2|\mathcal{B}_r|}\!\sum\limits_{\mathbf{x}\in \mathcal{B}_r}\!\bigl\|\mathbf{h}_{\mathbf{w}(\boldsymbol{\theta})}^{(\ell)}(\mathbf{x})-\mathbf{h}_{\mathbf{w}_0}^{(\ell)}(\mathbf{x})\bigr\|_2^{2}\)
\State \(\mathbf{g}_f \gets \nabla_{\boldsymbol{\theta}}L_f(\boldsymbol{\theta})\), \(\quad \mathbf{g}_r \gets \nabla_{\boldsymbol{\theta}}L_r(\boldsymbol{\theta})\)
\State \(\mathbf{q}_\tau \gets \mathbf{g}_f-\frac{\mathbf{g}_f^\top \mathbf{g}_r}{\|\mathbf{g}_r\|_2^2+\tau}\,\mathbf{g}_r\)
\If{\(\|\mathbf{q}_\tau\|_2\le \varepsilon_q\)}
    \State \(\boldsymbol{\theta} \gets \boldsymbol{\theta}-\eta\,\mathbf{g}_r\) \Comment{Degenerate fallback, Section~\ref{sec:fallback}}
\Else
    \State \(\mathbf{u}_\tau \gets \frac{\mathbf{q}_\tau}{\|\mathbf{q}_\tau\|_2}\), \(\quad \boldsymbol{\delta}_\tau \gets \rho\,\mathbf{u}_\tau\), \(\quad \alpha_\tau \gets \frac{\rho}{\|\mathbf{q}_\tau\|_2}\)
    \State \(\widetilde{\mathbf{g}}_{r,\tau} \gets \left.\nabla_{\mathbf z}L_r(\mathbf z)\right|_{\mathbf z=\boldsymbol{\theta}+\boldsymbol{\delta}_\tau}\) \Comment{Surrogate retain gradient in the trainable subspace}
    \State \(\mathbf{v} \gets \widetilde{\mathbf{g}}_{r,\tau}+\alpha_\tau\!\left(\widetilde{\mathbf{g}}_{r,\tau}-\frac{\mathbf{g}_r^\top \widetilde{\mathbf{g}}_{r,\tau}}{\|\mathbf{g}_r\|_2^2+\tau}\,\mathbf{g}_r-(\mathbf{u}_\tau^\top\widetilde{\mathbf{g}}_{r,\tau})\,\mathbf{u}_\tau\right)\)
    \State \(\boldsymbol{\theta} \gets \boldsymbol{\theta}+\beta\,\boldsymbol{\delta}_\tau-\eta\,\mathbf{v}\)
\EndIf
\end{algorithmic}
\end{algorithm}

The vision configuration uses SGD with momentum $0.9$ and weight decay $5\times 10^{-4}$, with default forget and retain batch sizes of $128$ unless stated otherwise. ROSU passes \(\mathbf v\) and the fallback gradient to this optimizer and adds \(\beta\boldsymbol{\delta}_\tau\) to the weights directly. OrthoGrad and GU apply the updates of Appendix~\ref{app:setup-cv} without momentum or weight decay.

For TOFU, we use \texttt{Llama-3.2-1B-Instruct} within the released OpenUnlearning runners~\cite{dorna2025openunlearning}. Each method retains its benchmark implementation and search space, while checkpoint selection follows Appendix~\ref{app:model-selection}.

For WMDP, we use \texttt{Zephyr-7B-$\beta$} with the representation-level setup and training protocol of UAM~\cite{kim2025unlearningaware}. The bootstrap configuration~\cite{kim2025unlearningaware} uses noise in Algorithm~\ref{alg:rosu_repr} sampled once per mini-batch from \(\mathcal{N}(\mathbf{0},\sigma_b^2 \mathbf{I})\) with $\sigma_b = 0.01$, which is fixed throughout.

Across all benchmarks, ROSU uses each pipeline's standard training setup. We search \(\rho\) in \([0.001,2.5]\) and the learning rate in \([10^{-5},10^{-1}]\), and choose \(c_t\) so that the largest step length \(c_t\eta_t\) lies in \([10^{-4},1]\). These hyperparameters are selected under the protocols of Appendix~\ref{app:model-selection}. Appendix~\ref{app:sensitivity} reports a separate sweep over \(\rho\).

In Algorithms~\ref{alg:rosu_full}--\ref{alg:rosu_repr}, \(\beta\) denotes the current-step amplification coefficient \(\beta_t\) of Section~\ref{sec:correction}. In the radius sweep of Appendix~\ref{app:sensitivity} it is switched off after the early epochs.

The stabilizer $\tau = 10^{-8}$ and degeneracy threshold $\varepsilon_q = 10^{-6}$ are fixed in all runs. All experiments are run on NVIDIA RTX A6000 GPUs with 48 GB of memory.

\subsection{Algorithm variants}
\label{app:algorithms}

Here we describe the No-Jacobian variant and the representation-level update for WMDP.

The No-Jacobian variant uses the regularized surrogate construction, amplification, and fallback of Algorithm~\ref{alg:rosu_full}. On nonfallback steps, it uses \(\mathbf v=\widetilde{\mathbf g}_{r,\tau}\), without the Jacobian correction.

Algorithm~\ref{alg:rosu_repr} extends ROSU to the representation-level WMDP setting of UAM~\cite{kim2025unlearningaware}. In Algorithm~\ref{alg:rosu_repr}, both \(L_f\) and \(L_r\) are squared-error losses between current and frozen pre-unlearning representations at a target layer \(\ell\). Here \(\mathbf{h}^{(\ell)}_{\mathbf{w}}(\mathbf{x})\) denotes the layer-\(\ell\) representation of input \(\mathbf{x}\) under weights \(\mathbf{w}\). Let \(\mathbf{w}_0\) denote the frozen pre-unlearning weights, \(\mathcal{T}\) the trainable coordinate set, and \(\boldsymbol{\theta}=\mathbf{w}_{\mathcal{T}}\in\mathbb{R}^{|\mathcal{T}|}\) the corresponding trainable vector. We write \(\mathbf{w}(\boldsymbol{\theta})\) for the full model obtained by inserting \(\boldsymbol{\theta}\) into \(\mathcal{T}\) while keeping all other parameters fixed at \(\mathbf{w}_0\). The retain-orthogonal construction and update are then performed in the \(\boldsymbol{\theta}\)-space. In Algorithm~\ref{alg:rosu_repr}, \(\mathbf z\in\mathbb R^{|\mathcal T|}\) is the loss argument used to evaluate the surrogate retain gradient.

The losses in Algorithm~\ref{alg:rosu_repr} average over examples and sum squared representation coordinates, with a factor of \(\frac{1}{2}\). Here \(\boldsymbol{\sigma}\) has the same representation shape as \(\mathbf h^{(\ell)}\), and the identity in its covariance acts on representation coordinates. The identity in the projector \(\mathbf I-\mathbf P_r^{(\tau)}\) of Section~\ref{sec:fallback} instead acts on the \(|\mathcal T|\) trainable coordinates. The bootstrap sample is held fixed during the update. At \(\mathbf w_0\), \(\sigma_b=0\) makes both loss gradients in Algorithm~\ref{alg:rosu_repr} zero, so its fallback step leaves the parameters unchanged.

\section{Additional experimental results}
\label{app:addl-results}

Appendix~\ref{app:remaining-benchmark-table} gives the Tiny-ImageNet results, and Appendix~\ref{app:component-ablations} isolates the update components. Amplification tuning is described in Appendix~\ref{app:beta-tuning}. Appendices~\ref{app:sensitivity}, \ref{app:retain-batch-size}, and~\ref{app:rank-k-empirical} examine the perturbation radius and numerical safeguards, retain mini-batch size, and sampled retain subspaces.

\subsection{Additional vision benchmark results}
\label{app:remaining-benchmark-table}

\begin{table}[H]
\centering
\caption{\textbf{Unlearning on Tiny-ImageNet.} Same notation as Table~\ref{tab:cifar10_unlearning}. FF is not reported.}
\label{tab:tinyimagenet_unlearning}
{\scriptsize
\setlength{\tabcolsep}{2pt}
\begin{tabular*}{\textwidth}{@{}l@{\extracolsep{\fill}}cccccc@{}}
\toprule
\textbf{Method} & \textbf{RA} & \textbf{FA} & \textbf{TA} & \textbf{$\Delta$Acc.}{\scriptsize$\downarrow$} & \textbf{MIA-Eff.} & \textbf{Time (min)} \\
\midrulebeforeband

\rowcolor{cwBand}
\multicolumn{7}{c}{\bandfill{\textbf{\cw{Class-wise forgetting}}}} \\
\midruleafterband

Retrain
& \stdnum{99.98}{0.00}\phantom{\ (00.00)}
& \phantom{0}\stdnum{0.00}{0.00\phantom{0}}\phantom{\ (00.00)}
& \stdnum{61.89}{0.18}\phantom{\ (00.00)}
& \phantom{00}\plainzero
& \stdnum{100.00}{0.00\phantom{0}}\phantom{\ (00.00)}
& 189.8\phantom{0} \\

FT
& \stdnum{98.54}{2.03}\ \cwdiff{1.44}\phantom{0}
& \stdnum{26.88}{13.27}\ \cwdiff{26.88}
& \stdnum{57.52}{1.95}\ \cwdiff{4.37}\phantom{0}
& \phantom{0}32.69
& \phantom{0}\stdnum{97.37}{2.36\phantom{0}}\ \cwdiff{2.63}\phantom{0}
& \phantom{00}9.41 \\

NG
& \stdnum{92.84}{3.34}\ \cwdiff{7.14}\phantom{0}
& \phantom{0}\stdnum{0.23}{0.32\phantom{0}}\ \cwdiff{0.23}\phantom{0}
& \stdnum{52.28}{2.23}\ \cwdiff{9.61}\phantom{0}
& \phantom{0}16.98
& \phantom{0}\stdnum{99.92}{0.17\phantom{0}}\ \cwdiff{0.08}\phantom{0}
& \phantom{00}9.43 \\

IU
& \stdnum{95.83}{3.51}\ \cwdiff{4.15}\phantom{0}
& \phantom{0}\stdnum{1.20}{2.05\phantom{0}}\ \cwdiff{1.20}\phantom{0}
& \stdnum{55.03}{2.64}\ \cwdiff{6.86}\phantom{0}
& \phantom{0}12.21
& \phantom{0}\stdnum{99.84}{0.31\phantom{0}}\ \cwdiff{0.16}\phantom{0}
& \phantom{00}1.27 \\

$\ell_{1}$-sparse
& \stdnum{98.19}{0.16}\ \cwdiff{1.79}\phantom{0}
& \phantom{0}\stdnum{0.00}{0.00\phantom{0}}\ \cwdiff{0.00}\phantom{0}
& \stdnum{59.66}{0.19}\ \cwdiff{2.23}\phantom{0}
& \phantom{00}4.02
& \stdnum{100.00}{0.00\phantom{0}}\ \cwdiff{0.00}\phantom{0}
& \phantom{0}10.47 \\

GU
& \stdnum{99.54}{0.22}\ \cwdiff{0.44}\phantom{0}
& \phantom{0}\stdnum{0.00}{0.00\phantom{0}}\ \cwdiff{0.00}\phantom{0}
& \stdnum{58.59}{0.62}\ \cwdiff{3.30}\phantom{0}
& \phantom{00}3.74
& \stdnum{100.00}{0.00\phantom{0}}\ \cwdiff{0.00}\phantom{0}
& \phantom{0}17.12 \\

OrthoGrad
& \stdnum{99.67}{0.18}\ \cwdiff{0.31}\phantom{0}
& \phantom{0}\stdnum{0.13}{0.31\phantom{0}}\ \cwdiff{0.13}\phantom{0}
& \stdnum{58.96}{0.64}\ \cwdiff{2.93}\phantom{0}
& \phantom{00}3.37
& \stdnum{100.00}{0.00\phantom{0}}\ \cwdiff{0.00}\phantom{0}
& \phantom{0}17.64 \\

PCGrad
& \stdnum{99.84}{0.09}\ \cwdiff{0.14}\phantom{0}
& \phantom{0}\stdnum{0.08}{0.23\phantom{0}}\ \cwdiff{0.08}\phantom{0}
& \stdnum{59.59}{0.71}\ \cwdiff{2.30}\phantom{0}
& \phantom{00}2.52
& \stdnum{100.00}{0.00\phantom{0}}\ \cwdiff{0.00}\phantom{0}
& \phantom{00}9.14 \\

UAM
& \stdnum{99.97}{0.02}\ \cwdiff{0.01}\phantom{0}
& \phantom{0}\stdnum{0.23}{0.26\phantom{0}}\ \cwdiff{0.23}\phantom{0}
& \stdnum{60.86}{0.64}\ \cwdiff{1.03}\phantom{0}
& \phantom{00}\underline{1.27}
& \phantom{0}\stdnum{99.97}{0.08\phantom{0}}\ \cwdiff{0.03}\phantom{0}
& \phantom{0}12.85 \\

\textbf{ROSU}
& \stdnum{99.97}{0.00}\ \cwdiff{0.01}\phantom{0}
& \phantom{0}\stdnum{0.00}{0.00\phantom{0}}\ \cwdiff{0.00}\phantom{0}
& \stdnum{61.89}{0.22}\ \cwdiff{0.00}\phantom{0}
& \phantom{00}\textbf{0.01}
& \stdnum{100.00}{0.00\phantom{0}}\ \cwdiff{0.00}\phantom{0}
& \phantom{0}14.29 \\

\midrulebeforeband

\rowcolor{rdBand}
\multicolumn{7}{c}{\bandfill{\textbf{\rd{Random forgetting}}}} \\
\midruleafterband

Retrain
& \stdnum{99.98}{0.00}\phantom{\ (00.00)}
& \stdnum{60.25}{0.19\phantom{0}}\phantom{\ (00.00)}
& \stdnum{60.72}{0.22}\phantom{\ (00.00)}
& \phantom{00}\plainzero
& \phantom{0}\stdnum{64.50}{0.60\phantom{0}}\phantom{\ (00.00)}
& 169.5\phantom{0} \\

FT
& \stdnum{99.98}{0.00}\ \rddiff{0.00}\phantom{0}
& \stdnum{99.96}{0.01\phantom{0}}\ \rddiff{39.71}
& \stdnum{61.43}{0.07}\ \rddiff{0.71}\phantom{0}
& \phantom{0}40.42
& \phantom{00}\stdnum{3.89}{1.21\phantom{0}}\ \rddiff{60.61}
& \phantom{00}8.48 \\

NG
& \phantom{0}\stdnum{0.60}{0.08}\ \rddiff{99.38}
& \phantom{0}\stdnum{0.57}{0.06\phantom{0}}\ \rddiff{59.68}
& \phantom{0}\stdnum{0.61}{0.09}\ \rddiff{60.11}
& 219.17
& \phantom{0}\stdnum{66.55}{46.56}\ \rddiff{2.05}\phantom{0}
& \phantom{00}8.47 \\

IU
& \stdnum{89.83}{8.53}\ \rddiff{10.15}
& \stdnum{89.25}{8.56\phantom{0}}\ \rddiff{29.00}
& \stdnum{50.84}{4.81}\ \rddiff{9.88}\phantom{0}
& \phantom{0}49.03
& \phantom{0}\stdnum{14.42}{6.60\phantom{0}}\ \rddiff{50.08}
& \phantom{00}1.25 \\

$\ell_{1}$-sparse
& \stdnum{98.86}{0.05}\ \rddiff{1.12}\phantom{0}
& \stdnum{59.66}{0.55\phantom{0}}\ \rddiff{0.59}\phantom{0}
& \stdnum{58.16}{0.23}\ \rddiff{2.56}\phantom{0}
& \phantom{00}\underline{4.27}
& \phantom{0}\stdnum{54.54}{0.62\phantom{0}}\ \rddiff{9.96}\phantom{0}
& \phantom{00}9.11 \\

GU
& \stdnum{99.98}{0.00}\ \rddiff{0.00}\phantom{0}
& \stdnum{99.98}{0.01\phantom{0}}\ \rddiff{39.73}
& \stdnum{61.72}{0.03}\ \rddiff{1.00}\phantom{0}
& \phantom{0}40.73
& \phantom{00}\stdnum{1.42}{0.13\phantom{0}}\ \rddiff{63.08}
& \phantom{0}15.29 \\

OrthoGrad
& \stdnum{98.63}{0.09}\ \rddiff{1.35}\phantom{0}
& \stdnum{65.39}{1.73\phantom{0}}\ \rddiff{5.14}\phantom{0}
& \stdnum{52.54}{0.35}\ \rddiff{8.18}\phantom{0}
& \phantom{0}14.67
& \phantom{0}\stdnum{41.19}{1.65\phantom{0}}\ \rddiff{23.31}
& \phantom{0}15.60 \\

PCGrad
& \stdnum{88.39}{0.69}\ \rddiff{11.59}
& \stdnum{65.44}{0.80\phantom{0}}\ \rddiff{5.19}\phantom{0}
& \stdnum{50.93}{0.37}\ \rddiff{9.79}\phantom{0}
& \phantom{0}26.57
& \phantom{0}\stdnum{37.20}{0.80\phantom{0}}\ \rddiff{27.30}
& \phantom{00}7.96 \\

UAM
& \stdnum{97.61}{0.89}\ \rddiff{2.37}\phantom{0}
& \stdnum{69.88}{1.06\phantom{0}}\ \rddiff{9.63}\phantom{0}
& \stdnum{56.37}{0.38}\ \rddiff{4.35}\phantom{0}
& \phantom{0}16.35
& \phantom{0}\stdnum{46.63}{1.06\phantom{0}}\ \rddiff{17.87}
& \phantom{0}11.69 \\

\textbf{ROSU}
& \stdnum{98.70}{0.03}\ \rddiff{1.28}\phantom{0}
& \stdnum{60.72}{0.25\phantom{0}}\ \rddiff{0.47}\phantom{0}
& \stdnum{58.73}{0.27}\ \rddiff{1.99}\phantom{0}
& \phantom{00}\textbf{3.74}
& \phantom{0}\stdnum{47.31}{0.48\phantom{0}}\ \rddiff{17.19}
& \phantom{0}13.16 \\
\bottomrule
\end{tabular*}
}
\end{table}

\subsection{Component ablations}
\label{app:component-ablations}

The component ablations in Tables~\ref{tab:ablation_cifar10}, \ref{tab:ablation_cifar100}, and~\ref{tab:ablation_tinyimagenet} remove one part of the update at a time. The \mbox{No-Jacobian} variant uses \(\mathbf{v}=\widetilde{\mathbf{g}}_{r,\tau}\) without the Jacobian correction in~\eqref{eq:rosu-update}. The \(\beta\boldsymbol{\delta}_\tau\)-only variant removes the retain step \(-\eta\mathbf{v}\), and the \(-\eta\mathbf{v}\)-only variant sets \(\beta=0\). Tables~\ref{tab:ablation_tofu} and~\ref{tab:ablation_wmdp} repeat the ablation on TOFU and WMDP. In Tables~\ref{tab:ablation_cifar10}--\ref{tab:ablation_wmdp}, \(\boldsymbol{\delta}\) abbreviates the implemented \(\boldsymbol{\delta}_\tau\). Figure~\ref{fig:a1a2_ablation} in Section~\ref{sec:ablation} reports the derivative variants.

\begingroup\setlength{\intextsep}{6pt plus 2pt minus 2pt}
\begin{table}[H]
\centering
\caption{\textbf{Component ablation on CIFAR-10.} Check marks show which update terms each variant keeps, and bold marks the lowest \(\Delta\mathrm{Acc}\) in each protocol.}
\label{tab:ablation_cifar10}
{
\begin{tabular}{lcc@{\hspace{8pt}}ccccc}
\toprule
\textbf{Method} & $\beta\boldsymbol{\delta}$ & $\eta \mathbf{v}$
  & \textbf{$\Delta$Acc.}{\scriptsize$\downarrow$}
  & \textbf{RA} & \textbf{FA} & \textbf{TA} & \textbf{MIA-Eff.} \\
\midrulebeforeband
\rowcolor{cwBand}
\multicolumn{8}{c}{\bandfill{\textbf{\cw{Class-wise forgetting}}}} \\
\midruleafterband
ROSU                             & \cmark & \cmark     & 0.17 & \stdnum{99.96}{0.01}  & \stdnum{0.00}{0.00}  & \stdnum{95.01}{0.53} & \stdnum{100.00}{0.00} \\
\quad\textit{No-Jacobian}         & \cmark & \cmark$^*$ & \textbf{0.15}          & \stdnum{99.97}{0.01}  & \stdnum{0.00}{0.00}  & \stdnum{95.02}{0.51} & \stdnum{100.00}{0.00} \\
\quad\textit{$\beta\boldsymbol{\delta}$-only} & \cmark & \xmark     & 26.22         & \stdnum{86.96}{1.71}  & \stdnum{0.00}{0.00}  & \stdnum{81.96}{1.67} & \stdnum{100.00}{0.01} \\
\quad\textit{$-\eta \mathbf{v}$-only}      & \xmark & \cmark     & 99.96         & \stdnum{100.00}{0.00} & \stdnum{99.91}{0.07} & \stdnum{95.11}{0.39} & \stdnum{3.34}{2.26}   \\
\midrulebeforeband
\rowcolor{rdBand}
\multicolumn{8}{c}{\bandfill{\textbf{\rd{Random forgetting}}}} \\
\midruleafterband
ROSU                             & \cmark & \cmark     & \textbf{0.26} & \stdnum{99.84}{0.03} & \stdnum{95.13}{0.23} & \stdnum{94.29}{0.09} & \stdnum{7.88}{0.16} \\
\quad\textit{No-Jacobian}         & \cmark & \cmark$^*$ & 1.03          & \stdnum{99.88}{0.02} & \stdnum{96.01}{0.31} & \stdnum{94.40}{0.16} & \stdnum{6.78}{0.18} \\
\quad\textit{$\beta\boldsymbol{\delta}$-only} & \cmark & \xmark     & 285.64        & \stdnum{1.41}{0.37}  & \stdnum{0.32}{0.02}  & \stdnum{2.11}{0.65}  & \stdnum{0.31}{0.05} \\
\quad\textit{$-\eta \mathbf{v}$-only}      & \xmark & \cmark     & 4.93          & \stdnum{99.91}{0.01} & \stdnum{99.66}{0.03} & \stdnum{94.68}{0.19} & \stdnum{3.83}{0.20} \\
\bottomrule
\multicolumn{8}{l}{\footnotesize $^*$\,No-Jacobian: $\eta \mathbf{v}$ present but without Jacobian correction ($\mathbf{v}=\widetilde{\mathbf{g}}_{r,\tau}$).}
\end{tabular}
}
\end{table}

\begin{table}[H]
\centering
\caption{\textbf{Component ablation on CIFAR-100.} Same notation as Table~\ref{tab:ablation_cifar10}.}
\label{tab:ablation_cifar100}
{
\begin{tabular}{lcc@{\hspace{8pt}}ccccc}
\toprule
\textbf{Method} & $\beta\boldsymbol{\delta}$ & $\eta \mathbf{v}$
  & \textbf{$\Delta$Acc.}{\scriptsize$\downarrow$}
  & \textbf{RA} & \textbf{FA} & \textbf{TA} & \textbf{MIA-Eff.} \\
\midrulebeforeband
\rowcolor{cwBand}
\multicolumn{8}{c}{\bandfill{\textbf{\cw{Class-wise forgetting}}}} \\
\midruleafterband
ROSU                             & \cmark & \cmark     & \textbf{0.14} & \stdnum{99.98}{0.00} & \stdnum{0.12}{0.26}  & \stdnum{77.78}{0.13} & \stdnum{100.00}{0.00} \\
\quad\textit{No-Jacobian}         & \cmark & \cmark$^*$ & 1.03          & \stdnum{99.98}{0.00} & \stdnum{0.14}{0.31}  & \stdnum{76.91}{0.22} & \stdnum{99.88}{0.26}  \\
\quad\textit{$\beta\boldsymbol{\delta}$-only} & \cmark & \xmark     & 16.46         & \stdnum{93.49}{0.91} & \stdnum{1.16}{2.37}  & \stdnum{68.99}{0.69} & \stdnum{99.34}{1.11}  \\
\quad\textit{$-\eta \mathbf{v}$-only}      & \xmark & \cmark     & 99.96         & \stdnum{99.98}{0.00} & \stdnum{99.94}{0.13} & \stdnum{77.82}{0.17} & \stdnum{14.50}{4.05}  \\
\midrulebeforeband
\rowcolor{rdBand}
\multicolumn{8}{c}{\bandfill{\textbf{\rd{Random forgetting}}}} \\
\midruleafterband
ROSU                             & \cmark & \cmark     & \textbf{0.67} & \stdnum{99.78}{0.00} & \stdnum{77.98}{0.08} & \stdnum{76.85}{0.18} & \stdnum{41.37}{0.75}  \\
\quad\textit{No-Jacobian}         & \cmark & \cmark$^*$ & 12.86         & \stdnum{99.96}{0.01} & \stdnum{90.01}{0.28} & \stdnum{77.19}{0.13} & \stdnum{33.09}{0.69}  \\
\quad\textit{$\beta\boldsymbol{\delta}$-only} & \cmark & \xmark     & 251.97        & \stdnum{0.84}{0.13}  & \stdnum{0.63}{0.05}  & \stdnum{0.90}{0.11}  & \stdnum{33.53}{46.57} \\
\quad\textit{$-\eta \mathbf{v}$-only}      & \xmark & \cmark     & 8.75          & \stdnum{99.74}{0.01} & \stdnum{86.17}{0.19} & \stdnum{76.70}{0.16} & \stdnum{36.27}{0.48}  \\
\bottomrule
\end{tabular}
}
\end{table}

\begin{table}[H]
\centering
\caption{\textbf{Component ablation on Tiny-ImageNet.} Same notation as Table~\ref{tab:ablation_cifar10}.}
\label{tab:ablation_tinyimagenet}
{
\begin{tabular}{lcc@{\hspace{8pt}}ccccc}
\toprule
\textbf{Method} & $\beta\boldsymbol{\delta}$ & $\eta \mathbf{v}$
  & \textbf{$\Delta$Acc.}{\scriptsize$\downarrow$}
  & \textbf{RA} & \textbf{FA} & \textbf{TA} & \textbf{MIA-Eff.} \\
\midrulebeforeband
\rowcolor{cwBand}
\multicolumn{8}{c}{\bandfill{\textbf{\cw{Class-wise forgetting}}}} \\
\midruleafterband
ROSU                             & \cmark & \cmark     & \textbf{0.01} & \stdnum{99.97}{0.00} & \stdnum{0.00}{0.00}  & \stdnum{61.89}{0.22} & \stdnum{100.00}{0.00} \\
\quad\textit{No-Jacobian}         & \cmark & \cmark$^*$ & \textbf{0.01}          & \stdnum{99.98}{0.00} & \stdnum{0.00}{0.00}  & \stdnum{61.88}{0.20} & \stdnum{100.00}{0.00} \\
\quad\textit{$\beta\boldsymbol{\delta}$-only} & \cmark & \xmark     & 148.38        & \stdnum{8.82}{23.98} & \stdnum{0.00}{0.00}  & \stdnum{4.67}{11.63} & \stdnum{50.00}{50.00} \\
\quad\textit{$-\eta \mathbf{v}$-only}      & \xmark & \cmark     & 100.29        & \stdnum{99.98}{0.00} & \stdnum{99.97}{0.08} & \stdnum{61.57}{0.14} & \stdnum{0.83}{0.61}   \\
\midrulebeforeband
\rowcolor{rdBand}
\multicolumn{8}{c}{\bandfill{\textbf{\rd{Random forgetting}}}} \\
\midruleafterband
ROSU                             & \cmark & \cmark     & \textbf{3.74} & \stdnum{98.70}{0.03} & \stdnum{60.72}{0.25} & \stdnum{58.73}{0.27} & \stdnum{47.31}{0.48}  \\
\quad\textit{No-Jacobian}         & \cmark & \cmark$^*$ & 8.16          & \stdnum{99.19}{0.07} & \stdnum{65.88}{0.03} & \stdnum{58.98}{0.45} & \stdnum{44.49}{0.63}  \\
\quad\textit{$\beta\boldsymbol{\delta}$-only} & \cmark & \xmark     & 219.45        & \stdnum{0.50}{0.00}  & \stdnum{0.50}{0.00}  & \stdnum{0.50}{0.00}  & \stdnum{66.51}{46.68} \\
\quad\textit{$-\eta \mathbf{v}$-only}      & \xmark & \cmark     & 31.69         & \stdnum{99.08}{0.00} & \stdnum{90.24}{0.10} & \stdnum{59.92}{0.34} & \stdnum{23.14}{0.36}  \\
\bottomrule
\end{tabular}
}
\end{table}

\begin{table}[H]
\centering
\caption{\textbf{Component ablation on TOFU.} Same variants as Table~\ref{tab:ablation_cifar10}, and bold marks the best Agg.}
\label{tab:ablation_tofu}

{
\begin{tabular}{lcc@{\hspace{8pt}}cccc}
\toprulebeforeband
\rowcolor{knBand}
\multicolumn{7}{c}{\bandfill{\textbf{\kn{TOFU}}}} \\
\midruleafterband
\textbf{Method} & $\beta\boldsymbol{\delta}$ & $\eta \mathbf{v}$
  & \textbf{Agg.}{\scriptsize$\uparrow$}
  & \textbf{Mem.}{\scriptsize$\uparrow$}
  & \textbf{Priv.}{\scriptsize$\uparrow$}
  & \textbf{Util.}{\scriptsize$\uparrow$} \\
\midrule
ROSU                             & \cmark & \cmark      & \textbf{0.58} & 0.38 & 0.74 & 0.82 \\
\quad\textit{No-Jacobian}         & \cmark & \cmark$^*$  & 0.53 & 0.53 & 0.54 & 0.53 \\
\quad\textit{$\beta\boldsymbol{\delta}$-only} & \cmark & \xmark      & 0.00 & 0.98 & 0.48 & 0.00 \\
\quad\textit{$-\eta \mathbf{v}$-only}      & \xmark & \cmark      & 0.00 & 0.00 & 0.43 & 0.88 \\
\bottomrule
\end{tabular}
}

\end{table}

\begin{table}[H]
\centering
\caption{\textbf{Component ablation on WMDP.} Same variants as Table~\ref{tab:ablation_cifar10}. As in \(\Delta\)Score, Bio and Cyber are judged by their distance from the chance level of \(0.25\).}
\label{tab:ablation_wmdp}

{
\begin{tabular}{lccccc}
\toprulebeforeband
\rowcolor{knBand}
\multicolumn{6}{c}{\bandfill{\textbf{\kn{WMDP}}}} \\
\midruleafterband
\textbf{Method} & $\beta\boldsymbol{\delta}$ & $\eta \mathbf{v}$
  & \textbf{MMLU}{\scriptsize$\uparrow$}
  & \textbf{Bio}
  & \textbf{Cyber} \\
\midrule
ROSU                             & \cmark & \cmark      & 0.570 & 0.289 & 0.246 \\
\quad\textit{No-Jacobian}         & \cmark & \cmark$^*$  & 0.577 & 0.420 & 0.261 \\
\quad\textit{$\beta\boldsymbol{\delta}$-only} & \cmark & \xmark      & 0.255 & 0.266 & 0.246 \\
\quad\textit{$-\eta \mathbf{v}$-only}      & \xmark & \cmark      & 0.252 & 0.251 & 0.266 \\
\bottomrule
\end{tabular}
}

\end{table}

\endgroup

\subsection{Perturbation radius and numerical safeguards}
\label{app:sensitivity}

Table~\ref{tab:rho_sensitivity} reports ROSU's sensitivity to the perturbation radius \(\rho\) on TOFU. Table~\ref{tab:rho_sensitivity_vision} reports the corresponding sweep on CIFAR-100 and CIFAR-10 random forgetting, with checkpoints selected for each radius. In this vision sweep, amplification is active only during the early epochs and switched off afterwards. Within each dataset, all runs use the same learning rate and \(c_t\). Means and standard deviations are over three seeds. \(\Delta\mathrm{Acc}\) remains close to the main-table values for every radius. Figure~\ref{fig:grad_norms} audits the numerical scales relevant to the stabilizer \(\tau\) and degeneracy threshold \(\varepsilon_q\). We denote the recorded projected-gradient norm by \(n_q\). In the audited records, \(n_q>\varepsilon_q\) throughout. The retain-gradient squared norm exceeds \(\tau\) on the vision benchmarks and TOFU. On WMDP, some records fall below \(\tau\) (Figure~\ref{fig:grad_norms}). This is consistent with a representation-level retain loss initially near zero. The stabilizer keeps the retain projector defined on these records. On the vision benchmarks, random forgetting yields smaller recorded projected-gradient norms than class-wise forgetting. This is consistent with higher forget--retain coupling.

\begin{figure}[htbp]
\centering
\includegraphics[width=\linewidth]{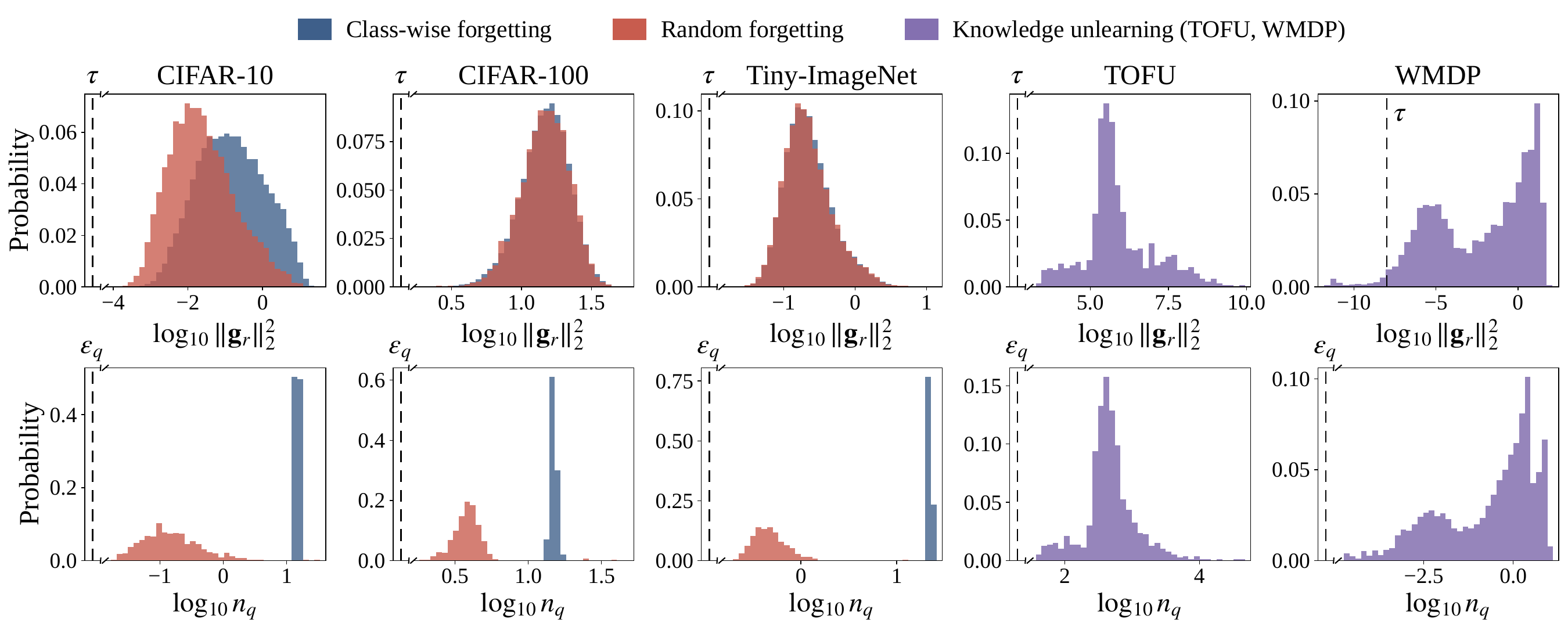}
\caption{\textbf{Gradient-norm audit.} Distributions of \(\|\mathbf{g}_r\|_2^2\) in the upper row and of the projected-gradient norm \(n_q\) in the lower row. The vision panels overlay \cw{class-wise} and \rd{random} forgetting. Dashed lines mark \(\tau=10^{-8}\) in the upper row and \(\varepsilon_q=10^{-6}\) in the lower row.}
\label{fig:grad_norms}
\end{figure}

\subsection{Retain mini-batch size and orthogonality}
\label{app:retain-batch-size}

\begin{figure}[tbp]
\centering
\includegraphics[width=\linewidth]{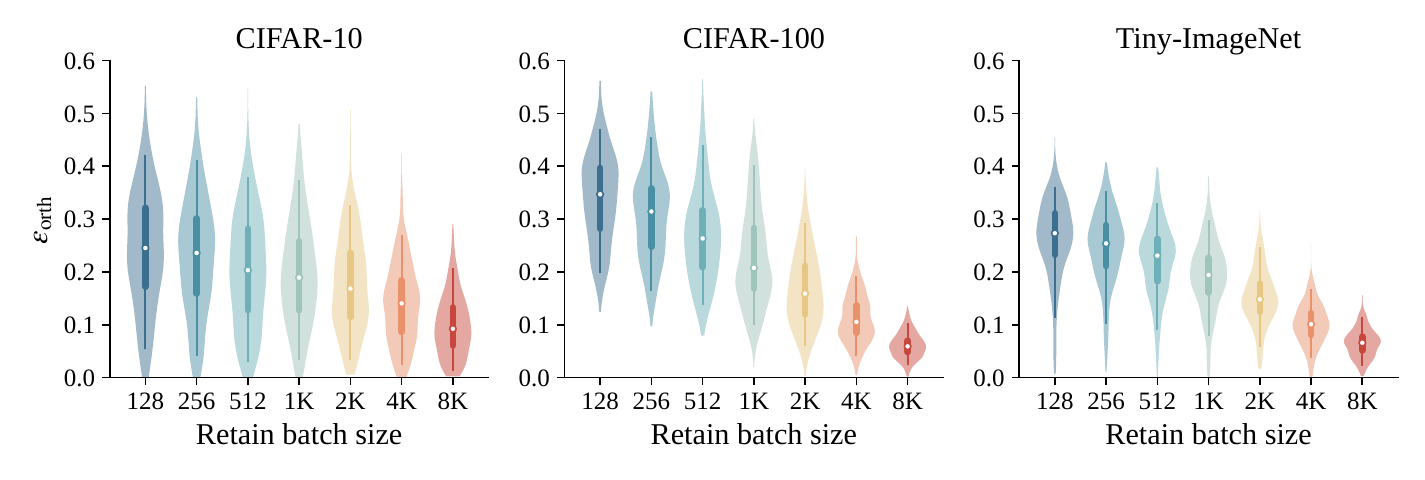}
\caption{\textbf{Stochastic orthogonality audit.} Full-retain orthogonality error \(\varepsilon_{\mathrm{orth}}\) for each retain batch size at the pre-unlearning checkpoint under random forgetting. White dots mark medians, and thick and thin sticks mark the interquartile and 5--95\% ranges.}
\label{fig:orth_violin}
\end{figure}

\begin{figure}[tbp]
\centering
\includegraphics[width=\textwidth]{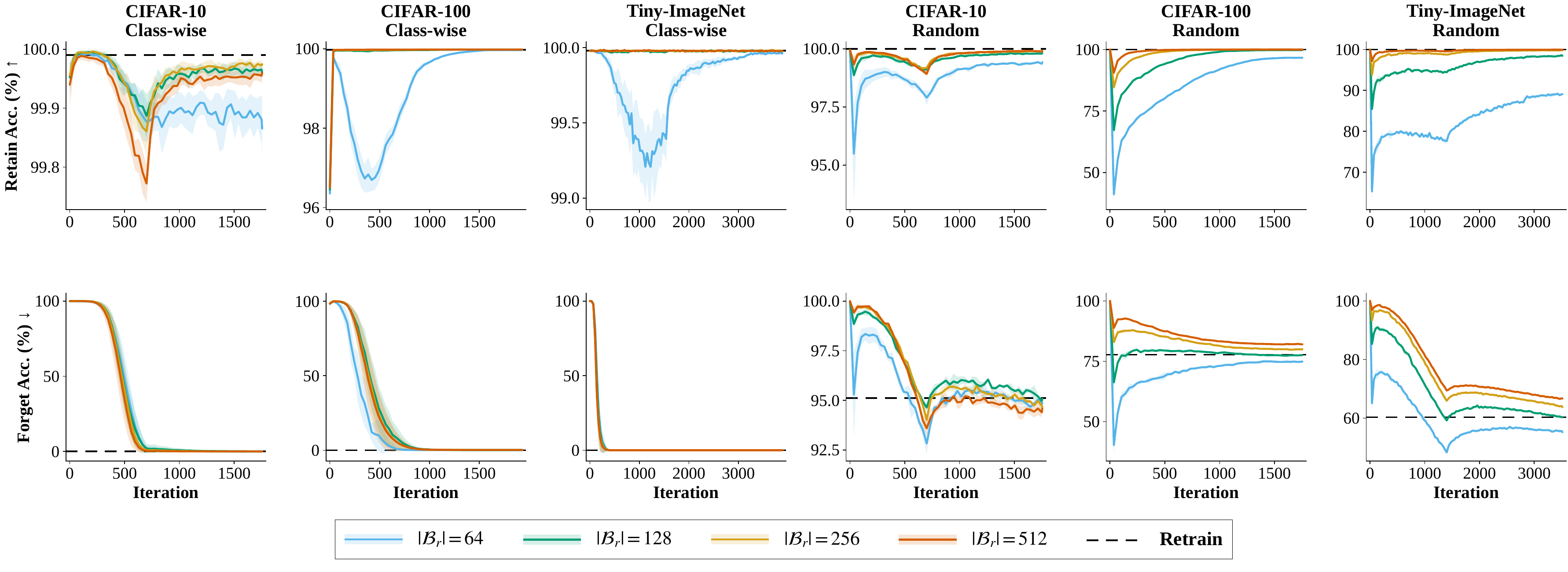}
\caption{\textbf{Sensitivity to the retain mini-batch size.} Retain and forget accuracy during unlearning for \(|\mathcal{B}_r|\in\{64,128,256,512\}\). Curves and bands are the mean and one standard deviation over ten forgotten classes or three seeds. Dashed lines mark the Retrain values of Tables~\ref{tab:cifar10_unlearning}, \ref{tab:cifar100_unlearning}, and~\ref{tab:tinyimagenet_unlearning}.}
\label{fig:retain_bs_sweep}
\end{figure}

ROSU estimates the retain direction from a mini-batch and regularizes the resulting projection. We examine the full-retain residual of this estimate and its effect on training. Figure~\ref{fig:orth_violin} measures the residual at the pre-unlearning checkpoint under random forgetting, with the model held fixed so that only the sampling varies. Figure~\ref{fig:retain_bs_sweep} tracks retain--forget trajectories across retain batch sizes under both class-wise and random forgetting.

\looseness=-1 For the pre-unlearning audit, we evaluate all gradients at the pre-unlearning checkpoint, with \(|\mathcal{B}_r|\in\{128,256,512,1024,2048,4096,8192\}\). Let \(\mathbf{g}_r^{\mathrm{full}}:=\nabla_{\mathbf{w}}L_r(\mathbf{w};\mathcal{D}_r)\), \(\mathbf{g}_f^{\mathrm{mini}}:=\nabla_{\mathbf{w}}L_f(\mathbf{w};\mathcal{B}_f)\), and \(\mathbf{g}_r^{\mathrm{mini}}:=\nabla_{\mathbf{w}}L_r(\mathbf{w};\mathcal{B}_r)\). We sample 500 forget/retain mini-batch pairs per cell, with forget batch size 128. The diagnostic implementation uses the same stabilizer \(\tau=10^{-8}\) as training and computes
\begin{equation*}
\begin{aligned}
\mathbf{q}_{\tau}^{\mathrm{mini}} &:=\mathbf{g}_f^{\mathrm{mini}} -\frac{(\mathbf{g}_f^{\mathrm{mini}})^\top\mathbf{g}_r^{\mathrm{mini}}} {\|\mathbf{g}_r^{\mathrm{mini}}\|_2^2+\tau}\,\mathbf{g}_r^{\mathrm{mini}}.
\end{aligned}
\end{equation*}
The direction diagnostic uses samples with \(\|\mathbf{q}_{\tau}^{\mathrm{mini}}\|_2>\varepsilon_q=10^{-6}\). For these samples, define \(\mathbf{u}_{\tau}^{\mathrm{mini}}:=\frac{\mathbf{q}_{\tau}^{\mathrm{mini}}}{\|\mathbf{q}_{\tau}^{\mathrm{mini}}\|_2}\).

The full-retain residual is bounded by a gradient-estimation term and a regularization term (Appendix~\ref{app:regularization}),
\begin{equation}
\label{eq:mini-full-transfer}
\begin{aligned}
\left|(\mathbf{g}_r^{\mathrm{full}})^\top\mathbf{u}_{\tau}^{\mathrm{mini}}\right| &\le\|\mathbf{g}_r^{\mathrm{full}}-\mathbf{g}_r^{\mathrm{mini}}\|_2 +\frac{\tau\left|(\mathbf{g}_f^{\mathrm{mini}})^\top\mathbf{g}_r^{\mathrm{mini}}\right|}{(\|\mathbf{g}_r^{\mathrm{mini}}\|_2^2+\tau)\|\mathbf{q}_{\tau}^{\mathrm{mini}}\|_2}.
\end{aligned}
\end{equation}
When \(\mathbf{g}_r^{\mathrm{full}}\neq\mathbf{0}\), Figure~\ref{fig:orth_violin} reports
\begin{equation*}
\begin{aligned}
\varepsilon_{\mathrm{orth}} &:=\frac{\left|(\mathbf{g}_r^{\mathrm{full}})^\top\mathbf{u}_{\tau}^{\mathrm{mini}}\right|} {\|\mathbf{g}_r^{\mathrm{full}}\|_2}\\
&=\left|\cos\bigl(\angle(\mathbf{g}_r^{\mathrm{full}},\mathbf{u}_{\tau}^{\mathrm{mini}})\bigr)\right|\\
&\le\frac{\|\mathbf{g}_r^{\mathrm{full}}-\mathbf{g}_r^{\mathrm{mini}}\|_2 +\left|(\mathbf{g}_r^{\mathrm{mini}})^\top\mathbf{u}_{\tau}^{\mathrm{mini}}\right|} {\|\mathbf{g}_r^{\mathrm{full}}\|_2}.
\end{aligned}
\end{equation*}
Here 0 means orthogonality to the full retain gradient and 1 means collinearity, either parallel or antiparallel. This is a relative error, scaled by the full retain-gradient norm.

The distributions in Figure~\ref{fig:orth_violin} concentrate toward zero as \(|\mathcal{B}_r|\) increases across all three datasets, and CIFAR-100 random forgetting shows the largest visible shift.

Figure~\ref{fig:retain_bs_sweep} shows how the same sampling noise affects training. In class-wise forgetting, forgetting succeeds with every retain batch size, and forget accuracy drops to near zero on every dataset. The main batch-size effect is on retain stability. Very small retain batches can introduce transient retain-accuracy drops on CIFAR-100 and Tiny-ImageNet, and can leave a small residual retain gap on CIFAR-10. In random forgetting, the effect is stronger and more systematic. Larger retain batches give a more stable retain-gradient estimate, so retain accuracy is higher and optimization is smoother. With smaller retain batches, forget accuracy is often lower and retention is weaker. This is consistent with the coupling measured in Figure~\ref{fig:motivation_a}. When the coupling between the forget and retain gradients is higher, the quality of the retain-gradient estimate matters more.

\subsection{\texorpdfstring{Extension with sampled rank-\(k\) retain protection}{Extension with sampled rank-k retain protection}}
\label{app:rank-k-empirical}

The rank-one construction of Section~\ref{sec:inner} protects a single retain direction. It extends to a protected subspace as follows. Let \(\mathbf{U}\in\mathbb{R}^{p\times k'}\) have \(k^{\prime}\le k\) orthonormal columns and set \(\mathbf{P}_{\mathbf{U}}:=\mathbf{U}\mathbf{U}^\top\) and \(\mathbf{q}_{\mathbf{U}}:=(\mathbf{I}-\mathbf{P}_{\mathbf{U}})\mathbf{g}_f\), where \(k\) is the number of sampled directions and \(k^{\prime}\) is their retained rank. The protected-subspace inner problem is
\begin{equation*}
\begin{aligned}
&\max\limits_{\boldsymbol{\delta}}\;\mathbf{g}_f^\top\boldsymbol{\delta}\quad\text{s.t.}\quad\|\boldsymbol{\delta}\|_2\le\rho,\quad\mathbf{U}^\top\boldsymbol{\delta}=\mathbf{0}.
\end{aligned}
\end{equation*}
On its feasible set \(\mathbf{g}_f^\top\boldsymbol{\delta}=\mathbf{q}_{\mathbf{U}}^\top\boldsymbol{\delta}\), so the Cauchy--Schwarz argument of Appendix~\ref{app:core-geometry} applies without change. If \(\mathbf{q}_{\mathbf{U}}\neq\mathbf{0}\), the unique maximizer is \(\boldsymbol{\delta}_{\mathbf{U}}=\rho\,\frac{\mathbf{q}_{\mathbf{U}}}{\|\mathbf{q}_{\mathbf{U}}\|_2}\) with value \(\rho\|\mathbf{q}_{\mathbf{U}}\|_2\), and \(\mathbf{h}^\top\boldsymbol{\delta}_{\mathbf{U}}=0\) for every \(\mathbf{h}\in\operatorname{span}(\mathbf{U})\). If \(\mathbf{q}_{\mathbf{U}}=\mathbf{0}\), every feasible perturbation has zero forget gain. At a fixed point and fixed gradients, enlarging a nested protected subspace shrinks the feasible set, so the optimal forget gain cannot increase.

The component of \(\mathbf g_r\) outside \(\operatorname{span}(\mathbf U)\) determines the first-order retain effect. Since \(\mathbf{P}_{\mathbf{U}}\mathbf{g}_r\) lies in \(\operatorname{span}(\mathbf{U})\) and is orthogonal to \(\boldsymbol{\delta}_{\mathbf{U}}\),
\[
\begin{aligned}
\mathbf{g}_r^\top\boldsymbol{\delta}_{\mathbf{U}}&=\bigl((\mathbf{I}-\mathbf{P}_{\mathbf{U}})\mathbf{g}_r\bigr)^\top\boldsymbol{\delta}_{\mathbf{U}}\\
&\le\rho\,\|(\mathbf{I}-\mathbf{P}_{\mathbf{U}})\mathbf{g}_r\|_2,
\end{aligned}
\]
and this bound vanishes exactly when \(\mathbf{g}_r\in\operatorname{span}(\mathbf{U})\). If \(L_r\) is \(C^2\) on an open neighborhood of the segment \(\{\mathbf{w}+\zeta\boldsymbol{\delta}_{\mathbf{U}}:\zeta\in[0,1]\}\) with Hessian operator norm at most \(M_r\) on it, then
\[
\begin{aligned}
L_r(\mathbf{w}+\boldsymbol{\delta}_{\mathbf{U}})&\le L_r(\mathbf{w})+\rho\,\|(\mathbf{I}-\mathbf{P}_{\mathbf{U}})\mathbf{g}_r\|_2+\frac{M_r}{2}\rho^2 .
\end{aligned}
\]

To test whether protecting multiple retain-gradient directions lowers \(\Delta\mathrm{Acc}\) in a high-coupling regime, we run sampled rank-\(k\) ROSU on CIFAR-100 random forgetting. We hold the pretrained checkpoint, data split, perturbation radius, learning-rate schedule, training budget, and seed protocol fixed. The \(k=1\) row in Table~\ref{tab:rank_k_cifar100_rd} is copied from the main-table rank-1 ROSU run. The rank-\(k\) diagnostic reruns \(k\in\{2,4,8\}\) with retain subspaces spanned by sampled per-sample retain gradients. These runs use the same amplification schedule as the main-table run.

For sampled retain examples \((\mathbf x_i,y_i)\), define \(\mathbf g_r^{(i)}:=\nabla_{\mathbf w}L_r(\mathbf w;\{(\mathbf x_i,y_i)\})\) at the current weights \(\mathbf w\), using the supervised singleton loss. For \(k\ge2\), let \(\mathbf U\) have orthonormal columns spanning the retained sampled directions from \([\mathbf g_r^{(1)},\ldots,\mathbf g_r^{(k)}]\). On steps with \(\mathbf q_{\mathbf U}\neq\mathbf{0}\), define
\[
\begin{aligned}
\mathbf{P}_{\mathbf{U}}&=\mathbf{U}\mathbf{U}^\top,\qquad \mathbf{q}_{\mathbf{U}}=(\mathbf{I}-\mathbf{P}_{\mathbf{U}})\mathbf{g}_f,\qquad \mathbf{P}_{\perp,\mathbf{U}}=\frac{\mathbf{q}_{\mathbf{U}}\mathbf{q}_{\mathbf{U}}^\top}{\|\mathbf{q}_{\mathbf{U}}\|_2^2}.
\end{aligned}
\]
Under the threshold rule of Algorithm~\ref{alg:rosu_full}, nonfallback steps satisfy \(\|\mathbf q_{\mathbf U}\|_2>\varepsilon_q\). On these steps, define \(\widetilde{\mathbf g}_{r,\mathbf U}:=\left.\nabla_{\mathbf z}L_r(\mathbf z;\mathcal B_r)\right|_{\mathbf z=\mathbf w+\boldsymbol{\delta}_{\mathbf U}}\). The rank-\(k\) relaxation drops the differential of the protected-subspace projector and replaces the forget Hessian by \(\mathbf I\). The transported update uses
\[
\begin{aligned}
\mathbf{v}&=\Bigl[\mathbf{I}+\frac{\rho}{\|\mathbf{q}_{\mathbf{U}}\|_2}\bigl(\mathbf{I}-\mathbf{P}_{\mathbf{U}}-\mathbf{P}_{\perp,\mathbf{U}}\bigr)\Bigr]\widetilde{\mathbf{g}}_{r,\mathbf U} ,
\end{aligned}
\]
and degenerate steps use the same retain-gradient fallback as Algorithm~\ref{alg:rosu_full}.

The ideal surrogate perturbation protects every direction in the retained span of \(\mathbf U\) exactly. The mini-batch mean gradient \(\mathbf g_r^{\mathrm{mini}}\) contributes the residual term \(\rho\|(\mathbf I-\mathbf P_{\mathbf U})\mathbf g_r^{\mathrm{mini}}\|_2\) in the bound above. This term is zero when the sampled span contains that mean gradient. Table~\ref{tab:rank_k_cifar100_rd} reports the results.

\begin{table}[H]
\centering
\caption{\textbf{Rank-\(k\) retain subspace on CIFAR-100.} The \(k=1\) row is the main-table run, and bold marks the lowest \(\Delta\mathrm{Acc}\) and its \(k\).}
\label{tab:rank_k_cifar100_rd}
\begin{tabular}{cccccc}
\toprulebeforeband
\rowcolor{rdBand}
\multicolumn{6}{c}{\bandfill{\textbf{\rd{CIFAR-100 random forgetting}}}} \\
\midruleafterband
\(k\) & \textbf{RA} & \textbf{FA} & \textbf{TA} & \textbf{$\Delta$Acc.}{\scriptsize$\downarrow$} & \textbf{MIA-Eff.} \\
\midrule
\(1\) & \stdnum{99.78}{0.00} & \stdnum{77.98}{0.08} & \stdnum{76.85}{0.18} & \(0.67\) & \stdnum{41.37}{0.75} \\
\(2\) & \stdnum{99.94}{0.01} & \stdnum{77.47}{0.33} & \stdnum{76.62}{0.14} & \(0.31\) & \stdnum{42.54}{0.56} \\
\(4\) & \stdnum{99.94}{0.01} & \stdnum{77.88}{0.40} & \stdnum{76.92}{0.22} & \(0.48\) & \stdnum{42.81}{0.74} \\
\(\mathbf{8}\) & \stdnum{99.94}{0.00} & \stdnum{77.68}{0.23} & \stdnum{76.54}{0.07} & \(\mathbf{0.18}\) & \stdnum{42.32}{0.85} \\
\bottomrule
\end{tabular}
\end{table}

Each rank-\(k\) rerun has a smaller \(\Delta\mathrm{Acc}\) than the main-table run. The \(\Delta\mathrm{Acc}\) values vary nonmonotonically with \(k\), and \(k=8\) gives the lowest, which is the closest aggregate match to Retrain. MIA-Eff. varies little across \(k\).

\section{Theory and derivations}
\label{app:theory}

This appendix proves Proposition~\ref{prop:inner}, Theorem~\ref{thm:retain-damage}, and Proposition~\ref{prop:approx-grad}, and derives the quantities used in Sections~\ref{sec:inner}, \ref{sec:local}, and~\ref{sec:method}. Appendix~\ref{app:core-geometry} treats the inner problem and the surrogate retain bounds, Appendix~\ref{app:exact-relaxed} the exact and relaxed retain corrections, and Appendix~\ref{app:regularization} the regularized implementation, the fallback, and mini-batch residuals. We hold \(\rho\) fixed within each step and suppress \(\mathbf w\) when all quantities share the same evaluation point.

\subsection{Inner geometry and surrogate retain bounds}
\label{app:core-geometry}

Here we prove Proposition~\ref{prop:inner} and Theorem~\ref{thm:retain-damage} and derive the gain ratio, the forgetting step, and the near-orthogonality gap used in Sections~\ref{sec:inner}, \ref{sec:local}, and~\ref{sec:method}.

\begin{proof}[Proof of Proposition~\ref{prop:inner}]
A feasible \(\boldsymbol{\delta}\) satisfies \(\mathbf g_r^\top\boldsymbol{\delta}=0\), hence \(\mathbf P_r\boldsymbol{\delta}=\mathbf{0}\). For every such \(\boldsymbol{\delta}\),
\[
\begin{aligned}
\mathbf{g}_f^\top\boldsymbol{\delta}&=\mathbf{g}_f^\top(\mathbf{I}-\mathbf{P}_r)\boldsymbol{\delta}\\
&=\mathbf{q}^\top\boldsymbol{\delta}\\
&\le\|\mathbf{q}\|_2\,\|\boldsymbol{\delta}\|_2\\
&\le\rho\,\|\mathbf{q}\|_2 ,
\end{aligned}
\]
where the first equality uses \(\mathbf P_r\boldsymbol{\delta}=\mathbf{0}\), the second uses the symmetry of \(\mathbf I-\mathbf P_r\), the first inequality is Cauchy--Schwarz, and the last uses \(\|\boldsymbol{\delta}\|_2\le\rho\).

Next, we show that the candidate \(\rho\frac{\mathbf q}{\|\mathbf q\|_2}\) attains this bound. It is feasible, since \(\mathbf g_r^\top\mathbf q=0\) and its norm is \(\rho\). Its forget gain is
\[
\begin{aligned}
\mathbf{g}_f^\top\Bigl(\rho\,\frac{\mathbf{q}}{\|\mathbf{q}\|_2}\Bigr)&=\frac{\rho}{\|\mathbf{q}\|_2}\,\mathbf{g}_f^\top(\mathbf{I}-\mathbf{P}_r)^\top(\mathbf{I}-\mathbf{P}_r)\,\mathbf{g}_f\\
&=\frac{\rho}{\|\mathbf{q}\|_2}\,\|\mathbf{q}\|_2^2\\
&=\rho\,\|\mathbf{q}\|_2 ,
\end{aligned}
\]
\looseness=-1 where the first equality uses that \(\mathbf I-\mathbf P_r\) is symmetric and idempotent, and the second uses the definition of \(\mathbf q\). Equality in the Cauchy--Schwarz step requires a nonnegative multiple of \(\mathbf q\), and equality in the last step requires norm \(\rho\). Since \(\mathbf q\neq\mathbf{0}\), the candidate is the only point that satisfies both.
\end{proof}

For the gain ratio, decompose \(\mathbf g_f=\mathbf P_r\mathbf g_f+\mathbf q\). The two components are orthogonal, and \(\|\mathbf P_r\mathbf g_f\|_2=|\mathbf u_r^\top\mathbf g_f|=\|\mathbf g_f\|_2|\cos\theta|\), so
\[
\begin{aligned}
\|\mathbf{q}\|_2^2&=\|\mathbf{g}_f\|_2^2-\|\mathbf{P}_r\mathbf{g}_f\|_2^2\\
&=\|\mathbf{g}_f\|_2^2\bigl(1-\cos^2\theta\bigr)\\
&=\|\mathbf{g}_f\|_2^2\sin^2\theta ,
\end{aligned}
\]
where the first equality is Pythagoras. Since \(\theta\in[0,\pi]\), \(\sin\theta\ge0\) and \(\|\mathbf q\|_2=\|\mathbf g_f\|_2\sin\theta\). The constrained and unconstrained gains are \(\rho\|\mathbf q\|_2\) and \(\rho\|\mathbf g_f\|_2\), so their ratio is \(\sin\theta\) at every common positive radius.

For retain-neutral amplification with \(\beta>0\), the argument above at radius \(\beta\rho\) shows that the unique maximizer is \(\beta\boldsymbol{\delta}_{\mathrm{rosu}}\), with gain \(\beta\rho\|\mathbf q\|_2\). For \(\beta=0\), the feasible set is \(\{\mathbf{0}\}\).

\begin{proof}[Proof of Theorem~\ref{thm:retain-damage}]
Fix a perturbation \(\boldsymbol{\delta}\) whose segment \(\{\mathbf w+\zeta\boldsymbol{\delta}:\zeta\in[0,1]\}\) satisfies the Hessian bound. Taylor's theorem with integral remainder gives
\[
\begin{aligned}
L_r(\mathbf{w}+\boldsymbol{\delta})&=L_r(\mathbf{w})+\mathbf{g}_r^\top\boldsymbol{\delta}+R_{\boldsymbol{\delta}},\\
R_{\boldsymbol{\delta}}&=\int\limits_0^1(1-\zeta)\,\boldsymbol{\delta}^\top\nabla^2L_r(\mathbf{w}+\zeta\boldsymbol{\delta})\,\boldsymbol{\delta}\,\mathrm{d}\zeta .
\end{aligned}
\]
The Hessian bound controls the remainder,
\[
\begin{aligned}
|R_{\boldsymbol{\delta}}|&\le\int\limits_0^1(1-\zeta)\,M_r\,\|\boldsymbol{\delta}\|_2^2\,\mathrm{d}\zeta\\
&=\frac{M_r}{2}\,\|\boldsymbol{\delta}\|_2^2 .
\end{aligned}
\]
For~\itemref{thm:retain-damage-i}, take \(\boldsymbol{\delta}=\boldsymbol{\delta}_{\mathrm{rosu}}\). Then
\[
\begin{aligned}
\bigl|F_{\mathrm{rosu}}(\mathbf{w})-L_r(\mathbf{w})\bigr|&=\bigl|\mathbf{g}_r^\top\boldsymbol{\delta}_{\mathrm{rosu}}+R_{\boldsymbol{\delta}_{\mathrm{rosu}}}\bigr|\\
&=\bigl|R_{\boldsymbol{\delta}_{\mathrm{rosu}}}\bigr|\\
&\le\frac{M_r}{2}\,\rho^2 ,
\end{aligned}
\]
where the second equality uses \(\mathbf g_r^\top\boldsymbol{\delta}_{\mathrm{rosu}}=\frac{\rho}{\|\mathbf q\|_2}\mathbf g_r^\top\mathbf q=0\) (the first claim of~\itemref{thm:retain-damage-i}), and the inequality applies the remainder bound with \(\|\boldsymbol{\delta}_{\mathrm{rosu}}\|_2=\rho\).

For~\itemref{thm:retain-damage-ii}, take \(\boldsymbol{\delta}=\boldsymbol{\delta}_{\mathrm{std}}\) as well. Subtracting the two expansions gives
\[
\begin{aligned}
F_{\mathrm{std}}(\mathbf{w})-F_{\mathrm{rosu}}(\mathbf{w})&=\mathbf{g}_r^\top\boldsymbol{\delta}_{\mathrm{std}}+R_{\boldsymbol{\delta}_{\mathrm{std}}}-R_{\boldsymbol{\delta}_{\mathrm{rosu}}}\\
&=\rho\,\|\mathbf{g}_r\|_2\cos\theta+R_{\boldsymbol{\delta}_{\mathrm{std}}}-R_{\boldsymbol{\delta}_{\mathrm{rosu}}}\\
&\ge\rho\,\|\mathbf{g}_r\|_2\cos\theta-|R_{\boldsymbol{\delta}_{\mathrm{std}}}|-|R_{\boldsymbol{\delta}_{\mathrm{rosu}}}|\\
&\ge\rho\,\|\mathbf{g}_r\|_2\cos\theta-M_r\rho^2 ,
\end{aligned}
\]
where the first equality uses \(\mathbf g_r^\top\boldsymbol{\delta}_{\mathrm{rosu}}=0\), the second uses~\eqref{eq:std-retain-effect}, and the last inequality applies the remainder bound to both segments with \(\|\boldsymbol{\delta}_{\mathrm{std}}\|_2=\|\boldsymbol{\delta}_{\mathrm{rosu}}\|_2=\rho\).
\end{proof}

Whenever \(\rho\|\mathbf g_r\|_2\cos\theta>M_r\rho^2\), the bound in~\itemref{thm:retain-damage-ii} is positive and therefore \(F_{\mathrm{std}}(\mathbf w)>F_{\mathrm{rosu}}(\mathbf w)\). For \(M_r>0\), this condition reads \(\cos\theta>\frac{M_r\rho}{\|\mathbf g_r\|_2}\), which requires positive coupling and \(\rho<\frac{\|\mathbf g_r\|_2\cos\theta}{M_r}\). If \(\cos\theta\le0\), the standard perturbation has a nonpositive first-order retain effect.

The curvature bound also covers the forgetting step. If \(0\le\beta\le1\), the amplified segment lies within the original segment. For \(\beta>1\), require that \(L_r\) is \(C^2\) near \(\{\mathbf w+\zeta\beta\boldsymbol{\delta}_{\mathrm{rosu}}:\zeta\in[0,1]\}\) and that the Hessian bound holds on it. With \(\boldsymbol{\delta}=\beta\boldsymbol{\delta}_{\mathrm{rosu}}\), the linear term of the expansion above vanishes, and the expansion then gives \(|L_r(\mathbf w+\beta\boldsymbol{\delta}_{\mathrm{rosu}})-L_r(\mathbf w)|\le\frac{M_r}{2}\beta^2\rho^2\).

Near orthogonality the two perturbations are close. Both perturbations are \(\rho\) times a unit vector, and \(\mathbf q^\top\mathbf g_f=\|\mathbf q\|_2^2\). Expanding the squared distance gives
\[
\begin{aligned}
\|\boldsymbol{\delta}_{\mathrm{rosu}}-\boldsymbol{\delta}_{\mathrm{std}}\|_2^2&=\rho^2\Bigl\|\frac{\mathbf{q}}{\|\mathbf{q}\|_2}-\frac{\mathbf{g}_f}{\|\mathbf{g}_f\|_2}\Bigr\|_2^2\\
&=2\rho^2\Bigl(1-\frac{\mathbf{q}^\top\mathbf{g}_f}{\|\mathbf{q}\|_2\,\|\mathbf{g}_f\|_2}\Bigr)\\
&=2\rho^2\bigl(1-\sin\theta\bigr)\\
&=\frac{2\rho^2\cos^2\theta}{1+\sin\theta}\\
&\le2\rho^2\cos^2\theta ,
\end{aligned}
\]
where the second equality expands the square of a difference of unit vectors and the third uses \(\mathbf q^\top\mathbf g_f=\|\mathbf q\|_2^2\) and the gain ratio. The fourth equality multiplies and divides by \(1+\sin\theta\), and the inequality uses \(\sin\theta\ge0\). Taking square roots gives \(\|\boldsymbol{\delta}_{\mathrm{rosu}}-\boldsymbol{\delta}_{\mathrm{std}}\|_2\le\sqrt2\rho|\cos\theta|\). If \(L_r\) is \(G_r\)-Lipschitz on the segment joining the two surrogate points, then \(|F_{\mathrm{rosu}}-F_{\mathrm{std}}|\le G_r\|\boldsymbol{\delta}_{\mathrm{rosu}}-\boldsymbol{\delta}_{\mathrm{std}}\|_2\le\sqrt2G_r\rho|\cos\theta|\).

\subsection{Exact and relaxed retain corrections}
\label{app:exact-relaxed}

Here we derive the exact chain-rule gradient of \(F_{\mathrm{rosu}}\), obtain the relaxed Jacobian of Section~\ref{sec:method}, and prove Proposition~\ref{prop:approx-grad}.

Let \(L_f,L_r\in C^2\), \(\mathbf g_r\neq\mathbf{0}\), and \(\mathbf q\neq\mathbf{0}\). Write \(\widetilde{\mathbf{g}}_r(\mathbf{w}):=\left.\nabla_{\mathbf z}L_r(\mathbf z)\right|_{\mathbf z=\mathbf w+\boldsymbol{\delta}_{\mathrm{rosu}}(\mathbf w)}\). Differentiating the surrogate retain objective with \(\rho\) fixed gives
\begin{equation*}
\begin{aligned}
\nabla F_{\mathrm{rosu}}(\mathbf{w})&=\bigl(\mathbf{I}+\mathbf{J}_{\boldsymbol{\delta}_{\mathrm{rosu}}}(\mathbf{w})^\top\bigr)\,\widetilde{\mathbf{g}}_r(\mathbf{w}).
\end{aligned}
\end{equation*}
Let \(\boldsymbol\xi\in\mathbb R^p\) be a parameter direction and \(\mathbf J_{\mathbf q}\) the Jacobian of \(\mathbf q\).
All quantities below are evaluated at \(\mathbf w\) unless another argument is shown.
At a nonzero vector \(\mathbf z\in\mathbb R^p\), the norm has derivative \(\mathrm D\|\mathbf z\|_2[\boldsymbol\xi]=\frac{\mathbf z^\top\boldsymbol\xi}{\|\mathbf z\|_2}\), and the quotient rule gives
\[
\mathrm D\left(\frac{\mathbf z}{\|\mathbf z\|_2}\right)[\boldsymbol\xi]
=\frac{1}{\|\mathbf z\|_2}
\left(\mathbf I-\frac{\mathbf z\mathbf z^\top}{\|\mathbf z\|_2^2}\right)\boldsymbol\xi.
\]

We next apply this derivative to the normalized retain gradient.
Since \(\mathbf u_r=\frac{\mathbf g_r}{\|\mathbf g_r\|_2}\),
\[
\begin{aligned}
\mathrm D\mathbf u_r[\boldsymbol\xi]
&=\mathrm D\left(\frac{\mathbf g_r}{\|\mathbf g_r\|_2}\right)[\boldsymbol\xi]\\
&=\frac{1}{\|\mathbf g_r\|_2}
\left(\mathbf I-\frac{\mathbf g_r\mathbf g_r^\top}{\|\mathbf g_r\|_2^2}\right)
\mathrm D\mathbf g_r[\boldsymbol\xi]\\
&=\frac{1}{\|\mathbf g_r\|_2}
(\mathbf I-\mathbf P_r)\nabla^2L_r(\mathbf w)\boldsymbol\xi,
\end{aligned}
\]
where the second equality applies the normalization derivative through the chain rule, and the third uses the definition of \(\mathbf P_r\) and \(\mathrm D\mathbf g_r[\boldsymbol\xi]=\nabla^2L_r(\mathbf w)\boldsymbol\xi\).
Differentiating \(\mathbf P_r=\mathbf u_r\mathbf u_r^\top\) now yields
\[
\begin{aligned}
\mathrm D\mathbf P_r[\boldsymbol\xi]
&=\mathrm D(\mathbf u_r\mathbf u_r^\top)[\boldsymbol\xi]\\
&=(\mathrm D\mathbf u_r[\boldsymbol\xi])\mathbf u_r^\top
+\mathbf u_r(\mathrm D\mathbf u_r[\boldsymbol\xi])^\top\\
&=\frac{(\mathbf I-\mathbf P_r)\nabla^2L_r(\mathbf w)\boldsymbol\xi\mathbf u_r^\top
+\mathbf u_r\boldsymbol\xi^\top\nabla^2L_r(\mathbf w)(\mathbf I-\mathbf P_r)}
{\|\mathbf g_r\|_2},
\end{aligned}
\]
where the second equality uses the product rule, and the third substitutes the preceding derivative and uses the symmetry of \(\mathbf P_r\) and \(\nabla^2L_r\).
The Hessian is symmetric because \(L_r\) is \(C^2\).

Finally, apply the normalization derivative at \(\mathbf z=\mathbf q\) to \(\boldsymbol\delta_{\mathrm{rosu}}\) in~\eqref{eq:rosu-solution}.
Together with the product rule for \(\mathbf q=(\mathbf I-\mathbf P_r)\mathbf g_f\), this gives
\begin{subequations}
\label{eq:exact-jacobian}
\begin{align}
\mathbf{J}_{\boldsymbol{\delta}_{\mathrm{rosu}}}(\mathbf{w})&=\frac{\rho}{\|\mathbf{q}\|_2}\bigl(\mathbf{I}-\mathbf{P}_\perp\bigr)\mathbf{J}_{\mathbf{q}}(\mathbf{w}),\label{eq:exact-jacobian-delta}\\
\mathbf{J}_{\mathbf{q}}(\mathbf{w})[\boldsymbol{\xi}]&=(\mathbf{I}-\mathbf{P}_r)\nabla^2L_f(\mathbf{w})\,\boldsymbol{\xi}-\mathrm{D}\mathbf{P}_r(\mathbf{w})[\boldsymbol{\xi}]\,\mathbf{g}_f ,\label{eq:exact-jacobian-q}
\end{align}
\end{subequations}
where~\eqref{eq:exact-jacobian-delta} uses \(\frac{\mathbf q\mathbf q^\top}{\|\mathbf q\|_2^2}=\mathbf P_\perp\), and~\eqref{eq:exact-jacobian-q} uses \(\mathrm D\mathbf g_f[\boldsymbol{\xi}]=\nabla^2L_f\boldsymbol{\xi}\). These are the two derivative terms that A1 and A2 replace.

A1 sets \(\mathrm D\mathbf P_r[\boldsymbol{\xi}]\) to zero, and A2 replaces \(\nabla^2L_f\) by \(\mathbf I\) in~\eqref{eq:exact-jacobian-q}. These replacements leave the forward constraint and the projector unchanged, and the relaxed derivative of \(\mathbf q\) becomes \(\mathbf I-\mathbf P_r\). Substituting it into~\eqref{eq:exact-jacobian-delta} gives
\begin{equation}
\label{eq:relaxed-jacobian}
\begin{aligned}
\widehat{\mathbf{J}}_{\boldsymbol{\delta}_{\mathrm{rosu}}}(\mathbf{w})&=\frac{\rho}{\|\mathbf{q}\|_2}(\mathbf{I}-\mathbf{P}_\perp)(\mathbf{I}-\mathbf{P}_r)\\
&=\frac{\rho}{\|\mathbf{q}\|_2}\bigl(\mathbf{I}-\mathbf{P}_r-\mathbf{P}_\perp+\mathbf{P}_\perp\mathbf{P}_r\bigr)\\
&=\frac{\rho}{\|\mathbf{q}\|_2}\bigl(\mathbf{I}-\mathbf{P}_r-\mathbf{P}_\perp\bigr),
\end{aligned}
\end{equation}
where the last equality uses \(\mathbf P_\perp\mathbf P_r=\mathbf{0}\), since \(\mathbf u_\perp\) is orthogonal to \(\mathbf u_r\). The ideal relaxed Jacobian is therefore symmetric, and substituting it into the chain-rule expression yields \(\mathbf v\) in~\eqref{eq:rosu-update}. Appendix~\ref{app:regularization} evaluates the cross term for the regularized operator.

\begin{proof}[Proof of Proposition~\ref{prop:approx-grad}]
Subtracting~\eqref{eq:relaxed-jacobian} from~\eqref{eq:exact-jacobian-delta} gives
\[
\begin{aligned}
\mathbf{J}_{\boldsymbol{\delta}_{\mathrm{rosu}}}-\widehat{\mathbf{J}}_{\boldsymbol{\delta}_{\mathrm{rosu}}}&=\frac{\rho}{\|\mathbf{q}\|_2}(\mathbf{I}-\mathbf{P}_\perp)\bigl(\mathbf{J}_{\mathbf{q}}-(\mathbf{I}-\mathbf{P}_r)\bigr),\\
\bigl(\mathbf{J}_{\mathbf{q}}-(\mathbf{I}-\mathbf{P}_r)\bigr)\boldsymbol{\xi}&=(\mathbf{I}-\mathbf{P}_r)\bigl(\nabla^2L_f-\mathbf{I}\bigr)\boldsymbol{\xi}-\mathrm{D}\mathbf{P}_r[\boldsymbol{\xi}]\,\mathbf{g}_f .
\end{aligned}
\]
The two terms on the right of the second line are the A2 and A1 errors. For \(\|\boldsymbol{\xi}\|_2=1\),
\[
\begin{aligned}
\bigl\|\bigl(\mathbf{J}_{\mathbf{q}}-(\mathbf{I}-\mathbf{P}_r)\bigr)\boldsymbol{\xi}\bigr\|_2&\le\|\mathbf{I}-\mathbf{P}_r\|_{\mathrm{op}}\,\|\nabla^2L_f-\mathbf{I}\|_{\mathrm{op}}\,\|\boldsymbol{\xi}\|_2+\|\mathrm{D}\mathbf{P}_r[\boldsymbol{\xi}]\|_{\mathrm{op}}\,\|\mathbf{g}_f\|_2\\
&\le\varepsilon_H+\varepsilon_P\|\mathbf{g}_f\|_2 ,
\end{aligned}
\]
where the first inequality is the triangle inequality with submultiplicativity, and the second uses \(\|\mathbf I-\mathbf P_r\|_{\mathrm{op}}\le1\), \(\|\boldsymbol{\xi}\|_2=1\), and the two assumptions. Since \(\|\mathbf I-\mathbf P_\perp\|_{\mathrm{op}}\le1\), the first line gives~\eqref{eq:approx-grad-bound-jac}. For~\eqref{eq:approx-grad-bound-grad}, the chain rule, the definition of \(\mathbf v\), and the symmetry of \(\widehat{\mathbf J}_{\boldsymbol{\delta}_{\mathrm{rosu}}}\) give
\[
\begin{aligned}
\bigl\|\nabla F_{\mathrm{rosu}}-\mathbf{v}\bigr\|_2&=\bigl\|\bigl(\mathbf{J}_{\boldsymbol{\delta}_{\mathrm{rosu}}}-\widehat{\mathbf{J}}_{\boldsymbol{\delta}_{\mathrm{rosu}}}\bigr)^\top\widetilde{\mathbf{g}}_r\bigr\|_2\\
&\le\bigl\|\mathbf{J}_{\boldsymbol{\delta}_{\mathrm{rosu}}}-\widehat{\mathbf{J}}_{\boldsymbol{\delta}_{\mathrm{rosu}}}\bigr\|_{\mathrm{op}}\,\|\widetilde{\mathbf{g}}_r\|_2\\
&\le\frac{\rho\,\|\widetilde{\mathbf{g}}_r\|_2}{\|\mathbf{q}\|_2}\bigl(\varepsilon_H+\varepsilon_P\|\mathbf{g}_f\|_2\bigr),
\end{aligned}
\]
where the first inequality uses that a matrix and its transpose have the same operator norm, and the second is the first bound.
\end{proof}

The bound also yields a descent condition. Suppose \(\|\nabla F_{\mathrm{rosu}}-\mathbf v\|_2<\|\mathbf v\|_2\), which also forces \(\mathbf v\neq\mathbf{0}\). Then
\[
\begin{aligned}
\mathbf{v}^\top\nabla F_{\mathrm{rosu}}&=\|\mathbf{v}\|_2^2+\mathbf{v}^\top\bigl(\nabla F_{\mathrm{rosu}}-\mathbf{v}\bigr)\\
&\ge\|\mathbf{v}\|_2^2-\|\mathbf{v}\|_2\,\|\nabla F_{\mathrm{rosu}}-\mathbf{v}\|_2\\
&>0 ,
\end{aligned}
\]
where the first inequality is Cauchy--Schwarz and the second uses the assumption. This makes \(-\mathbf v\) a descent direction, and small enough steps along it decrease \(F_{\mathrm{rosu}}\).

Under the Hessian bound on the ideal surrogate segment,
\[
\begin{aligned}
\|\widetilde{\mathbf{g}}_r\|_2&=\Bigl\|\mathbf{g}_r+\int\limits_0^1\nabla^2L_r(\mathbf{w}+\zeta\boldsymbol{\delta}_{\mathrm{rosu}})\,\boldsymbol{\delta}_{\mathrm{rosu}}\,\mathrm{d}\zeta\Bigr\|_2\\
&\le\|\mathbf{g}_r\|_2+M_r\rho ,
\end{aligned}
\]
where the equality is the fundamental theorem of calculus along the segment, and the inequality uses the Hessian bound with \(\|\boldsymbol{\delta}_{\mathrm{rosu}}\|_2=\rho\). Substituting it into~\eqref{eq:approx-grad-bound-grad} expresses the deviation through the retain gradient at the current weights.

\subsection{Regularization and mini-batch residuals}
\label{app:regularization}

Here we quantify how the stabilizer \(\tau\), the degenerate fallback, and mini-batch estimation change the ideal construction.

For \(\tau\ge0\) and \(\mathbf g_r\neq\mathbf{0}\), the operator and direction satisfy
\begin{subequations}
\label{eq:qtau-decomp}
\begin{align}
\mathbf{P}_r^{(\tau)}&:=\frac{\mathbf{g}_r\mathbf{g}_r^\top}{\|\mathbf{g}_r\|_2^2+\tau}\nonumber\\
&=\frac{\|\mathbf{g}_r\|_2^2}{\|\mathbf{g}_r\|_2^2+\tau}\,\mathbf{P}_r,\label{eq:qtau-decomp-P}\\
\mathbf{q}_\tau&:=(\mathbf{I}-\mathbf{P}_r^{(\tau)})\mathbf{g}_f\nonumber\\
&=\mathbf{q}+\frac{\tau}{\|\mathbf{g}_r\|_2^2+\tau}\,\mathbf{P}_r\mathbf{g}_f ,\label{eq:qtau-decomp-q}
\end{align}
\end{subequations}
where~\eqref{eq:qtau-decomp-q} uses~\eqref{eq:qtau-decomp-P} and \(1-\frac{\|\mathbf g_r\|_2^2}{\|\mathbf g_r\|_2^2+\tau}=\frac{\tau}{\|\mathbf g_r\|_2^2+\tau}\). For \(\tau>0\), \(\mathbf P_r^{(\tau)}\) is not idempotent. The retain inner product of \(\mathbf q_\tau\) is
\[
\begin{aligned}
\mathbf{g}_r^\top\mathbf{q}_\tau&=\mathbf{g}_r^\top\mathbf{q}+\frac{\tau}{\|\mathbf{g}_r\|_2^2+\tau}\,\mathbf{g}_r^\top\mathbf{P}_r\mathbf{g}_f\\
&=\frac{\tau\,\mathbf{g}_f^\top\mathbf{g}_r}{\|\mathbf{g}_r\|_2^2+\tau},
\end{aligned}
\]
where the second equality uses \(\mathbf g_r^\top\mathbf q=0\) and \(\mathbf g_r^\top\mathbf P_r=\mathbf g_r^\top\).

The operator difference has norm \(\|\mathbf P_r^{(\tau)}-\mathbf P_r\|_{\mathrm{op}}=\frac{\tau}{\|\mathbf g_r\|_2^2+\tau}\), because \(\|\mathbf P_r\|_{\mathrm{op}}=1\). Applying that difference to \(\mathbf g_f\) yields
\begin{equation}
\label{eq:qtau-gap}
\begin{aligned}
\|\mathbf{q}_\tau-\mathbf{q}\|_2&\le\frac{\tau}{\|\mathbf{g}_r\|_2^2+\tau}\,\|\mathbf{g}_f\|_2 .
\end{aligned}
\end{equation}
The two components in~\eqref{eq:qtau-decomp-q} are orthogonal, so
\[
\begin{aligned}
\|\mathbf{q}_\tau\|_2^2&=\|\mathbf{q}\|_2^2+\Bigl(\frac{\tau}{\|\mathbf{g}_r\|_2^2+\tau}\Bigr)^2\|\mathbf{P}_r\mathbf{g}_f\|_2^2\\
&\ge\|\mathbf{q}\|_2^2 ,
\end{aligned}
\]
where the equality is Pythagoras. If \(\mathbf g_r=\mathbf{0}\) and \(\tau>0\), the implemented formula gives \(\mathbf q_\tau=\mathbf g_f\), and the ideal projector is undefined.

Let \(\mathbf q_\tau\neq\mathbf{0}\), \(\mathbf u_\tau:=\frac{\mathbf q_\tau}{\|\mathbf q_\tau\|_2}\), and \(\mathbf P_\tau:=\mathbf u_\tau\mathbf u_\tau^\top\). Algorithm~\ref{alg:rosu_full} uses the symmetric matrix \(\widehat{\mathbf J}_\tau^{\mathrm{impl}}:=\frac{\rho}{\|\mathbf q_\tau\|_2}(\mathbf I-\mathbf P_r^{(\tau)}-\mathbf P_\tau)\). Repeating A1/A2 with \(\mathbf P_r^{(\tau)}\) instead gives the product-form Jacobian
\[
\begin{aligned}
\widehat{\mathbf{J}}_\tau^{\mathrm{prod}}&:=\frac{\rho}{\|\mathbf{q}_\tau\|_2}(\mathbf{I}-\mathbf{P}_\tau)(\mathbf{I}-\mathbf{P}_r^{(\tau)})\\
&=\frac{\rho}{\|\mathbf{q}_\tau\|_2}\bigl(\mathbf{I}-\mathbf{P}_r^{(\tau)}-\mathbf{P}_\tau\bigr)+\frac{\rho}{\|\mathbf{q}_\tau\|_2}\,\mathbf{P}_\tau\mathbf{P}_r^{(\tau)}\\
&=\widehat{\mathbf{J}}_\tau^{\mathrm{impl}}+\frac{\rho}{\|\mathbf{q}_\tau\|_2}\,\mathbf{P}_\tau\mathbf{P}_r^{(\tau)} .
\end{aligned}
\]
Rearranging and transposing gives
\begin{equation*}
\begin{aligned}
\widehat{\mathbf{J}}_\tau^{\mathrm{prod}}-\widehat{\mathbf{J}}_\tau^{\mathrm{impl}}
&=\frac{\rho}{\|\mathbf{q}_\tau\|_2}\,\mathbf{P}_\tau\mathbf{P}_r^{(\tau)},\\
\bigl(\widehat{\mathbf{J}}_\tau^{\mathrm{prod}}\bigr)^\top-\bigl(\widehat{\mathbf{J}}_\tau^{\mathrm{impl}}\bigr)^\top
&=\frac{\rho}{\|\mathbf{q}_\tau\|_2}\,\mathbf{P}_r^{(\tau)}\mathbf{P}_\tau ,
\end{aligned}
\end{equation*}
where the second line uses the symmetry of \(\mathbf P_\tau\) and \(\mathbf P_r^{(\tau)}\). The cross term vanishes at \(\tau=0\). For \(\tau>0\), it can be nonzero, and it still vanishes when \(\mathbf g_f^\top\mathbf g_r=0\). The two cross products are transposes of each other with equal operator norms, and the transported correction uses the transpose.

Writing \(d_\tau:=\|\mathbf{g}_r\|_2^2+\tau\), the cross term has the rank-one form \(\mathbf{P}_\tau\mathbf{P}_r^{(\tau)}=\frac{\mathbf{u}_\tau^\top\mathbf{g}_r}{d_\tau}\,\mathbf{u}_\tau\mathbf{g}_r^\top\), and the retain inner product above gives \(|\mathbf{u}_\tau^\top\mathbf{g}_r|=\frac{\tau|\mathbf{g}_f^\top\mathbf{g}_r|}{d_\tau\|\mathbf{q}_\tau\|_2}\). Taking operator norms gives
\begin{equation}
\label{eq:reg-jac-bound}
\begin{aligned}
\bigl\|\widehat{\mathbf{J}}_\tau^{\mathrm{prod}}-\widehat{\mathbf{J}}_\tau^{\mathrm{impl}}\bigr\|_{\mathrm{op}}&=\frac{\rho}{\|\mathbf{q}_\tau\|_2}\,\frac{|\mathbf{u}_\tau^\top\mathbf{g}_r|\,\|\mathbf{g}_r\|_2}{d_\tau}\\
&=\frac{\rho\,\tau\,|\mathbf{g}_f^\top\mathbf{g}_r|\,\|\mathbf{g}_r\|_2}{\|\mathbf{q}_\tau\|_2^2\,d_\tau^2}\\
&\le\frac{\rho\,\tau\,\|\mathbf{g}_f\|_2\,\|\mathbf{g}_r\|_2^2}{\|\mathbf{q}_\tau\|_2^2\,d_\tau^2},
\end{aligned}
\end{equation}
where the first equality uses \(\|\mathbf x\mathbf y^\top\|_{\mathrm{op}}=\|\mathbf x\|_2\|\mathbf y\|_2\) with \(\|\mathbf u_\tau\|_2=1\), the second substitutes \(|\mathbf u_\tau^\top\mathbf g_r|\), and the inequality is Cauchy--Schwarz. The transported-vector difference is \(\frac{\rho}{\|\mathbf q_\tau\|_2}\mathbf P_r^{(\tau)}\mathbf P_\tau\widetilde{\mathbf g}_{r,\tau}\), whose norm is at most the bound in~\eqref{eq:reg-jac-bound} times \(\|\widetilde{\mathbf g}_{r,\tau}\|_2\). Here \(\widetilde{\mathbf g}_{r,\tau}\) is evaluated at \(\mathbf w+\rho\mathbf u_\tau\). The bound in~\eqref{eq:reg-jac-bound} is uniformly \(O(\tau)\) when \(\|\mathbf g_r\|_2\) and \(\|\mathbf q_\tau\|_2\) are bounded away from zero and \(\rho\) and \(\|\mathbf g_f\|_2\) are bounded above.

Substituting \(\boldsymbol{\delta}_\tau=\rho\mathbf u_\tau\) into the retain inner product above gives
\begin{equation}
\label{eq:regularized-retain-residual}
\begin{aligned}
\mathbf{g}_r^\top\boldsymbol{\delta}_\tau&=\frac{\rho}{\|\mathbf{q}_\tau\|_2}\,\mathbf{g}_r^\top\mathbf{q}_\tau\\
&=\frac{\rho\,\tau\,\mathbf{g}_f^\top\mathbf{g}_r}{(\|\mathbf{g}_r\|_2^2+\tau)\,\|\mathbf{q}_\tau\|_2}.
\end{aligned}
\end{equation}
Assume \(L_r\) is \(C^2\) near the segment from \(\mathbf w\) to \(\mathbf w+\boldsymbol{\delta}_\tau\), with Hessian operator norm at most \(M_r\) on that segment. The expansion of Appendix~\ref{app:core-geometry} then gives
\[
\begin{aligned}
L_r(\mathbf{w}+\boldsymbol{\delta}_\tau)-L_r(\mathbf{w})&=\mathbf{g}_r^\top\boldsymbol{\delta}_\tau+R_{\boldsymbol{\delta}_\tau}\\
&\le\mathbf{g}_r^\top\boldsymbol{\delta}_\tau+\frac{M_r}{2}\,\rho^2 ,
\end{aligned}
\]
where the inequality applies the remainder bound with \(\|\boldsymbol{\delta}_\tau\|_2=\rho\). Scaling the regularized displacement by \(\beta\) multiplies the linear residual by \(\beta\). The ideal construction with \(\tau=0\) is exactly retain-neutral to first order.

When \(\|\mathbf q_\tau\|_2\le\varepsilon_q\), Algorithm~\ref{alg:rosu_full} takes a plain retain-gradient step. For \(\mathbf g_r\neq\mathbf{0}\) and \(\tau\ge0\), the forgone forget gain satisfies
\begin{equation*}
\begin{aligned}
\max\limits_{\boldsymbol{\delta}}\bigl\{\mathbf{g}_f^\top\boldsymbol{\delta}\;\text{s.t.}\;\|\boldsymbol{\delta}\|_2\le\rho,\;\mathbf{g}_r^\top\boldsymbol{\delta}=0\bigr\}&=\rho\,\|\mathbf{q}\|_2\\
&\le\rho\,\|\mathbf{q}_\tau\|_2\\
&\le\rho\,\varepsilon_q ,
\end{aligned}
\end{equation*}
where the equality is Proposition~\ref{prop:inner}, the first inequality uses \(\|\mathbf q_\tau\|_2\ge\|\mathbf q\|_2\) from the decomposition above, and the second is the fallback test. The equality also holds when \(\mathbf q=\mathbf{0}\), because every feasible gain is then zero.

If \(\mathbf g_r=\mathbf{0}\) and \(\tau>0\), the retain constraint is vacuous and \(\mathbf q_\tau=\mathbf g_f\). The same test then bounds the ball optimum by \(\rho\|\mathbf g_f\|_2=\rho\|\mathbf q_\tau\|_2\le\rho\varepsilon_q\). Applying the argument at radius \(\beta\rho\) bounds the forgone gain from retain-neutral amplification by \(\beta\rho\varepsilon_q\).

For mini-batch residuals, use the gradients and unit direction defined in Appendix~\ref{app:retain-batch-size}. Dividing the mini-batch form of~\eqref{eq:regularized-retain-residual} by \(\rho\) gives
\begin{equation*}
\begin{aligned}
(\mathbf{g}_r^{\mathrm{mini}})^\top\mathbf{u}_{\tau}^{\mathrm{mini}} &=\frac{\tau(\mathbf{g}_f^{\mathrm{mini}})^\top\mathbf{g}_r^{\mathrm{mini}}} {(\|\mathbf{g}_r^{\mathrm{mini}}\|_2^2+\tau)\|\mathbf{q}_{\tau}^{\mathrm{mini}}\|_2}.
\end{aligned}
\end{equation*}
Adding and subtracting \(\mathbf g_r^{\mathrm{mini}}\) in the full-retain inner product yields
\begin{equation*}
\begin{aligned}
\left|(\mathbf{g}_r^{\mathrm{full}})^\top\mathbf{u}_{\tau}^{\mathrm{mini}}\right| &=\left|(\mathbf{g}_r^{\mathrm{full}}-\mathbf{g}_r^{\mathrm{mini}})^\top\mathbf{u}_{\tau}^{\mathrm{mini}}+(\mathbf{g}_r^{\mathrm{mini}})^\top\mathbf{u}_{\tau}^{\mathrm{mini}}\right|\\
&\le\left|(\mathbf{g}_r^{\mathrm{full}}-\mathbf{g}_r^{\mathrm{mini}})^\top\mathbf{u}_{\tau}^{\mathrm{mini}}\right|+\left|(\mathbf{g}_r^{\mathrm{mini}})^\top\mathbf{u}_{\tau}^{\mathrm{mini}}\right|\\
&\le\|\mathbf{g}_r^{\mathrm{full}}-\mathbf{g}_r^{\mathrm{mini}}\|_2+\frac{\tau\left|(\mathbf{g}_f^{\mathrm{mini}})^\top\mathbf{g}_r^{\mathrm{mini}}\right|}{(\|\mathbf{g}_r^{\mathrm{mini}}\|_2^2+\tau)\|\mathbf{q}_{\tau}^{\mathrm{mini}}\|_2},
\end{aligned}
\end{equation*}
where the first inequality is the triangle inequality, and the second uses Cauchy--Schwarz with \(\|\mathbf u_\tau^{\mathrm{mini}}\|_2=1\) and the mini-batch residual above. This is~\eqref{eq:mini-full-transfer}. For \(\mathbf g_r^{\mathrm{full}}\neq\mathbf{0}\), dividing by its norm gives the normalized diagnostic in Appendix~\ref{app:retain-batch-size}.

If \(\mathbf{g}_r^{\mathrm{mini}}\neq\mathbf{0}\), define \(\mathbf{P}_r^{\mathrm{mini}}:=\frac{\mathbf{g}_r^{\mathrm{mini}}(\mathbf{g}_r^{\mathrm{mini}})^\top}{\|\mathbf{g}_r^{\mathrm{mini}}\|_2^2}\) and \(\mathbf{q}^{\mathrm{mini}}:=(\mathbf{I}-\mathbf{P}_r^{\mathrm{mini}})\mathbf{g}_f^{\mathrm{mini}}\). For \(\mathbf{q}^{\mathrm{mini}}\neq\mathbf{0}\), the normalized direction \(\mathbf{u}_{\perp}^{\mathrm{mini}}:=\frac{\mathbf{q}^{\mathrm{mini}}}{\|\mathbf{q}^{\mathrm{mini}}\|_2}\) has exactly zero mini-batch retain inner product. Its full-retain violation is therefore bounded by the estimation error alone, \(|(\mathbf{g}_r^{\mathrm{full}})^\top\mathbf{u}_{\perp}^{\mathrm{mini}}|\le\|\mathbf{g}_r^{\mathrm{full}}-\mathbf{g}_r^{\mathrm{mini}}\|_2\). By~\eqref{eq:qtau-gap}, the ideal and implemented components satisfy
\begin{equation*}
\begin{aligned}
\|\mathbf{q}_{\tau}^{\mathrm{mini}}-\mathbf{q}^{\mathrm{mini}}\|_2 &\le\frac{\tau\|\mathbf{g}_f^{\mathrm{mini}}\|_2} {\|\mathbf{g}_r^{\mathrm{mini}}\|_2^2+\tau}.
\end{aligned}
\end{equation*}
To compare the normalized directions, take nonzero vectors \(\mathbf z_1,\mathbf z_2\). Adding and subtracting \(\frac{\mathbf z_2}{\|\mathbf z_1\|_2}\) gives
\begin{equation*}
\begin{aligned}
\left\|\frac{\mathbf z_1}{\|\mathbf z_1\|_2}-\frac{\mathbf z_2}{\|\mathbf z_2\|_2}\right\|_2&\le\frac{\|\mathbf z_1-\mathbf z_2\|_2}{\|\mathbf z_1\|_2}+\|\mathbf z_2\|_2\left|\frac{1}{\|\mathbf z_1\|_2}-\frac{1}{\|\mathbf z_2\|_2}\right|\\
&=\frac{\|\mathbf z_1-\mathbf z_2\|_2+\bigl|\|\mathbf z_2\|_2-\|\mathbf z_1\|_2\bigr|}{\|\mathbf z_1\|_2}\\
&\le\frac{2\|\mathbf z_1-\mathbf z_2\|_2}{\|\mathbf z_1\|_2},
\end{aligned}
\end{equation*}
where the last inequality is the reverse triangle inequality. Exchanging the roles of \(\mathbf z_1\) and \(\mathbf z_2\) gives the same bound with \(\|\mathbf z_2\|_2\) in the denominator, so the bound holds with \(\min\{\|\mathbf z_1\|_2,\|\mathbf z_2\|_2\}\). With \(\mathbf z_1=\mathbf q^{\mathrm{mini}}\) and \(\mathbf z_2=\mathbf q_\tau^{\mathrm{mini}}\), the normalized ideal and regularized directions are therefore close when the relative quantity
\begin{equation*}
\begin{aligned}
&\frac{\tau\|\mathbf{g}_f^{\mathrm{mini}}\|_2} {(\|\mathbf{g}_r^{\mathrm{mini}}\|_2^2+\tau) \min\{\|\mathbf{q}^{\mathrm{mini}}\|_2,\|\mathbf{q}_{\tau}^{\mathrm{mini}}\|_2\}}
\end{aligned}
\end{equation*}
is small. This criterion compares the direction gap with the projected-gradient norms. Figure~\ref{fig:grad_norms} reports the individual gradient-norm and fallback diagnostics.

We now compare the implemented correction with its ideal counterpart at a fixed \(\mathbf w\), with \(\mathbf g_r\neq\mathbf{0}\) and \(\mathbf q\neq\mathbf{0}\). Keep \(\rho\) fixed. Both corrections are computed from the same retain loss, so \(\tau\) is the only difference between them. The reverse triangle inequality gives
\[
\begin{aligned}
\left|\frac{\rho}{\|\mathbf q_\tau\|_2}-\frac{\rho}{\|\mathbf q\|_2}\right|
&=\frac{\rho\bigl|\|\mathbf q_\tau\|_2-\|\mathbf q\|_2\bigr|}
{\|\mathbf q\|_2\|\mathbf q_\tau\|_2}\\
&\le\frac{\rho\|\mathbf q_\tau-\mathbf q\|_2}
{\|\mathbf q\|_2\|\mathbf q_\tau\|_2}.
\end{aligned}
\]
Next, the two direction projectors satisfy
\[
\mathbf P_\tau-\mathbf P_\perp
=(\mathbf u_\tau-\mathbf u_\perp)\mathbf u_\tau^\top
+\mathbf u_\perp(\mathbf u_\tau-\mathbf u_\perp)^\top,
\]
and therefore
\[
\begin{aligned}
\|\mathbf P_\tau-\mathbf P_\perp\|_{\mathrm{op}}
&\le2\|\mathbf u_\tau-\mathbf u_\perp\|_2\\
&\le\frac{4\|\mathbf q_\tau-\mathbf q\|_2}{\|\mathbf q\|_2},
\end{aligned}
\]
where the first inequality uses the rank-one matrix norm and unit-vector norms, and the second uses the normalization bound above and \(\|\mathbf q_\tau\|_2\ge\|\mathbf q\|_2\).

Adding and subtracting the ideal matrix at the regularized scale gives
\[
\begin{aligned}
\widehat{\mathbf J}_\tau^{\mathrm{impl}}
-\widehat{\mathbf J}_{\boldsymbol\delta_{\mathrm{rosu}}}
&=\left(\frac{\rho}{\|\mathbf q_\tau\|_2}-\frac{\rho}{\|\mathbf q\|_2}\right)
(\mathbf I-\mathbf P_r-\mathbf P_\perp)\\
&\quad+\frac{\rho}{\|\mathbf q_\tau\|_2}
(\mathbf P_r-\mathbf P_r^{(\tau)}+\mathbf P_\perp-\mathbf P_\tau).
\end{aligned}
\]
The ideal matrix \(\mathbf I-\mathbf P_r-\mathbf P_\perp\) is an orthogonal projector, so its operator norm is at most one. Applying the preceding bounds yields
\[
\begin{aligned}
\bigl\|\widehat{\mathbf J}_\tau^{\mathrm{impl}}
-\widehat{\mathbf J}_{\boldsymbol\delta_{\mathrm{rosu}}}\bigr\|_{\mathrm{op}}
&\le\left|\frac{\rho}{\|\mathbf q_\tau\|_2}-\frac{\rho}{\|\mathbf q\|_2}\right|\\
&\quad+\frac{\rho}{\|\mathbf q_\tau\|_2}
\left(\|\mathbf P_r-\mathbf P_r^{(\tau)}\|_{\mathrm{op}}
+\|\mathbf P_\perp-\mathbf P_\tau\|_{\mathrm{op}}\right)\\
&\le\frac{\rho}{\|\mathbf q_\tau\|_2}
\left(\frac{\tau}{d_\tau}
+\frac{5\|\mathbf q_\tau-\mathbf q\|_2}{\|\mathbf q\|_2}\right)\\
&\le\frac{\rho\tau}{d_\tau\|\mathbf q_\tau\|_2}
\left(1+\frac{5\|\mathbf g_f\|_2}{\|\mathbf q\|_2}\right),
\end{aligned}
\]
where the first inequality uses the triangle inequality and submultiplicativity, the second combines the scale and projector bounds, and the third uses~\eqref{eq:qtau-gap}. Since \(d_\tau\ge\|\mathbf g_r\|_2^2>0\) and \(\|\mathbf q_\tau\|_2\ge\|\mathbf q\|_2>0\), this matrix difference is \(O(\tau)\).

It remains to account for the change in the surrogate retain gradient. The \(C^2\) assumption in Section~\ref{sec:correction} gives a neighborhood of the ideal surrogate on which the Hessian is continuous. Choose a closed ball around that surrogate within this neighborhood. The direction bound implies \(\boldsymbol\delta_\tau\to\boldsymbol\delta_{\mathrm{rosu}}\). For sufficiently small \(\tau\), the connecting segment lies in this ball, where the Hessian norm has a common bound, which we again denote by \(M_r\). Integrating the Hessian along this segment gives
\[
\begin{aligned}
\|\widetilde{\mathbf g}_{r,\tau}-\widetilde{\mathbf g}_r\|_2
&\le M_r\|\boldsymbol\delta_\tau-\boldsymbol\delta_{\mathrm{rosu}}\|_2\\
&=M_r\rho\|\mathbf u_\tau-\mathbf u_\perp\|_2\\
&\le\frac{2M_r\rho\|\mathbf q_\tau-\mathbf q\|_2}{\|\mathbf q\|_2},
\end{aligned}
\]
where the last inequality uses the normalization bound. This bound is also \(O(\tau)\) by~\eqref{eq:qtau-gap}.

Write \(\mathbf v_\tau:=(\mathbf I+\widehat{\mathbf J}_\tau^{\mathrm{impl}})\widetilde{\mathbf g}_{r,\tau}\) for the correction formula in the nonfallback branch. Both \(\mathbf I-\mathbf P_r^{(\tau)}\) and \(\mathbf P_\tau\) are positive semidefinite contractions, so
\[
\begin{aligned}
-\mathbf I&\preceq\mathbf I-\mathbf P_r^{(\tau)}-\mathbf P_\tau\preceq\mathbf I,\\
\|\widehat{\mathbf J}_\tau^{\mathrm{impl}}\|_{\mathrm{op}}
&\le\frac{\rho}{\|\mathbf q_\tau\|_2}.
\end{aligned}
\]
Adding and subtracting the implemented matrix applied to the ideal surrogate gradient gives
\[
\begin{aligned}
\mathbf v_\tau-\mathbf v
&=(\widehat{\mathbf J}_\tau^{\mathrm{impl}}
-\widehat{\mathbf J}_{\boldsymbol\delta_{\mathrm{rosu}}})\widetilde{\mathbf g}_r\\
&\quad+(\mathbf I+\widehat{\mathbf J}_\tau^{\mathrm{impl}})
(\widetilde{\mathbf g}_{r,\tau}-\widetilde{\mathbf g}_r),
\end{aligned}
\]
and therefore
\[
\begin{aligned}
\|\mathbf v_\tau-\mathbf v\|_2
&\le\bigl\|\widehat{\mathbf J}_\tau^{\mathrm{impl}}
-\widehat{\mathbf J}_{\boldsymbol\delta_{\mathrm{rosu}}}\bigr\|_{\mathrm{op}}
\|\widetilde{\mathbf g}_r\|_2\\
&\quad+\left(1+\frac{\rho}{\|\mathbf q_\tau\|_2}\right)
\frac{2M_r\rho\|\mathbf q_\tau-\mathbf q\|_2}{\|\mathbf q\|_2}\\
&=O(\tau),
\end{aligned}
\]
where the inequality uses the triangle inequality, submultiplicativity, and the gradient bound. The final order follows from the two preceding estimates at this fixed nondegenerate point.

\end{document}